\documentclass[letterpaper,12pt]{article}
\usepackage{natbib} \bibpunct{(}{)}{;}{a}{}{,}
\usepackage{fullpage}
\usepackage{amsmath,amssymb,amsthm}
\usepackage{enumerate}
\usepackage{appendix}
\usepackage{graphicx}
\usepackage{caption}
\usepackage{subcaption}
\usepackage{color}
\usepackage{hyperref}
\hypersetup{
        colorlinks   = true,
        linkcolor    = blue,
        citecolor    = green
}
\usepackage{authblk}
\usepackage{algorithm}
\usepackage{algpseudocode}

\newcommand{\E}{\mathbb{E}}
\newcommand{\R}{\mathbb{R}}
\newcommand{\IR}{\R}
\newcommand{\IN}{\mathbb{N}}
\newcommand{\IP}{\mathbb{P}}
\newcommand{\IE}{\E}

\newcommand{\bbO}{\mathbb{O}}

\newcommand{\xbar}{\bar{x}}

\newcommand{\Bhat}{\hat{B}}
\newcommand{\Rhat}{\hat{R}}

\newcommand{\Vhat}{\hat{V}}
\newcommand{\What}{\hat{W}}
\newcommand{\Xhat}{\hat{X}}
\newcommand{\bhat}{\hat{b}}
\newcommand{\rhat}{\hat{r}}

\newcommand{\Rhatpara}{\Rhat_{\parallel}}

\newcommand{\Vhatpara}{\Vhat_{\parallel}}
\newcommand{\Whatpara}{\What_{\parallel}}

\newcommand{\Rhatperp}{\Rhat_{\perp}}

\newcommand{\Vhatperp}{\Vhat_{\perp}}
\newcommand{\Whatperp}{\What_{\perp}}

\newcommand{\Btilde}{\tilde{B}}

\newcommand{\Qtilde}{\tilde{Q}}
\newcommand{\Rtilde}{\widetilde{R}}
\newcommand{\Stilde}{\tilde{S}}
\newcommand{\Vtilde}{\widetilde{V}}
\newcommand{\Wtilde}{\widetilde{W}}

\newcommand{\Sigmatilde}{\widetilde{\Sigma}}

\newcommand{\btilde}{\tilde{b}}
\newcommand{\rtilde}{\tilde{r}}

\newcommand{\btrue}{b}

\newcommand{\regu}{\gamma} %Regularizer; were using r, but some collison.

\newcommand{\Rtildepara}{\Rtilde_{\parallel}}
\newcommand{\Vtildeperp}{\Vtilde_\perp}
\newcommand{\Vtildepara}{\Vtilde_{\parallel}}
\newcommand{\Wtildeperp}{\Wtilde_\perp}
\newcommand{\Wtildepara}{\Wtilde_{\parallel}}

\newcommand{\tti}{{2,\infty}}

\newcommand{\argmin}[1]{\underset{#1}{\operatorname{argmin}}}
\newcommand{\argmax}[1]{\underset{#1}{\operatorname{argmax}}}
\newcommand{\diag}{\operatorname{diag}}
\newcommand{\CCA}{\operatorname{CCA}}
\newcommand{\CCAregu}{\CCA_{\regu}}
\newcommand{\CCAdagger}{\CCA_{\dagger}}

\newcommand{\RDPG}{\operatorname{RDPG}}
\newcommand{\ridge}{\operatorname{ridge}}

\newcommand{\gLASSO}{\operatorname{gLASSO}}
\newcommand{\rLASSO}{\operatorname{rLASSO}}
\newcommand{\trace}{\operatorname{trace}}
\newcommand{\Bern}{\operatorname{Bern}}
\newcommand{\Dirichlet}{\operatorname{Dirichlet}}
\newcommand{\Var}{\operatorname{Var}}
\newcommand{\Cov}{\operatorname{Cov}}

\newcommand{\calB}{\mathcal{B}}
\newcommand{\calN}{\mathcal{N}}
\newcommand{\calT}{\mathcal{T}}

\newcommand{\Op}[1]{O_P\!\left( #1 \right)}
\newcommand{\op}[1]{o_P\!\left( #1 \right)}
\newcommand{\Omegap}[1]{\Omega_P\!\left( #1 \right)}
\newcommand{\Thetap}[1]{\Theta_P\!\left( #1 \right)}

\newcommand{\ttiRate}{\xi}
\newcommand{\spectralRate}{\eta}
\newcommand{\subspaceRate}{\zeta}

\newtheorem{theorem}{Theorem}[section]
\newtheorem{lemma}[theorem]{Lemma}
\newtheorem{corollary}[theorem]{Corollary}
\newtheorem{remark}[theorem]{Remark}
\newtheorem{definition}[theorem]{Definition}
\newtheorem{assumption}{Assumption}

\newcommand{\lassoreg}{\alpha} % symbol to use for the regularization param.
\newcommand{\netnoise}{N} % symbol to use for the noise in the network on covariates regression.
\newcommand{\covarcov}{q} % symbol to use for covariance of the covariates in the simulation study
\newcommand{\extradeltarate}{\omega_n}

\begin{document}

\title{Testing for correlation between network structure and high-dimensional node covariates}
\author[1]{Alexander Fuchs-Krei{\ss}\textsuperscript{*}}
\author[2]{Keith Levin\textsuperscript{*}}

\affil[1]{\small Leipzig University}
\affil[2]{\small University of Wisconsin--Madison}

\let\svthefootnote\thefootnote
\let\thefootnote\relax\footnotetext{* Equal contribution; author order was determined by coin flip.}
\let\thefootnote\svthefootnote

\maketitle

\begin{abstract}
In many application domains, networks are observed with node-level features.
In such settings, a common problem is to assess whether or not nodal covariates are correlated with the network structure itself.
Here, we present four novel methods for addressing this problem.
Two of these are based on a linear model relating node-level covariates to latent node-level variables that drive network structure.
The other two are based on applying canonical correlation analysis to the node features and network structure, avoiding the linear modeling assumptions.
We provide theoretical guarantees for all four methods when the observed network is generated according to a low-rank latent space model endowed with node-level covariates, which we allow to be high-dimensional.
Our methods are computationally cheaper and require fewer modeling assumptions than previous approaches to network dependency testing.
We demonstrate and compare the performance of our novel methods on both simulated and real-world data.
\end{abstract}

\section{Introduction} \label{sec:intro}
In a vast range of scientific disciplines, data takes the form of collections of interacting entities that are naturally described by networks \citep{Newman2018}.
For example, social media data typically comprise users, pairs of whom engage in friendship or following relations \citep{Golbeck2013,McCArmJoh2013}.
Co-authorship and collaboration networks describe which pairs of researchers work together \citep{JiJinKeLi2022}.
In neuroscience, networks arise in the form of connectome data, in which nodes correspond to brain regions or individual neurons and edges encode strength of connections between brain regions or formation of synapses between neurons \citep{Sporns2012,SotZal2019,mouseBrain,insectBrain}.
Networks arise in bioinformatics and omics in the form of interaction networks, in which the nodes correspond to genes or proteins, and edges encode co-expression or involvement in the same biological processes \citep{Alon2003,HakPinRobLov2008,BarGulLos2011,geneCoexp}.
Network data is common in economics as a way to describe trade among countries \citep{Chaney2014,SetZadBah2022,SosAreTor2024} or correlation structure among asset prices \citep{EllGolJac2014,Acemoglu2015,SosBet2023}.

In many settings such as those described above, networks are observed along with node-level covariates that may or may not be related to network structure.
In social media data, user-level features are typically available in the form of user post histories or basic demographic and background information.
Nodes in co-authorship networks are typically endowed with biographical information, current or previous institutional affiliations, and authors' research interests.
In connectome data, brain regions may be accompanied by cortical thickness measurements \citep[see, e.g.,][]{cortical} or transcriptome data \citep{brainTranscriptome}.
Nodes in omics interaction networks may be accompanied by annotations or known auxiliary data about what traits or processes a gene or protein is implicated in \citep{MKSCOFB02}.
Trade networks may have node covariates in the form of country-level data such as GDP, demographic makeup, geographic information or membership in treaty organizations.

When such node-level data is available, a natural question concerns whether or not node covariates are associated with network structure.
These covariates may drive network structure, as in social networks in which users who attend the same college or who share hobbies are more likely to form friendships.
In the other direction, there are settings where network structure may drive covariates, such as in the presence of contagion.
A fundamental task for network researchers is to identify which, if any, node-level features are predictive of network structure or vice versa.

Let us consider a specific example: In a social media setting, suppose that we observe which users are connected via an undirected friendship relation.
In addition to the network structure, suppose that we observe many features for each user (e.g., age, gender, interests, and so on).
When modeling this network and its relation to node features, if we care about interaction terms or we are unsure about how features relate to network structure (e.g., a linear or non-linear relationship), we might want to include higher-order interaction terms and series expansions of node-level attributes.
Such a feature vector can quickly grow to intractable size.
For social media analysis, it is clearly of interest to know to what extent the network constitutes new information: if node features explain the network structure well, we can potentially ignore the network and restrict our attention to node features when performing downstream tasks such as node classification.
If the covariates are predictive of network structure, this may also help us in link prediction tasks or to understand a new vertex for which we only observe the features.
On the other hand, if the network is not well explained by the features, the network structure constistutes additional information that must be taken into account in modeling and analysis.

Of course, beyond the prediction setting just sketched, the relation between network and node features may also be of scientific interest in and of itself.
In the same social media setting, for example, it might be of interest to social scientists to understand whether connections between social media accounts are related to, for example, the topics or pages that users engage with.
In protein-protein interaction data, scientists might be interested in whether or not interactions of certain proteins are related to their (known) functions or other annotation information.
In studying trade networks, it may be of interest to understand which country-level features are related to trade or other dyadic relations among nations.

Abstracting the above examples, our goal in this work is to assess whether or not observed node-level features explain network structure (or vice versa).
We approach this problem by assuming that the observed network structure is driven by low-dimensional latent variables associated to the nodes.
If these latent variables were known, we could directly assess whether or not they are associated with the observed node-level covariates.
Of course, in practice, latent network structure is not observed and must be estimated.
A second challenge arises from the possibility that node-level covariates may be high-dimensional, complicating the process of assessing which, if any, of them are predictive for the network structure.
Handling both of these complications is the main technical problem addressed in this paper.

We approach the network dependency testing problem by supposing that the network is generated via a random dot product graph \citep[RDPG;][]{AthFisLevLyzParQinSusTanVogPri2018}, a latent space network model in which latent node-level vectors, termed {\em latent positions}, drive network structure.
The RDPG includes the widely-used stochastic blockmodel and its variants \citep[e.g.,][]{HolLasLei1983,DCSBM,MMSBM,GaoMaZhaZho2018,JinKeLuo2024} as special cases, and is the restriction of the broader class of graphons to the case of finite spectra \citep[see, e.g.,][for details]{Lei2021,Lovasz2012}.
We then aim to ask how the observed features relate to the unobserved latent positions.
To this end, we suggest two approaches, one model-based and the other model-free.
In the model-based approach, we relate the features and the latent positions through a high-dimensional linear model (see Section~\ref{sec:setup}), using ridge and LASSO regularization to deal with high-dimensional node covariates.
In the model-free approach, we do not make any explicit modeling assumptions on the relation between latent positions and node covariates, and instead use canonical correlation analysis (CCA) between the node features and the latent positions to detect potential dependency.
While we focus in the present work on the RDPG, we anticipate that our results, particularly those related to our CCA-based methods, can be extended to a broader class of node-exchangeable network models, such as those in \cite{HofRafHan2002,Lovasz2012,Lei2021}.

\paragraph{Contribution}
This paper contributes to the study of network dependency testing.
We introduce four basic methods for assessing whether network structure is correlated with node-level covariates, under the assumption that the network is generated from a low-rank latent space model.
Two of these methods are based on linear regression (see Section~\ref{subsec:mod_approach}).
The other two are based on CCA (see Section~\ref{subsec:cca}).
Our methods are general enough to incorporate high-dimensional node-level features, either via regularization or via high-dimensional or sparse CCA methods, and since they rely only on matrix-vector products and leading singular values, they are computationally efficient compared to previous methods for this task.

\subsection{Prior Work}

Incorporating node-level covariates into network models has long been of interest to the statistical network analysis community.
A prominent line of work in this direction concerns network regression, in which the aim is to use network structure, be it latent or observed, to predict node-level respones.
For example, \cite{AchAgtTroParPri2023} % https://arxiv.org/pdf/2305.02473, https://appliednetsci.springeropen.com/articles/10.1007/s41109-023-00598-8
posit a network model in which the latent positions of an RDPG lie on a manifold.
The estimated latent positions are used to predict node-level responses.
Similar methods aimed at the network regression setting have been pursued previously in the literature \citep[see, e.g.,][and references therein]{LiLevZhu2019,HayLev2024,WanPowSwePau2024,HayFreLev2025}.
Unlike the present setting, these models make the assumption that the node-level covariates are responses in a regression and are thus inherently low-dimensional, obviating the possibility of applying these methods to the high-dimensional setting without substantial modification.

Much recent work in network regression has focused on causal inference \citep[see, e.g.,][]{McFSha2023,ChaPau2024,NatWarPau2025,HayFreLev2025} wherein we observe responses at each of the vertices and aim to assess causal effects of node- or network-level treatment.
We stress that while the present work is related to network regression, we do not aim to answer causal questions here, and our methods are largely agnostic as to the direction of causal influence between the node-level covariates and the network structure.
Similarly, while recent work has explored the identifiability issues that arise when both contagion and homophily are present in networks \citep{ShaTho2011}, we are not concerned here with determining which of these two dynamics are at play.
Rather, our goal in the present work is only to test whether or not node-level covariates are associated with network structure, rather than determining which drives the other. 

Many methods have been developed over the past fifteen years aiming to leverage node covariates to improve estimation of latent network structure such as community memberships, latent positions or graphon structure \citep{XuKeWanCheChe2012,YanMcALes2013,ZhaLevZhu2016,BinVogRoh2017,SuWonLee2020,ChaOlhWol2022,JamYuaGayArr2024}.
Since they are focused on estimation of latent network structure, these methods do not directly yield a way to assess whether node features are predictive of (or predicted by) network structure.
Indeed, most of these methods make the implicit assumption that some or all observed node covariates are related to latent network structure.
A related line of work relates latent variables to observed node covariates via a parametric model \citep{PerWol2013,YanJiaFieLen2018,WanPauBoe2023,LiXuZhu2024,ZhaNiu2025}.
As a side effect of their strong modeling assumptions, it is possible that these methods can be modified to test for whether or not covariates are related to latent network structure, either by leveraging asymptotic distributional results \citep{PerWol2013,YanJiaFieLen2018,LiXuZhu2024} or via computationally intensive Bayesian inference \citep[e.g.,][]{WanPauBoe2023,ZhaNiu2025}.
To the best of our knowledge, however, none of the above-cited papers apply their methods to network dependency testing.

\cite{ZhaXuZhu2022} proposed a model in which node-level latent variables drive both network structure and the observed node covariates via two different generalized linear models (GLMs).
They developed a method for estimating the latent variables and the GLM parameters and demonstrated the utility of their model in link prediction for missing entries of the adjacency matrix.
Their proposed model is similar to ours in its use of node-level latent variables, but the presence of two models, one for the network and one for the node covariates, renders their approach ill-suited to the task of directly testing dependence between node features and network structure.

Most similar to the present work,
\cite{FosHof2015} %https://www.tandfonline.com/doi/full/10.1080/01621459.2015.1008697#d1e93
considered network dependency testing under a Hoff-style latent space model \citep{HofRafHan2002,Hoff2005}.
They consider a setting in which node-level latent positions and observed covariates are jointly normally distributed.
Under the null hypothesis, the covariance between the latent positions and observed covariates is zero.
Our work follows a similar intuition, albeit under a different, much broader class of models (see Sections~\ref{subsec:mod_approach} and~\ref{subsec:cca} below).
Our methods apply under much weaker assumptions, relying on permutation tests rather than distributional assumptions, and do not require computationally intensive sampling methods. Moreover, we allow for high-dimensional covariates and provide theoretical guarantees even for methods that use estimated latent positions rather than the true latent positions in a na\"{i}ve plug-in approach.

The other previous work closest to the methods presented here is due to \cite{LeeShePriVog2019}, who proposed to assess network dependency by applying a diffusion maps embedding to the observed network and testing for dependency between these embeddings and the node covariates using distance correlation \citep{SzeRiz2013,SzeRiz2014}.
Among the key assumptions in their work is that the network is generated according to a node-exchangeable model \citep{Aldous1981,Lovasz2012,OrbRoy2015}.
As such, our assumption in the present work that the network be generated from a random dot product graph is only a slight restriction of this assumption \citep[i.e., the equivalent of imposing spectral decay conditions on the network-generating graphon; see][]{Lei2021}.
Our CCA-based methods, introduced in Section~\ref{subsec:cca} below, make similar modeling assumptions to those in \cite{LeeShePriVog2019}, but we avoid the computational and analytical complexity introduced by choosing the diffusion map parameter $t$ and constructing the distance correlation statistic, which requires computing all pairwise distances among vertices' embeddings and node covariates.

Elsewhere in the literature, there is also an interest in studying the relation between network structure and edge-level covariates, such as in \cite{CBC24}.
In that work, the network latent positions are used to capture the structure that is {\em not} explained by the edge-level covariates.
In a related line of work, \cite{MMHCAP22} considered a block structure as residual randomness after controlling for assortativity derived from a categorical vertex covariate.
That is, on top of the block structure, two vertices are more likely to form an edge when they share the same covariates.
Broadly similar ideas were pursued in \cite{LatRobOua2018} under the graphon, motivated by goodness-of-fit testing. 
While in principle these modelling strategies can be used to develop a test for the assortative impact of a single categorical covariate in a block model, their interest lies in community detection, and dependency testing is not pursued.

More broadly, the question of whether and when node-level attributes are helpful for inference has been explored in other areas.
For example, \cite{LevPriLyz2020} characterized the conditions under which node-level covariates are or are not beneficial in the context of vertex nomination problems \citep{MarPriCop2011,Coppersmith2014,FisLyzPaoChePri2015}.
The importance of incorporating node covariates has also been highlighted in the related graph matching literature \citep[see, e.g.,][]{LiJohSusPriLyz2024}.

\subsection{Roadmap and Notation}

The remainder of the paper is organized as follows:
In Section~\ref{sec:setup}, we give a detailed description of the network dependency problem and our proposed methods.
In Section~\ref{sec:results}, we present our main theoretical results pertaining to these novel methods.
Section~\ref{sec:expts} contains experiments demonstrating the performance of our methods on both simulated and real-world data.
We conclude with a brief summary and discussion of future research questions in Section~\ref{sec:conclusion}.

Before proceeding to present our methods, we pause to establish notation.
We write $[n]=\{1,2,\dots,n\}$ to denote the integers $1$ through $n$.
We write $I$ for the identity matrix and $J$ for the matrix of all ones, with the sizes of these matrices clear from the context.
For any matrix $M\in\IR^{m\times n}$, we denote by $M_{ik}$ its $(i,k)$-th entry, by $M_{\cdot k}$ the $k$-th column of $M$, and by $M_{i\cdot}$ the $i$-th row of $M$.
We write $M_i=M_{i\cdot}^\top $ to denote the $i$-th row of $M$ taken as a column vector.
For a vector $a\in\IR^p$ and a set of indices $S\subseteq[p]$,  we write $a_S\in\IR^p$ to denote a vector with $(a_S)_j=0$ for $j\notin S$ and $(a_S)_j=a_j$ for $j\in A$.
We let $\|a\|_p=\left( \sum_{i=1}^n|a_i|^p \right)^{1/p}$ denote the $p$-norm of vector $a$.
In the case of the Euclidean norm, $p=2$, we omit the index: $\|a\|=\|a\|_2$.
Using this notation, we denote by $\|M\|_{2,\infty}=\max_{i\in[m]}\|M_i\|$ the maximum of the Euclidean norms of the rows of $M$, better known as the $(\tti)$-norm \citep{CapTanPri2019}.
We write $\|M\|_F$ to denote the Frobenius norm of $M$ and $\|M\|$ to denote its spectral norm.
We use $\bbO_d$ to denote the space of all $d$-by-$d$ orthogonal matrices.
We make standard use of Landau notation for asymptotics, writing $f(n) = O( g(n) )$ to denote that $f(n)/g(n)$ is bounded by a constant as $n \rightarrow \infty$, $f(n) = o( g(n) )$ to denote that $f(n)/g(n) \rightarrow 0$ as $n\rightarrow \infty$ and $f(n) = \Omega( g(n) )$ to denote that $g(n) = O( f(n) )$.
We write $f(n) = \Theta( g(n) )$ if both $f(n) = O( g(n) )$ and $f(n) = \Omega( g(n) )$.
We also use standard probabilistic analogues of these.
For example, for a sequence of random variables $Y_n$, we write $Y_n = \op{ 1 }$ to indicate that $Y_n$ converges to zero in probability.

\section{Setup and Methodology} \label{sec:setup}

We assume that our data takes the form of a network on $n$ vertices with adjacency matrix $A \in \IR^{n \times n}$, in which each vertex $i=1,2,\dots,n$ has an associated feature vector $Z_i \in \IR^p$.
We will allow throughout that $p$ depends on $n$ (possibly even $p>n$).
Exact restrictions on their relation are provided in the individual theorems.
The network may be either binary, in which case all entries of $A$ are in $\{0,1\}$, or weighted, in which case the entries of $A$ are elements of $\IR$, typically but not necessarily non-negative.
For the sake of simplicity, we assume that the observed network is undirected, so that $A = A^\top$, though most of our results presented below can be extended straightforwardly to the case of a directed network.
We make no assumption about the presence or absence of self-edges (i.e., the diagonal entries of $A$), as the diagonal elements of $A$ have no influence on the asymptotic behavior of the estimators that we consider below.

We model the observed network as coming from a random dot product graph \citep[RDPG;][]{AthFisLevLyzParQinSusTanVogPri2018}, a low-rank latent space network model in which each vertex $i \in [n]$ has an associated {\em latent position} $X_i \in \IR^d$, drawn i.i.d.~according to some distribution $F$ on $\IR^d$.
Conditional on these latent positions, the edges of the network are generated independently in such a way that
\begin{equation} \label{eq:simple_RDPG}
\left\{ A_{ij} - X_i^\top X_j : 1 \le i < j \le n \right\}
\end{equation}
is a collection of independent mean-zero random variables.
When this is the case, we write
\begin{equation*}
(A,X) \sim \RDPG( F,n ),
\end{equation*}
where we have collected the latent positions $X_1,X_2,\dots,X_n \in \IR^d$ in the rows of $X \in \IR^{n \times d}$ for notational simplicity.
Most typically under the RDPG \citep[see, e.g.,][]{SusTanFisPri2012}, $F$ is such that the latent positions obey $0 \le X_i^\top X_j \le 1$, and we take $A_{ij} \sim \Bern( X_i^\top X_j )$.
We stress again that the setting considered here is more general: we require only that the elements of $A-XX^\top$ be independent (up to symmetry) and mean zero conditional on $X$.
We do not necessarily make any particular distributional assumptions on the edges beyond standard moment decay conditions (see Section~\ref{sec:results} for further discussion of this point).
We assume throughout that the dimension $d$ of the latent space is constant with respect to $n$, though we expect that many of our results can be extended to the case where $d$ is growing suitably slowly with $n$ \citep[see, e.g., the bounds in][]{LLL22}. 
Finally, while we assume here that the observed network is distributed as an RDPG, we anticipate that our results can be extended to related models, namely any suitably-structured graphon such as the generalised RDPG \citep{GRDPG} or the graph root distribution \citep{Lei2021}, albeit at the expense of more extensive bookkeeping.

To model association (or lack thereof) between node features and network structure, we assume that the latent positions and the node-level features are drawn i.i.d.~according to a joint distribution on $\IR^d \times \IR^p$.
That is, we assume that $\{ (X_i,Z_i) : i=1,2,\dots, n \}$ is a collection of i.i.d.~random variables.
If the covariance $\Cov( X_i, Z_i ) \in \IR^{d \times p}$ is non-zero, this indicates that node-level features and (latent) network structure are related.
For ease of notation, we will always assume that $Z_i$ contains an intercept term.
Just as we collect the latent positions into the rows of $X \in \IR^{n \times d}$, we collect the node attributes into the rows of $Z \in \IR^{n \times p}$.
We remind the reader that our goal in this work is simply to test association between network structure and node-level features.
As discussed in Section~\ref{sec:intro}, we do not wish to make any assumptions in the present work about the direction of causal influence between the features and the latent positions.
Rather, our only assumption is that the network (as encoded by $A$) and the node-level features $Z_1,Z_2,\dots,Z_n$ are independent conditional on $X$.

\begin{remark} \label{rem:ZfromA}
In many settings, it may be appealing to allow the node covariates to depend directly on $A$, rather than on the latent network structure encoded by $X$.
For example, the feature vector $Z_i$ may include the degree $d_i = \sum_j A_{ij}$ of vertex $i \in [n]$.
Alternatively, we may wish to model node-level features as arising from a contagion process, whereby the features $Z_i$ of vertex $i$ depend in some way on the average features of its neighbors, e.g., $Z_i=\sum_j A_{ij} Z_j / d_i$.
Under the RDPG model assumed above, we anticipate that many of these settings can be treated by the methods that we develop below.
For example, the tools developed and deployed in \cite{HayLev2024} and \cite{HayFreLev2025} show that we can often treat $A$ and $X X^\top$ as (approximately) interchangeable, a property that we use frequently in the sequel.
Unfortunately, these settings where features are allowed to depend directly on $A$ (rather than relating to $A$ only via $X$) require far more bookkeeping and analysis.
Thus, for the sake of space and simplicity, we leave this interesting setting for future work.
\end{remark}

\begin{remark} \label{rem:rotnonID}
It is well known \citep[see, e.g.,][]{AthFisLevLyzParQinSusTanVogPri2018}, and can be seen directly from Equation~\eqref{eq:simple_RDPG}, that for any orthogonal matrix  %$Q\in\IR^{d\times d}$
$Q\in\bbO_d$, the matrix of latent positions $XQ$ yields the same distribution over networks as does $X$.
Thus, the latent positions $X_1,X_2,\dots,X_n$ are identifiable only up to some unknown orthogonal transformation.
Interestingly, with minor exception in Section~\ref{subsec:mod_approach} below, this will be of no consequence in our results to follow, as this non-identifiability will typically ``cancel out''.
\end{remark}

\subsection{Model-Based Approach} \label{subsec:mod_approach}

In the first of our two broad approaches, we assume that the latent positions $X$ and the features $Z$ are related according to
\begin{equation} \label{eq:model}
X=ZB+E,
\end{equation}
where $B \in \IR^{p \times d}$ is an unknown matrix of coefficients, $Z\in\IR^{n\times p}$ is the matrix of features with independent rows,
$E\in\IR^{n\times d}$ is a noise matrix with independent mean-zero rows, generated independently of %$B$ and
 $Z$.
In other words, for $k\in[d]$, we assume that the $k$-th latent dimension $X_{\cdot k}$ is related to the features $Z$ via a linear model with coefficient vector $B_{\cdot k}$.
We stress that in this model, we assume independent latent positions (inherited from the independence structure of $Z$ and $E$), but this does not preclude the fact that edges are independent only {\em conditionally} on $X$.
Moreover, the columns of $Z$, $E$, and $X$ may be dependent.
Our goal is to estimate $B$, which will in turn tell us which features (i.e., columns of $Z$), if any, explain the network structure encoded by $X$.
In particular, non-zero rows of $B$ correspond to columns of $Z$ that are informative about network structure.
Said another way, for a vertex $i\in[n]$, if for dimension $k \in [p]$ the row $B_{k,\cdot}^\top \in \IR^d$ is equal to zero, then this indicates that the $k$-th feature $Z_{ik}$ is irrelevant to explaining the latent positions.
That is, feature $k \in [p]$, is unrelated to the network structure.
Similarly, if a column $B_{\cdot k} \in \IR^p$ for $k\in[d]$ is equal to zero, this indicates that the $k$-th coefficient of the latent position $X_{ik}$ is not explained by any of the observed node covariates.
Thus, to understand the relation between the features and the network under this model, it suffices to estimate $B$.

As discussed in Remark~\ref{rem:rotnonID}, for any orthogonal transformation %$Q \in \IR^{d \times d}$
$Q\in\bbO_d$, the matrices $X$ and $XQ$ give rise to the same distribution over networks, so that we cannot hope to distinguish $X$ from $XQ$ based only on the adjacency matrix $A$.
Analogously, right-multiplying Equation~\eqref{eq:model} by $Q$ implies that we can only hope to estimate $BQ$ for some unknown $Q$.
This non-identifiability is typically not an issue for downstream applications \citep[see discussion in][]{AthFisLevLyzParQinSusTanVogPri2018}.
Unfortunately, the matrix $Q$ may obfuscate sparsity patterns present in $B$ unless an entire row of $B$ equals zero.
Therefore, in the high-dimensional setting, we require sparsity in the sense that most rows of $B$ are equal to zero, i.e., that the corresponding node feature is irrelevant for explaining the latent positions.
One way to encourage this sparsity is to split Equation~\eqref{eq:model} into $d$ separate regressions, one for each latent dimension, and encourage sparsity in each of the $d$ regressions.
That is,
\begin{equation} \label{eq:separateregs}
  X_{\cdot k}=Zb^{(k)}+E_{\cdot k} ~\text{ for }~k=1,2,\dots,d,
\end{equation}
where $b^{(1)},b^{(2)},\dots,b^{(d)}\in\IR^p$ correspond to the columns of $B$ encoding the regression coefficients.
The coefficient vector $B_j\in\IR^d$ corresponding to the $j$-th node covariate is then given by $B_j=(b_j^{(1)},b_j^{(2)},\dots,b_j^{(d)})^\top$.

We perform the regression in Equation~\eqref{eq:separateregs} for each of the $d$ dimensions of the latent positions.
To deal with high-dimensional node features, we consider using either ridge regression or LASSO, which we discuss in detail below in Sections~\ref{subsubsec:ridge} and~\ref{subsubsec:LASSO}, respectively.
Of course, since we do not have access to the latent positions $X$ directly, we must plug in estimates $\Xhat$ for them.
This estimation can be achieved using, for example, the adjacency spectral embedding \citep[ASE;][]{SusTanFisPri2012} or related methods \citep{XieXu2020,XieXu2021,WuXie2022,LLL22}.
This yields estimates
\begin{equation} \label{eq:model-based:separateregs}
\bhat_r^{(k)} =
\argmin{b\in\IR^p}\frac{1}{n}\left\|\Xhat_{\cdot k}-Zb\right\|^2
+ \lassoreg_k\|b\|_r^r , \text{ for } k\in[d] ,
\end{equation}
where $r=1$ corresponds to the regular LASSO and $r=2$ corresponds to ridge regression,
and $\lassoreg_1,\lassoreg_2,\dots,\lassoreg_d \in [0,\infty)$ are regularization parameters.
We collect the estimates $\bhat_r^{(1)},\bhat_r^{(2)},\dots,\bhat_r^{(d)}$ in the columns of a matrix $\Bhat_r$, that is, $\Bhat_{r,ik}=\bhat_{r,i}^{(k)}$.
Moreover, we write for easier recognition
\begin{equation*}
\Bhat^{\ridge}:=\Bhat_2~\text{ and }~\Bhat^{\rLASSO}:=\Bhat_1.
\end{equation*}
If $B$ is indeed row-sparse, then the $d$ separate regressions in Equation~\eqref{eq:separateregs} share the same sparsity structure, which can be promoted directly using group LASSO or multi-task learning \citep[see, e.g, Chapters 4 and 8 of][]{vdGB11}. 
As such, we might instead consider the estimator
\begin{equation} \label{eq:def:BhatgLASSO}
\Bhat^{\gLASSO}
= \begin{bmatrix}
\Bhat_{\gLASSO,1}^\top \\ \Bhat_{\gLASSO,2} \\ \vdots \\ \Bhat_{\gLASSO,p}^\top
\end{bmatrix}
=\argmin{\beta\in\IR^{p\times d}}
	\frac{1}{nd}\left\|\Xhat-Z\beta\right\|_F^2
	+\frac{\lassoreg}{\sqrt{d}}\sum_{j=1}^p\|\beta_j\|,
\end{equation}
where $\lassoreg \ge 0$ is a regularization parameter.

For selection of the tuning parameters $\lassoreg_1,\lassoreg_2,\dots,\lassoreg_d$ in Equation~\ref{eq:model-based:separateregs} or $\lassoreg$ in Equation~\ref{eq:def:BhatgLASSO}, we use cross-validation \citep{HTW15}.
Note that a ridge penalty of the type $\sum_{j=1}^p\|\beta_j\|^2$ corresponds to the case $\lassoreg_1=\lassoreg_2=\cdots=\lassoreg_d$ in the separate regression and will therefore not be treated separately.
If regularization is not required by the user, it is of course possible to set $\lassoreg=0$ in all estimators (in fact, all these estimators are identical in this case).
We study this case when we analyze ridge regression in Section~\ref{subsubsec:ridge} below.

\begin{remark} \label{rem:discussion_model_based_approach:dim_selection}
The estimation procedure used to produce $\Xhat$, as used above in Equations~\eqref{eq:model-based:separateregs} and~\eqref{eq:def:BhatgLASSO}, is well-studied in the literature \citep{SusTanFisPri2012,AthFisLevLyzParQinSusTanVogPri2018,LLL22}, but we note that it requires knowledge of the dimension $d$.
Theoretically sound dimension selection is widely recognized as one of the major open problems in latent space network models, but many heuristic approaches exist that work well in practice \citep[see, e.g.,][]{ZhuGho2006,HanYanFan2019,LiLevZhu2020,CheRocRohYu2021} and have a variety of performance guarantees. We state our results below under the assumption that $d$ has been selected correctly, in light of the fact that many of these methods come with theoretical guarantees that the correct dimension is selected with high probability for suitably large networks.
\end{remark}

We prove in Section~\ref{sec:results} that both ridge regression and LASSO are suitable for estimation of $B$, even though we use $\Xhat$ instead of $X$ as the response in the regression.
Note finally that considering Equation~\eqref{eq:model} as a model rather than the true data generating process is conceptually possible in the same situations as for regular linear models (i.e., outside of a network context).
To illustrate this point, consider the case of additional, unobserved, structured terms.
That is, suppose that we fit the model in Equation~\eqref{eq:model} while the true data generating process has the form
\begin{equation} \label{eq:long_regression}
X=ZB+T+E,
\end{equation}
where $T$ is an unobserved term.
Since $Z$ is assumed to contain an intercept, we may assume that $T$ is mean zero.
As in standard regression, whether or not $B$ can be estimated from $X$ and $Z$ alone depends on the relation between $T$ and $Z$.
Figure~\ref{fig:DAG} shows directed acyclic graphs (DAGs) for typical situations.
Not observing $T$ will not introduce a bias in the estimation of $B$ if $T$ is independent of $Z$ (see Figure~\ref{fig:DAG1}).
Hence, the short regression in Equation~\eqref{eq:model} gives an unbiased estimate of $B$.
If $T$ can be modeled as a function of $Z$, as in Figure~\ref{fig:DAG2}, the meaning of matrix $B$ in the short regression Equation~\eqref{eq:model} would differ from its meaning in Equation~\eqref{eq:long_regression}: it comprises the direct effect of $Z$ on $X$ (as captured by $B$ in \eqref{eq:long_regression}) and the indirect effect of $Z$ on $X$ through $T$.
We do not consider this problematic because, since $T$ is unobserved, one might arguably care about the overall effect only, and this can be estimated. 
A problem would arise, however, if $T$ confounds $Z$ and $X$, as in Figure~\ref{fig:DAG3}.
In this case, the model in Equation~\eqref{eq:model} cannot be used to give an accurate estimate of $B$ in Equation~\eqref{eq:long_regression}.
While our goal here is not causal inference, we note that this is a classical problem in the causal inference literature.
To solve it, one would require, for example, an instrument.
We do not pursue this further here, but assume instead that we are in a setup where a model like that in Equation~\eqref{eq:model} can be reasonably deployed.

\begin{figure}
\centering
\begin{subfigure}{0.3\textwidth}
\includegraphics[width=\textwidth]{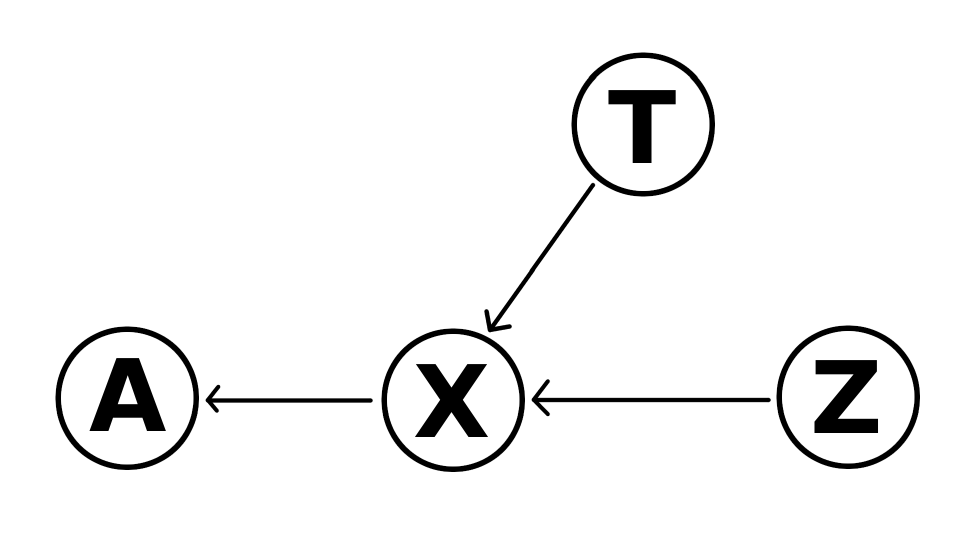}
\caption{$T$ is marginally independent of $Z$}
\label{fig:DAG1}
\end{subfigure}
\begin{subfigure}{0.3\textwidth}
\includegraphics[width=\textwidth]{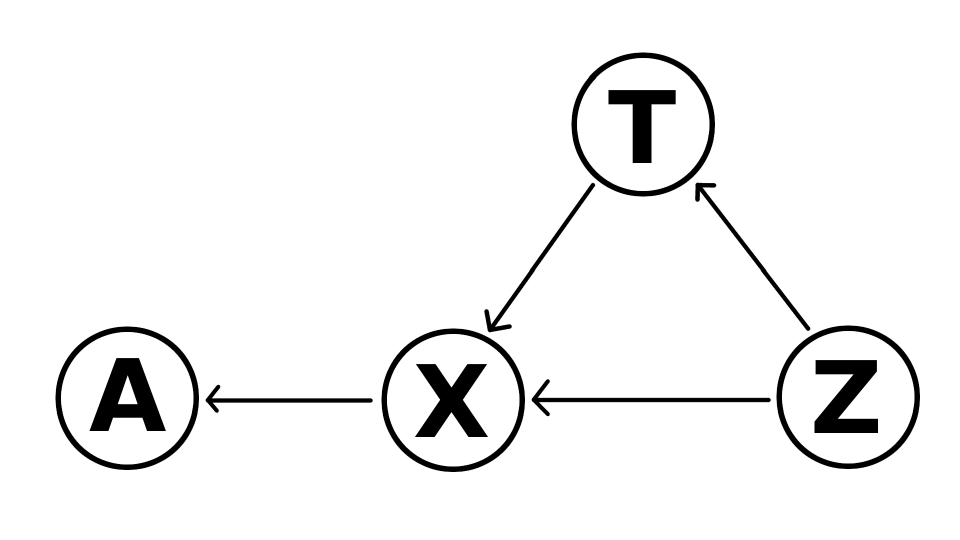}
\caption{ $T$ is a mediator between $X$ and $Z$}
\label{fig:DAG2}
\end{subfigure}
\begin{subfigure}{0.3\textwidth}
\includegraphics[width=\textwidth]{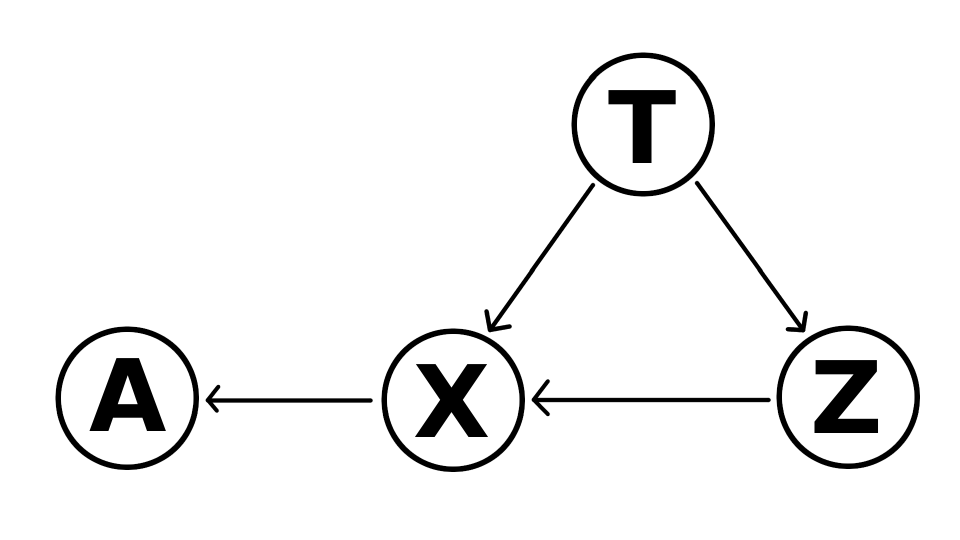}
\caption{$T$ confounds $X$ and $Z$}
\label{fig:DAG3}
\end{subfigure}
\caption{Graphs describing the different ways how $T$ can impact the variables of interest. A node in the above graphs is a function of its parents and independent noise. }
\label{fig:DAG}
\end{figure}

\subsection{Model-Free Approach} \label{subsec:cca}

In many settings, we may wish to avoid assuming an explicit relation between $X$ and $Z$ as in Equation~\eqref{eq:model}, but we may still be willing to believe that the network $A$ is generated according to a low-rank model as in Equation~\eqref{eq:simple_RDPG}.
If the latent positions $X_1,X_2,\dots,X_n \in \R^d$ were observed, a classical model-free way of assessing the presence of a relation between the latent network structure and the observed features $Z_1,Z_2,\dots,Z_n$ would be canonical correlation analysis \citep[CCA;][]{Hotelling1936}.
We measure the association between the latent positions and the node covariates according to how well we can correlate their projections onto a common one-dimensional space:
\begin{equation} \label{eq:def:CCA:classical}
\begin{aligned}
\rho_{X,Z} = \max_{u \in \IR^{d}} \max_{v \in \IR^{p}}~
& \frac{1}{n} \sum_{i=1}^n u^\top X_i v^\top Z_i
	- \left( \frac{1}{n} \sum_{i=1}^n u^\top X_i \right)
	\left( \frac{1}{n} \sum_{i=1}^n v^\top Z_i \right) \\
&~~~\text{ subject to }~ \\
\frac{1}{n} & \sum_{i=1}^n \left( u^\top \! X_i \right)^2
	\!-\! \left(  \frac{1}{n} \sum_{i=1}^n u^\top \! X_i \right)^2
	\!=
	\frac{1}{n} \sum_{i=1}^n \left( v^\top \! Z_i \right)^2
	\!-\! \left(  \frac{1}{n} \sum_{i=1}^n v^\top \! Z_i \right)^2
	\!= 1 .
\end{aligned} \end{equation}
In other words, we aim to find directions $u$ and $v$ in $\IR^d$ and $\IR^p$, respectively, along which our two data sets are maximally correlated.
The optimal choices of $u$ and $v$ correspond to the {\em first canonical directions} for $X$ and $Z$.
The resulting value of the optimization problem is the CCA alignment score or, as we will often call it in the sequel, the {\em CCA coefficient}.
If this alignment score is suitably far from zero, then this constitutes evidence that the latent positions and the node-level features are related.
We note that the CCA alignment in Equation~\eqref{eq:def:CCA:classical} can be extended to allow for multiple canonical directions by choosing collections of orthonormal vectors $u_1,u_2,\dots,u_r \in \IR^d$ and $v_1,v_2,\dots,v_r \in \IR^p$.
We largely focus here on the case of one canonical direction for the sake of space and simplicity.
That said, our theoretical results established below, which draw from tools in the subspace perturbation literature \citep[see, e.g.,][]{CaiZha2018,CapTanPri2019}, apply also to this more general multi-dimensional setting.

As has been observed elsewhere \citep[see, e.g.,][]{Muirhead1982,SchMul2000,KesLev2023}, the CCA problem in Equation~\eqref{eq:def:CCA:classical} can be recast in terms of finding the singular values of a matrix:
\begin{equation} \label{eq:def:CCA:projection}
\rho_{X,Z} 
= \max_{u \in \IR^d, \|u\|=1} \max_{v \in \IR^p, \|v\|=1}
u^\top \Sigmatilde_X^{-1/2} \Sigmatilde_{X,Z} \Sigmatilde_Z^{-1/2} v ,
\end{equation}
where, writing
\begin{equation} \label{eq:def:Mn}
M_n = I - \frac{1}{n} J \in \IR^{n \times n}
\end{equation}
for the ``centering matrix'',
\begin{equation} \label{eq:def:SigmatildeMarginal}
\Sigmatilde_X = \frac{1}{n} X^\top M_n X \in \IR^{d \times d}
~\text{ and }~
\Sigmatilde_Z = \frac{1}{n} Z^\top M_n Z \in \IR^{p \times p}
\end{equation}
are the the sample covariance matrices of $X$ and $Z$, respectively,
and
\begin{equation} \label{eq:def:SigmatildeXZ}
\Sigmatilde_{X,Z} = \frac{1}{n} X^\top M_n Z \in \IR^{d \times p}
\end{equation}
is the sample covariance matrix between the latent positions $X$ and the node features $Z$.
We note that we include tildes on all three of these matrices to stress that they are finite-sample estimates of respective population quantities $\Sigma_X$, $\Sigma_{X,Z}$ and $\Sigma_Z$.
Thus, in light of Equation~\eqref{eq:def:CCA:projection}, we may equivalently examine the leading singular value(s) of
\begin{equation} \label{eq:def:CCAclassic}
\CCA(X,Z)
= \Sigmatilde_X^{-1/2} \Sigmatilde_{X,Z} \Sigmatilde_Z^{-1/2}
\in \R^{d \times p} .
\end{equation}

Of course, in practice, the latent positions $X$ are unobserved and thus cannot be used in the optimization in Equation~\eqref{eq:def:CCA:projection}.
To account for this, as with our model-based approaches presented in Section~\ref{subsec:mod_approach}, we will replace $X$ with an estimate $\Xhat$ and proceed with the CCA alignment of $\Xhat$ and $Z$ instead, %as outlined in Algorithm~\ref{alg:model-free:ASE}.
examining the singular values of
\begin{equation} \label{eq:def:CCA:ASE}
\CCA( \Xhat, Z )
= \Sigmatilde_{\Xhat}^{-1/2} \Sigmatilde_{\Xhat,Z} \Sigmatilde_Z^{-1/2} .
\end{equation}

The main challenge in obtaining a performance guarantee for the CCA matrix in Equation~\eqref{eq:def:CCA:ASE} arises from the fact that we have replaced the matrix of latent positions $X \in \IR^{n \times d}$ with estimates $\Xhat \in \IR^{n \times d}$.
As discussed previously in Remark~\ref{rem:rotnonID}, this comes at the cost of rotational nonidentifiability, in the sense that $\Xhat$ recovers $X$ only up to right-multiplication by some unknown orthogonal transformation.
The key technical challenge in our proofs lies in accounting for this orthogonal non-identifiability and showing that the resulting error between $\Xhat$ and a suitably transformed $X$ is not too large, so that the CCA alignment score of $\Xhat$ and $Z$ is close to that of $X$ and $Z$.
We will see that the orthogonal non-identifiability poses no particular challenge, as the CCA coefficient defined in Equation~\eqref{eq:def:CCA:projection} is ultimately a property of subspace alignment, and is thus invariant to orthogonal transformations of the rows of either $X$ or $Z$.

In building the matrix in Equation~\eqref{eq:def:CCA:ASE}, as in our model-based approaches, a user must select an embedding dimension, and Remark~\ref{rem:discussion_model_based_approach:dim_selection} applies here as well.
In the event that the node covariates are high-dimensional, a range of CCA methods exist for handling such data, including those designed to account for structure (e.g., sparsity) in the features.
In our theoretical results below, we assume that the dimension $p$ of the features $Z \in \IR^{n \times p}$ is fixed with respect to $n$.
We anticipate, however, that standard methods for sparse or high-dimensional CCA could be deployed here \citep[see, e.g.,][]{WitTib2009,SonSchRamHas2016,GaoMaZho2017} as could regularized covariance estimation methods \citep[see, e.g.,][and citations therein]{FisSun2011,LedWol2018,LeLevBicLev2020}.

The method just described has the disadvantage of needing to select the latent dimension $d$.
Since $X$ is not observed, an alternative to working with $X$ is to directly apply CCA between $A$ and $Z$, and examine the behavior of $\CCA(A,Z)$.
In so doing, we avoid having to choose a latent dimension $d$, but this comes at the cost of having introduced another high-dimensional feature: the $n$-dimensional rows of the adjacency matrix $A$.
As with estimating $\Sigmatilde_Z$, there are a range of methods available to handle the resulting high-dimensional covariance matrix.
We use a simple shrinkage estimate, and consider the singular values of the matrix
\begin{equation} \label{eq:def:CCAfullnet}
\CCA_\regu(A,Z)
= 
\left( \Sigmatilde_A + \regu I \right)^{-1/2} A^\top M_n Z \Sigmatilde_Z,
\end{equation}
where $\regu = \regu_n \ge 0$ is a regularization parameter to be specified below, $M_n$ is the centering matrix as in Equation~\eqref{eq:def:Mn}, and $\Sigmatilde_A$ is defined analogously to $\Sigmatilde_X$ and $\Sigmatilde_Z$ above,
\begin{equation*}
\Sigmatilde_A = \frac{1}{n} A^\top M_n A \in \IR^{n \times n} .
\end{equation*}
We will see in our results below that a wide range of values for $\regu$ result in satisfactory behavior of $\CCA_\regu(A,Z)$, meaning that this is a far less fraught model selection problem than selecting the embedding dimension.

\section{Main Results} \label{sec:results}

As mentioned in Section~\ref{sec:setup}, we consider a more general edge noise model than the ``classical'' RDPG \citep{AthFisLevLyzParQinSusTanVogPri2018}.
Rather than making explicit distributional assumptions on the edges in our network (e.g., Bernoulli edges), our theoretical results require only that the adjacency matrix and certain quantities derived from it concentrate about their population analogues at a suitably fast rate.
We state in Assumptions~\ref{assum:spectralRate},~\ref{assum:ttiRate} and~\ref{assum:subspaceRate} below the rates that are relevant for our discussion, and describe afterwards two settings in which these rates can be explicitly given.

\begin{assumption} \label{assum:spectralRate} %previously called A0
Conditional on $X$, the random network $A$ has expectation $XX^\top$ and obeys
\begin{equation} \label{eq:assum:spectralRate}
\left\| A - XX^\top \right\| = \Op{ \spectralRate_n }
\end{equation}
for some sequence of non-negative reals $(\spectralRate_n)_{n=1}^\infty$.
\end{assumption}

Since we do not observe $X$, our results rely on estimating $X$ (up to orthogonal non-identifiability, as discussed in Remark~\ref{rem:rotnonID}).
A number of methods for estimating $X$ are available, depending on the precise choice of model for how $A$ varies about $X$ \citep[see, for example,][]{RohChaYu2011,SusTanFisPri2012,LeiRin2015,LevLyz2017,LTAPP17,LevAthTanLyzYouPri2017,AthFisLevLyzParQinSusTanVogPri2018,XieXu2020,XieXu2021,WuXie2022,LLL22}.
Rather than choosing one of these particular methods or models, we formulate our results under more general assumptions on how well the estimate $\Xhat$ recovers the true latent positions $X$.

\begin{assumption} \label{assum:ttiRate} %previously called A1
We can produce a sequence of estimates $\Xhat\in\IR^{n\times d}$ of the latent positions $X \in \IR^{n \times d}$ such that there exists a sequence of random orthogonal matrices %$Q\in\IR^{d\times d}$
$Q\in\bbO_d$ satisfying
\begin{equation*}
\left\| \Xhat - X Q \right\|_{\tti}
= \Op{ \ttiRate_n }
\end{equation*}
for some sequence of non-negative reals $(\ttiRate_n)_{n\in\IN}$.
\end{assumption}

Our results for the model-free approaches presented in Section~\ref{subsec:cca} depend on a slightly stronger assumption on the edge distribution than that in Assumption~\ref{assum:spectralRate}.

\begin{assumption} \label{assum:subspaceRate}
The entries of $A-XX^\top$ are such that, for any sequence of $W,V \in \R^{n \times d}$ conditionally independent of $A-XX^\top$ given $X$ and having columns of bounded norm,
\begin{equation*}
\left\| W^\top \left( A - X X^\top \right) V \right\|
= \Op{ \subspaceRate_n }
\end{equation*}
for some sequence of non-negative reals $(\subspaceRate_n)_{n\in\IN}$.
\end{assumption}

The above assumptions concern how quickly the observed network concentrates about its expectation and how quickly an associated estimate of the latent positions $X$ recovers them (as measured in two-to-infinity norm) and how the error $A-XX^\top$ interacts with fixed subspaces.
Before proceeding, we mention a pair of results from the literature under two different network models and indicate the choices of $\spectralRate_n$, $\ttiRate_n$ and $\subspaceRate_n$ that they imply.

In the case of binary networks, consider the following simplification of Lemma 5 from \cite{LTAPP17}. %https://projecteuclid.org/journals/electronic-journal-of-statistics/volume-8/issue-2/Perfect-clustering-for-stochastic-blockmodel-graphs-via-adjacency-spectral-embedding/10.1214/14-EJS978.pdf
\begin{theorem}[\citet{LTAPP17}, Theorem 5] \label{thm:vince}
Let $(A,X)$ follow an RDPG with binary edges and suppose that $\IE X_1X_1^\top \in \IR^{d \times d}$ is full rank with $d$ distinct eigenvalues.
There exists a sequence of orthogonal matrices %$Q\in\IR^{d\times d}$
$Q\in\bbO_d$ such that
\begin{equation*}
\left\|\Xhat- X Q \right\|_{\tti} = \Op{ \sqrt{\frac{d}{n}} \log^2 n } .
\end{equation*}
\end{theorem}

This result lets us choose $\ttiRate_n= \sqrt{ d/n } \log^2n$ in Assumption~\ref{assum:ttiRate}.
Using Theorem 5.2 from \cite{LeiRin2015} and trivially upper bounding the row sums of $XX^\top$ by $n$, we may choose $\spectralRate_n = \sqrt{n}$ in Assumption~\ref{assum:spectralRate}.
Using standard concentration inequalities \citep{Vershynin2020}, we may choose $\subspaceRate_n = d \log n$.

In order to give these rates for weighted networks, as well as for many of our results in the sequel, we will make use of standard concentration inequalities for sub-gamma random variables \citep{BLM13,Vershynin2020}.

\begin{definition}
A centered random variable $Y$ is called {\em $(\nu,b)$-sub-gamma} for $\nu,b>0$ if for all $t \in\left(-\frac{1}{b},\frac{1}{b}\right)$,
\begin{equation*}
 \log\IE e^{t Y}\leq\frac{t^2\nu}{2(1-b|t|)}.
\end{equation*}
\end{definition}

Now, in the case of weighted networks, consider the following result, adapted from \cite{LLL22}.
\begin{theorem}[\citet{LLL22}, Lemma 5, Theorem 6] \label{thm:LLL} %https://jmlr.org/papers/v23/19-1056.html
Suppose that conditional on $X \in \IR^{n \times d}$, the entries of $A-XX^\top$ are independent $(\nu,b)$-subgamma random variables.
Then
\begin{equation*}
\left\| A - X X^\top \right\| = \Op{ \sqrt{(\nu + b^2) n} \log n }.
\end{equation*}
Suppose further that 
\begin{equation*}
(\nu+b^2) n \log^2 n = o\left( \lambda_d^2( XX^\top ) \right).
\end{equation*}
Then there exists a sequence of orthogonal matrices %$Q \in \IR^{d \times d}$
$Q\in\bbO_d$ such that
\begin{equation*}
\left\|\Xhat - X Q\right\|_{\tti}
\le \frac{ Cd\sqrt{ \nu + b^2 } \log n }{ \lambda_d^{1/2}( X^\top X ) }
+ \frac{ C dn (\nu+b^2) \kappa( X^\top X ) \log^2 n }
	{ \lambda_d^{3/2}( X^\top X ) },
\end{equation*}
where $\kappa( M )$ denotes the condition number of matrix $M$.
\end{theorem}

Thus, under the setting of Theorem~\ref{thm:LLL}, we may choose
\begin{equation*}
\spectralRate_n = \sqrt{\nu + b^2} \log n
\end{equation*}
in Assumption~\ref{assum:spectralRate} and
\begin{equation*}
\ttiRate_n= 
\frac{ Cd\sqrt{ \nu + b^2 } \log n }{ \lambda_d^{1/2}( X^\top X ) }
+ \frac{ C dn (\nu+b^2) \kappa( X^\top X ) \log^2 n }
	{ \lambda_d^{3/2}( X^\top X ) } 
\end{equation*}
in Assumption~\ref{assum:ttiRate}.
Using standard concentration inequalities \citep[see, e.g., Proposition 22 in][]{LLL22}, we may take
\begin{equation*}
\subspaceRate_n = d \sqrt{\nu + b^2} \log n
\end{equation*}
in Assumption~\ref{assum:subspaceRate}.

Some of our results below make use of the following additional assumption on the covariance structure of $X$ and $Z$.

\begin{assumption} \label{assum:invertible}
The $(X_i,Z_i)$ pairs are drawn i.i.d.~according to a distribution on $\IR^d \times \IR^p$ such that $(X_1,Z_1) \in \IR^{d + p}$ has covariance
\begin{equation} \label{eq:def:Sigmafull}
\Sigma = \begin{bmatrix} \Sigma_X & \Sigma_{X,Z} \\
			\Sigma_{X,Z}^\top & \Sigma_Z \end{bmatrix}
\in \IR^{(d+p) \times (d+p)},
\end{equation}
with both $\Sigma_X \in \IR^{d \times d}$ and $\Sigma_Z \in \IR^{p \times p}$ invertible.
\end{assumption}

\subsection{Model-Based Approaches}

We begin by establishing results for a pair of network association testing methods under parametric modeling assumptions.
In particular, we consider settings where the latent positions $X$ are related to the observed nodal covariates $Z$ via the linear model in Equation~\eqref{eq:model}.

\subsubsection{Ridge regression} \label{subsubsec:ridge}

Our model in Equation~\eqref{eq:model} assumes that the latent positions (i.e., the rows of $X \in \IR^{n \times d}$) are related to the node-level covariates (i.e., the rows of $Z \in \IR^{n \times p}$) via a linear transformation encoded by $B \in \IR^{p \times d}$.
Ignoring for now the fact that the latent positions are accessible only via the observed adjacency matrix $A$, this suggests regressing the latent positions against the nodal covariates.
While we might consider an estimation procedure for all $d$ columns of $B$ at once, we leave this matter for future work and instead consider the admittedly more na\"{i}ve approach of modeling the $d$ columns of $X$ via $d$ separate regressions:
\begin{equation} \label{eq:regression:column}
X_{\cdot, k} = Z b^{(k)} + E_{\cdot, k}~\text{ for }~k \in [d].
\end{equation}
Recall the ridge regression estimators $b_2^{(k)}$ from Equation~\eqref{eq:model-based:separateregs} with $r=2$.
For simplicity of notation, we denote these estimates in this section by $\bhat^{(k)}$.
These can be expressed as
\begin{equation} \label{eq:def:bhatk}
\bhat^{(k)} = \left( \frac{1}{n} Z^\top Z + \lassoreg_k I_p \right)^{-1}
                \frac{1}{n} Z^\top \Xhat_{\cdot, k},
\end{equation}
where for each $k \in [d]$, $\lassoreg_k \ge 0$ is a regularization parameter, which we allow to change for each column of $X \in \IR^{n \times d}$.
An immediate challenge arises from the fact that we do not have access to $X$, except through $A$.
In addition, as discussed in Remark~\ref{rem:rotnonID}, $X$ is identified only up to an orthogonal transformation $Q$.
If the true $X$ were available, we could obtain the following oracle estimators for any orthogonal matrix $Q$
\begin{equation} \label{eq:def:btildek}
\btilde^{(k)}_Q = \left( \frac{1}{n} Z^\top Z + \lassoreg_k I_p \right)^{-1}
                \frac{1}{n} Z^\top (XQ)_{\cdot, k}.
\end{equation}
The following result shows that the estimates $\bhat^{(k)}$ based on the estimated latent positions $\Xhat$ are close to the oracle estimates $\btilde^{(k)}_Q$, once we choose $Q$ so as to account for the orthogonal non-identifiability in the latent space.

\begin{theorem} \label{thm:ridge:convergence}
Suppose that $(A,X) \sim \RDPG(F,n)$ with the distribution $F$ corresponding to the model in Equation~\eqref{eq:regression:column}.
In addition to Assumptions~\ref{assum:spectralRate} and~\ref{assum:ttiRate}, suppose that the nodal covariates $Z$ are such that for all $i \in [n]$ and $j \in [p]$, $Z_{ij}$ is a $(\nu_Z,b_Z)$-subgamma random variable with $(\nu_Z+b_Z^2)\log (p \vee n) = O( n )$.
Then there is a sequence of random orthogonal matrices $Q=Q_n \in \bbO_d$ such that
\begin{equation*}
\max_{k \in [d]} \left\| \bhat^{(k)} - \btilde^{(k)}_Q \right\|
= \Op{ \frac{\ttiRate_n}{\min_{k \in [d]} \lassoreg_k } }.
\end{equation*}
\end{theorem}
A proof can be found in Appendix~\ref{apx:ridge}.

\begin{corollary} \label{cor:ridge:lassoreg}
Suppose that the ridge regularization parameters $\lassoreg_1,\lassoreg_2,\dots,\lassoreg_d$ are chosen according to 
\begin{equation*}
\lassoreg_k = \frac{ \| \btrue^{(k)} \|^2 }{ \gamma \sigma_k^2 },
\end{equation*}
where $\gamma = \lim_{n\to\infty} p/n \in (0,\infty)$ and $\sigma_k^2=\Var E_{1k}$.
Then, with $Q$ as in Theorem~\ref{thm:ridge:convergence},
\begin{equation*}
\max_{k \in [d]} \left\| \btilde^{(k)}_Q - \bhat^{(k)} \right\|
= \Op{ \gamma \ttiRate_n 
        \max_{k \in [d]} \frac{ \sigma_k^2 }{ \| \btrue^{(k)} \|^2 } } .
\end{equation*}
\end{corollary}

\begin{remark} \label{rem:BQ}
Recall that we collected the estimates produced by Equation~\eqref{eq:model-based:separateregs} with $r=2$ in the columns of the matrix
\begin{equation*}
\Bhat^{\ridge}
= \begin{bmatrix} \bhat_2^{(1)} & \bhat_2^{(2)} & \cdots & \bhat_2^{(d)}
\end{bmatrix} .
\end{equation*}
Analogously, we may construct for any orthogonal matrix $Q\in\bbO_d$, the matrix
\begin{equation*}
\Btilde_Q
=\begin{bmatrix} \btilde^{(1)}_Q & \btilde^{(2)}_Q &\cdots 
		&\btilde^{(d)}_Q \end{bmatrix} \in \IR^{p \times d},
\end{equation*}
where for $k=1,2,\dots,d$, the $\btilde^{(k)}_Q$ are the oracle estimates defined in Equation~\eqref{eq:def:btildek}.
By definition, for $\ell=1,2,\dots,p$,
\begin{equation*}
\Btilde_{Q,\ell\cdot}
=\left[\left( \frac{1}{n} Z^\top Z + \lassoreg_k I_p \right)^{-1}\right]_{\ell\cdot}\frac{1}{n} Z^\top
	XQ.
\end{equation*}
Thus, while there is a non-identifiability associated with this regression, it applies only to the basis for the latent positions $X$, and not to the node-level covariates $Z$.
This means that there is no ambiguity in which node features do or do not contribute to variation in latent positions, even if those latent positions are never observed.
For example, since $Q$ is orthogonal, the Euclidean norms of the rows of $\Btilde_Q$ are invariant to $Q$.
Thus, one can compare the magnitudes of rows of $\Btilde_Q$ to compare the influences of different node features on network structure, even though this structure is unobserved.
\end{remark}

\subsubsection{LASSO} \label{subsubsec:LASSO}
Analogously to our ridge regression results above, we prove a convergence rate for our group LASSO-based estimator in Equation~\eqref{eq:def:BhatgLASSO}, following the arguments and general proof structure in Chapters 6.2.2 and 8.6 of \cite{vdGB11}.
The main challenge is again that we must account for the fact that we work with $\Xhat$ rather than $X$ and that $\Xhat$ recovers $XQ$ for some unknown random orthogonal matrix $Q$, rather than recovering $X$ directly.

In the LASSO literature, estimation rates often depend on the compatibility constant.
In typical settings, one considers the covariates ($Z$ in our notation) fixed and proves results conditionally on these covariates.
In such a setting, the compatibility constant is, as the name suggests, a constant.
Here, we would like to allow for randomness in $Z$, and thus our compatibility ``constant'' is random, similar to the situation in \cite{KMP25}.
The convergence rate of our estimator thus depends on the behavior of the {\em random multitask compatibility constant}.

\begin{definition} \label{def:compatibility}
Given a set of indices $S \subseteq [p]$, define $\calB_{p,d}(S) \subseteq \IR^{p \times d}$ to be the set of all matrices $\beta\in\IR^{p\times d}$ with rows $\beta_j$ such that $\sum_{j\in S^c}\|\beta_j\|\leq3\sum_{j\in S}\|\beta_j\|$.
We call the random variable
\begin{equation} \label{eq:def:phi}
\phi_S = \min_{\beta\in\calB_{p,d}(S)} \frac{1}{\sqrt{nd}}
	\left\|Z\beta\right\|_F\frac{\sqrt{d|S|}}{\sum_{j\in S}\|\beta_j\|},
\end{equation}
the {\em random multitask compatibility constant}.
\end{definition}
Since our compatibility constant is random, we require what we call a {\em Random Multitask Compatibility Condition} for the set $S$, namely $1/\phi_S=O_P(1)$.
For a given $S$ and fixed non-random $Z$, $\phi_S>0$ is a classical condition in high-dimensional statistics.
See Section 8.6.3 of \citet{vdGB11} and references therein for a discussion and comparison to other conditions.

Similarly to the classical LASSO literature, our results require tail bounds on the correlations between the noise and the covariates.
Equation~\eqref{eq:def:calTn} gives a precise statement of the event that has to hold with high probability.
To guarantee the existence of such tail bounds, we require some regularity conditions on the noise in the form of subgamma tail bounds \citep[see][]{BLM13,Vershynin2020}.
We use the sub-gamma property to prove a concentration result on sums with random coefficients, which appears as Lemma~\ref{lem:cond_subGamma} in Appendix~\ref{app:proofsLASSO}.
Using this result, we can formulate the main result of this section.

\begin{theorem} \label{thm:LASSO:convergence_rate}
Under the model in Equation~\eqref{eq:model}, assume that $Z$ and $E$ are independent of one another and have i.i.d.~rows.
Under Assumption~\ref{assum:ttiRate}, suppose that for each $k\in[d]$, $E_{1k}$ is $(\nu,b)$-sub-gamma and that, for each $j\in[p]$, $Z_{1j}^2-\IE(Z_{1j}^2)$ is $(\nu_Z,b_Z)$-sub-gamma.
Suppose furthermore that $\log pd = o(n)$ and $\sup_{n\in\IN}\max_{j\in[p]}\IE(Z_{1j}^2)<\infty$.

For any $\epsilon>0$, there are positive constants $K_{\epsilon}$ and $N_{\epsilon}$ such that for
\begin{equation} \label{eq:def:lassoreg0n}
\lassoreg_{0,n}=K_{\epsilon}\max\left\{\frac{\ttiRate_n}{\sqrt{d}},\frac{\log\max(pd,n)}{\sqrt{n}}\right\},
\end{equation}
$\lassoreg \geq 2\lassoreg_{0,n}$, and $n\geq N_{\epsilon}$, it holds with probability at least $1-5\epsilon$ that
\begin{equation} \label{eq:thm:LASSO:convergence_rate:1}
\frac{1}{nd}\left\|Z\left(\Bhat_{\gLASSO}-BQ\right)\right\|_F^2
+\frac{\lassoreg}{\sqrt{d}}\sum_{j=1}^p\left\|\Bhat_{\gLASSO,j}-Q^\top B_j\right\|
\leq\frac{4\lassoreg^2|S|}{\phi_S^2}.
\end{equation}
Let, in addition, $\extradeltarate\to\infty$ be arbitrary and $\phi_S^{-1}=\Op{1}$.
Then, there exists a choice of regularization parameter $\lassoreg$ with $\lassoreg=\Op{\max\left\{\frac{\extradeltarate}{\sqrt{d}}\xi_n,\frac{\log\max(pd,n)}{\sqrt{n}}\right\}}$ such that
\begin{equation} \label{eq:thm:LASSO:convergence_rate:2}
\frac{1}{\sqrt{d}}\sum_{j=1}^p\left\|\Bhat_{\gLASSO,j}-Q^\top B_j\right\|
=\Op{|S|\max\left\{\frac{\extradeltarate}{\sqrt{d}}\xi_n,\frac{\log\max(pd,n)}{\sqrt{n}}\right\}}.
\end{equation}
\end{theorem}
A detailed proof can be found in Appendix~\ref{app:proofsLASSO}.

\begin{remark} \label{rem:LASSO:simplify1}
The requirement of Theorem \ref{thm:LASSO:convergence_rate} that all entries of $Z$ have the same tail properties is only a mild restriction compared to allowing every variable to have its own sub-gamma parameters, in light of the fact that one typically rescales the features to be of comparable orders when performing regression. Similarly, a uniform tail bound on the columns of $E$ seems, in consideration of the previous discussion, reasonable to assume.
\end{remark}
\begin{remark} \label{rem:LASSO:simplify2}
In the proof of Theorem \ref{thm:LASSO:convergence_rate}, we give precise definitions for $N_{\epsilon}$ and $K_{\epsilon}$. Hence, it is possible to remove the asymptotic view point from Equation~\eqref{eq:thm:LASSO:convergence_rate:2} and prove a finite sample result by replacing these definitions in \eqref{eq:def:lassoreg0n} and \eqref{eq:thm:LASSO:convergence_rate:1}. In fact, by investigating the interplay of all involved constants, for any given $n$, it is possible to choose $N_{\epsilon}<n$ (even $N_{\epsilon}=1$ is possible) at the expense of an increase in $K_{\epsilon}$. In this way, one can produce a finite sample statement that holds for all $n\in\IN$.
\end{remark}

To illustrate our above results, let us consider a concrete situation.
Suppose that $p\to\infty$ and that we are in the setting of Theorem~\ref{thm:vince}, in which case we have $\ttiRate_n=\sqrt{d/n} \log^2 n$.
Choose furthermore $\extradeltarate=\sqrt{\log p}\to\infty$. Then, the convergence rate of Equation~\eqref{eq:thm:LASSO:convergence_rate:2} in Theorem~\ref{thm:LASSO:convergence_rate} simplifies to
\begin{equation*} 
\frac{1}{\sqrt{d}}\sum_{j=1}^p\left\|\Bhat_{\gLASSO,j}-Q^\top B_j\right\|
= \Op{|S|\frac{\max\left\{ \sqrt{\log p}\log^2n,\log (pd \vee n) \right\} }
	{\sqrt{n}} }.
\end{equation*}
In typical LASSO results \citep[see, e.g., Corollary 6.2 in][]{vdGB11}, the convergence rate in $\ell_1$-norm is of the order $|S|\sqrt{n^{-1} \log p}$. 
The rate implied by Theorem~\ref{thm:LASSO:convergence_rate} is slightly worse by logarithmic factors.
This is for two reasons.
First, because we observe $\Xhat$ rather than $X$, we must control the error between these two terms by ensuring that certain events hold with high probability.
Using Theorem~\ref{thm:vince} for this purpose yields the additional $\log^2 n$ term.
Secondly, we allow for random covariates $Z$, which we handle by a slightly different exponential inequality in Lemma~\ref{lem:cond_subGamma}.
The union bounds needed for this argument incur the additional $\log pd$ term in our rate.

Theorem~\ref{thm:LASSO:convergence_rate} states that there is a random orthogonal transformation $Q$ such that each row of $\Bhat_{\gLASSO}$ lies close to the corresponding row of $BQ$.
Since $Q$ is unknown in practice, it is only of limited use to directly interpret the values of $\Bhat_{\gLASSO}$.
If one is interested in the actual effects of the covariates, one can instead interpret $\Bhat_{\gLASSO}\Bhat_{\gLASSO}^\top$ which approximates $BQ(BQ)^\top=BB^\top$, since $Q$ is orthogonal.
This allows us to interpret the size of the overall effect of the covariates and their relations to each other, as discussed in Remark~\ref{rem:BQ}.

\subsection{Results for CCA-Based Approaches} \label{subsec:cca-results}

Motivated by relaxing the linear model assumptions of the regression-based methods above, we now present two methods based on CCA.
These methods aim to detect correlation between the node-level features and network structure as encoded in the latent positions without assuming anything about how $X$ is generated from $Z$ (or vice versa).
One way to assess whether network structure is correlated with node-level features $Z \in \IR^{n \times p}$ is to consider the CCA alignment between $Z$ and the latent positions of the network.
That is, we look at the ordered singular values of
%\begin{equation*} \begin{aligned}
%&\CCA( Z, X ) \\
%&~~~=
%\left( \frac{1}{n} \sum_{i=1}^n (X_i-\xbar) (X_i-\xbar)^\top \!\! \right)^{\!\!-1/2} \!
%\left( \frac{1}{n} \sum_{i=1}^n (X_i-\xbar) (Z_i-\zbar)^\top \!\! \right) \!
%\left( \frac{1}{n} \sum_{i=1}^n (Z_i-\zbar) (Z_i-\zbar)^\top \!\! \right)^{\!\!-1/2} 
%\! \! \! \! ,
%\end{aligned} \end{equation*}
%where $\xbar = n^{-1} \sum_i X_i \in \IR^d$ and $\zbar = n^{-1} \sum_i Z_i$.
%Recalling the centering matrix $M_n \in \IR^{n \times n}$ as defined in Equation~\eqref{eq:def:Mn}, we can simplify the above expression to
\begin{equation} \label{eq:def:CCA:XtoZ}
\CCA( Z, X )  
= \Sigmatilde_X^{-1/2} 
\Sigmatilde_{X,Z} 
\Sigmatilde_Z^{-1/2},
\end{equation}
where $M_n \in \IR^{n \times n}$ is the centering matrix defined in Equation~\eqref{eq:def:Mn}, $\Sigmatilde_{X,Z}$ is as defined in Equation~\eqref{eq:def:SigmatildeXZ}, and $\Sigmatilde_X \in \IR^{d \times d}$ and $\Sigmatilde_Z \in \IR^{p \times p}$ are as defined in Equation~\eqref{eq:def:SigmatildeMarginal}.

In particular, we have the associated CCA coefficient
\begin{equation} \label{eq:def:rho:XtoZ}
\rho_{X,Z} = \left\| \CCA( Z, X ) \right\|.
\end{equation}

\begin{remark} \label{rem:covar-reg}
We note that the definition of the CCA coefficient defined in Equation~\eqref{eq:def:rho:XtoZ} relies on covariance estimates $\Sigmatilde_X, \Sigmatilde_Z$ and $\Sigmatilde_{X,Z}$.
In the high-dimensional regime where $p \rightarrow \infty$, as considered in Sections~\ref{subsubsec:ridge} and~\ref{subsubsec:LASSO}, the latter two estimates will be unstable in the absence of regularization of some form.
The simplest approach to this is to replace, for example, $\Sigmatilde_Z$ with a regularized version
\begin{equation*}
\frac{1}{n} Z^\top M_n Z + \gamma I
\end{equation*}
for some small $\gamma > 0$.
Of course, high-dimensional covariance estimation is a well-studied problem, and far more sophisticated methods have been proposed \citep[see, e.g.,][and references therein]{BicLev2008,FisSun2011,LedWol2012,Touloumis2015,YaoZheBai2015}.
In what follows, for the sake of space and simplicity, we restrict our attention to the case where both $d$ and $p$ are fixed with respect to $n$.
We anticipate that the results in the sequel can be extended to the regime where $p$ grows with $n$ by replacing $\Sigmatilde_Z$ with a suitably-chosen regularized estimate of $\Sigma_Z$ from this extensive literature.
In our high-dimensional experiments presented in Section~\ref{sec:expts}, we use the regularization method from \cite{FisSun2011}.
\end{remark}

As discussed previously, in practice, we do not have access to the true latent positions $X$, but rather must estimate them (up to some rotational non-identifiability) as $\Xhat$.
Substituting these estimates for the true latent positions in Equation~\eqref{eq:def:CCA:XtoZ}, we consider the leading singular values of
\begin{equation} \label{eq:def:CCA:XhatZ}
\CCA( \Xhat, Z ) = \left( \frac{1}{n} \Xhat^\top M_n \Xhat \right)^{-1/2}
		\frac{ \Xhat^\top M_n Z }{ n } ~ \Sigmatilde_Z^{-1/2} .
\end{equation}

The following result establishes that CCA applied to the estimated latent positions and the node-level features correctly estimates the correlation between the node features and the latent network structure. 
A detailed proof is given in Appendix~\ref{apx:cca:XtoZ}.

In what follows, we assume that the latent positions $X_1,X_2,\dots,X_n$ are drawn i.i.d.~with a non-singular covariance matrix $\Sigma_X \in \IR^{d \times d}$.
We note that this assumption is with minimal loss of generality: if $\Sigma_X$ has rank $d' < d$, we may make a suitable change of basis and restrict our attention to a $d'$-dimensional subspace, yielding $d'$-dimensional latent positions with an invertible covariance matrix.

\begin{theorem} \label{thm:CCA:XhatZ}
Suppose that Assumptions~\ref{assum:spectralRate},~\ref{assum:ttiRate}
and~\ref{assum:invertible} hold,
% Lemma ~\ref{lem:svals:X},
%Lemma~\ref{lem:control:SigmaX}; same assumptions as lem:svals:X
and suppose further that the quantity $\ttiRate_n$ in Assumption~\ref{assum:ttiRate} obeys $\ttiRate_n = o( 1 )$.
Then
\begin{equation*}
\rho_{\Xhat,Z} = \left\| \CCA( \Xhat, Z ) \right\|
\end{equation*}
satisfies
\begin{equation*}
\left| \rho_{\Xhat,Z} - \rho_{X,Z} \right| = \op{ 1 }.
\end{equation*}
\end{theorem}

Following the intuitive idea of performing CCA between the node-level features of the network with the latent positions, it is natural to consider CCA alignment of the node covariates and the adjacency matrix itself.
That is, we look at the ordered singular values of
\begin{equation*}
\CCA( A, Z )
= \Sigmatilde_A^{-1/2} \Sigmatilde_{A,Z} \Sigmatilde_Z^{-1/2},
\end{equation*}
where $\Sigmatilde_Z$ is as defined in Equation~\eqref{eq:def:SigmatildeMarginal} and $\Sigmatilde_A$ is defined analogously as
\begin{equation} \label{eq:def:SigmatildeA}
\Sigmatilde_A 
%= \frac{1}{n} \sum_{i=1}^n (A_i-\Abar) (A_i-\Abar)^\top 
%= \frac{1}{n} (A- \Abar \onevec_n^\top )^\top (A- \Abar \onevec_n^\top)
%= \frac{1}{n} A^\top A - \Abar \Abar^\top
= \frac{1}{n} A M_n A \in \IR^{ n \times n },
\end{equation}
and, analogously to Equation~\eqref{eq:def:SigmatildeXZ},
\begin{equation} \label{eq:def:SigmatildeAZ}
\Sigmatilde_{A,Z}
%= \frac{1}{n} \sum_{i=1}^n (A_i-\Abar)(Z_i-\Zbar)
%= \frac{1}{n} (A - \Abar \onevec_n^\top)^\top ( Z- \onevec_n \Zbar^\top )
= \frac{1}{n} A M_n Z.
\end{equation}

Unfortunately, doing this na\"{i}vely will result in a trivial CCA alignment score close to $1$, owing to the fact that $A$ is typically near full-rank.
To account for this problem, we regularize the eigenvalues of $\Sigmatilde_A$, and consider the regularized full-network CCA matrix
\begin{equation} \label{eq:def:CCA:fullnet}
\CCAregu( A, Z )
= \left( \Sigmatilde_A + \regu I \right)^{-1/2}
        \frac{1}{n} A M_n Z \Sigmatilde_Z^{-1/2} ,
\end{equation}
where $\regu > 0$.
This yields a corresponding CCA coefficient
\begin{equation} \label{eq:def:rho:AtoZ}
\rho_{A,Z}^{(\regu)} = \left\| \CCAregu( A, Z ) \right\| .
\end{equation}

As in Theorem~\ref{thm:CCA:XhatZ} above, under suitable conditions, this (regularized) full-network CCA method also correctly detects the presence or absence of correlation between the node-level features and network structure.
A detailed proof is given in Appendix~\ref{apx:cca:fullnet}.

\begin{theorem} \label{thm:CCA:fullnet}
%assumptions for Lemma~\ref{lem:fullnettrue:singval}
% and for Lemma~\ref{lem:svalgrowth:XandSigmatildeXX}
Under Assumption~\ref{assum:invertible},
%assums for Lemma~\ref{lem:full2true:signalclose}
suppose that Assumption~\ref{assum:spectralRate} holds with
\begin{equation} \label{eq:assum:fullnet:spectralRate:growth}
\spectralRate_n = o( n )
\end{equation}
and Assumption~\ref{assum:subspaceRate} holds with
\begin{equation} \label{eq:assum:fullnet:subspaceRate:growth}
\subspaceRate_n = o( n ) .
\end{equation}
Let $\rho_{A,Z}^{(\regu)}$ be as in Equation~\eqref{eq:def:rho:AtoZ} and let $\rho_{X,Z}$ as in Equation~\eqref{eq:def:rho:XtoZ}. 
Then, provided that the regularization parameter $\regu$ is positive and obeys $\regu = O(n)$,
\begin{equation*}
\left| \rho_{A,Z}^{(\regu)} - \rho_{X,Z} \right| 
=  \Op{ \frac{\spectralRate_n}{\sqrt{\regu n} }
+ \frac{|\regu|}{n} + \frac{ \subspaceRate_n }{ n^{3/2} } } .
\end{equation*}
\end{theorem}
\begin{corollary} \label{cor:CCA:fullnet:consistent}
Under the setting of Theorem~\ref{thm:CCA:fullnet}, suppose that
$\spectralRate_n^2 / n \ll \regu \ll n$.
Then
\begin{equation*}
\left| \rho_{A,Z}^{(\regu)} - \rho_{X,Z} \right| 
= \op{ 1 } .
\end{equation*}
\end{corollary}

\begin{remark}
It is tempting to optimize the rate in Theorem~\ref{thm:CCA:fullnet} with respect to the regularization parameter $\regu = \regu_n \ge 0$ by choosing $\regu = C\spectralRate_n^{2/3} n^{1/3}$.
Under the typical setting (see Theorems~\ref{thm:vince} and~\ref{thm:LLL} above), $\subspaceRate_n$ and $\spectralRate_n/\sqrt{n}$ are both of at most polylogarithmic order and so a natural option is to ignore the logarithmic factors and choose the (approximately) optimal $\regu = C n^{2/3}$.
We note, however, that $\spectralRate_n^2/n \ll \regu \ll n$ is sufficient to ensure consistency of the regularized full-network CCA method, in the sense that
\begin{equation*}
\left| \rho_{A,Z}^{(\regu)} - \rho_{X,Z} \right| = \op{ 1 }.
\end{equation*}
In our experiments in Section~\ref{sec:expts}, we take $\regu = \sqrt{n}$ throughout.
\end{remark}

%% Pretty sure we never actually have to define this. Keeping it just in case.
%Without loss of generality \citep[see, e.g., the discussion of identifiability  in][]{AthFisLevLyzParQinSusTanVogPri2018}, write the latent positions matrix $X$ as
%\begin{equation} \label{eq:X:spectraldecomp}
%X = U S^{1/2},
%\end{equation}
%where $U \in \IR^{n \times d}$ has orthonormal columns and $S = \diag(s_1,s_2,\dots,s_d) \in \IR^{d \times d}$.
%\begin{equation*}
%\Xhat = \Uhatpara \Shat^{1/2},
%\end{equation*}
%where $\Uhatpara \in \IR^{n \times d}$ has the leading $d$ orthonormal eigenvectors of $A$ as its columns and $\Shat$ is a diagonal matrix containing the corresponding eigenvalues.

\section{Experiments} \label{sec:expts}
In this section, we evaluate our proposed methods on a range of data, both simulated and from real-world sources.
Since, in applying all of the methods presented above, our ultimate goal is to assess whether or not the latent positions $X$ are associated to the node covariates $Z$, we describe in Section \ref{subsec:test} a permutation test to formally test for such association.
In the following Section \ref{subsec:simulations}, we begin by assessing our methods on simulated data.
Our focus lies on investigating the effects of model choice and of the strength and structure of correlation between the node covariates and latent positions on the permutation tests.
In Section~\ref{subsec:data_analysis}, we apply our methods to a pair of network data sets collected in the real world.

\subsection{Permutation Tests}
\label{subsec:test}
Our aim is to formally test the null-hypothesis \emph{no association between $X$ and $Z$}. In our model-based methods described in Section~\ref{subsec:mod_approach}, this corresponds to testing the hypothesis
\begin{equation*}
H_0 : B=0 .
\end{equation*}
In the case of our CCA-based methods described in Section~\ref{subsec:cca}, this corresponds to testing
\begin{equation*}
H_0 : \E ~\Sigma_X^{-1/2} (X_1 \!-\! \E X_1)^\top (Z_1\!-\! \E Z_1) 
	\Sigma^{-1/2}_Z = 0 ,
\end{equation*}
where $\Sigma_X$ and $\Sigma_Z$ are the population covariances of $X_1$ and $Z_1$, respectively.
In all cases, we use permutation testing.
That is, for a given choice of test statistic
\begin{equation*}
T_n = T( X_1,X_2,\dots,X_n ; Z_1,Z_2,\dots,Z_n ),
\end{equation*}
we repeatedly randomly permute the rows of $Z$ to obtain replicates of the form
\begin{equation*}
T\left( X_1,X_2,\dots,X_n ;
	Z_{\pi(1)},Z_{\pi(2)},\dots,Z_{\pi(n) } \right),
\end{equation*}
and compare these replicates to our original observed statistic.

\paragraph{LASSO-based methods}
\label{subsec:gLASSO}
For both regular and group LASSO, we choose $T_n$ to be the \emph{covariance test} introduced by \cite{LTTT14}.
See Appendix \ref{apx:ctest_glasso} for further details.
In the case of the regular LASSO, we separately test for association of each latent dimension $k\in[d]$ with the covariates and apply a Bonferroni correction to account for multiple testing.

For the group LASSO, \cite{LTTT14} does not provide a specific formula for $T_n$.
Therefore, we derive in Appendix \ref{apx:ctest_glasso} a formula for the covariance test statistic for the group LASSO in our case (or, more generally, in the case of multivariate regression; see Example 4.2 in \cite{HTW15} for a brief introduction).
These results may be of independent interest.
Note that using group LASSO avoids the multiple testing issues encountered in the regular LASSO.

\paragraph{Ridge Regression}
For ridge regression, the covariance test statistic used above cannot be directly generalized.
Instead, we take $T_n$ simply to be the $L^2$-norm of the estimate obtained when the tuning parameter is chosen via cross-validation.
Since each repetition of the permutation step requires running cross-validation again, computations are relatively slow for this statistic.

\paragraph{CCA based methods}
For the CCA-based methods, we use the CCA value as a test statistic.
That is, our test statistic is
\begin{equation*}
T( X_1,X_2,\dots,X_n ; Z_1,Z_2,\dots,Z_n )
=
\left\| 
	\frac{1}{n} \Sigmatilde_X^{-1/2} X^\top M_n Z \Sigmatilde_Z^{-1/2} 
\right\| .
\end{equation*}
In contrast to the penalized methods above, CCA has the advantage that we do not have to choose a tuning parameter. Therefore, the computations are comparatively fast and the permutation test can be computed efficiently.

\subsection{Simulations}
\label{subsec:simulations}
In this section, we evaluate our proposed methods empirically in six quite different scenarios that are designed to illustrate the advantages and shortcomings of our methods.
Four of these scenarios follow the model in Equation~\eqref{eq:model}.
The remaining two scenarios follow a different structure.

\subsubsection{Data Generating Processes}
\label{subsubsec:DGPs}
The data generating processes of the first four scenarios all have the following structure:
We specify a covariance matrix $\Sigma_Z \in \IR^{p \times p}$ and generate random node covariates according to
\begin{equation*}
Z_i \overset{\text{ind.}}{\sim}\calN(0,\Sigma_Z) ~\text{ for }~i \in [n] .
\end{equation*}
We generate a matrix $E \in \IR^{n \times d}$ according to
\begin{equation*}
E = [E_{ij}] \in \IR^{n \times d}
~\text{ where }~
E_{ij} \overset{\text{ind.}}{\sim} \calN(0,2)
~\text{ for }~i \in [n] \text{ and } j \in [d] .
\end{equation*}
We then specify a matrix of coefficients $B \in \IR^{p\times d}$ and construct latent positions and an adjacency matrix according to
\begin{equation*}
X = s ZB + E,\qquad A=XX^\top +\netnoise,
\end{equation*}
where $s \ge 0$ corresponds to a signal strength parameter and
\begin{equation*}
\netnoise_{ij}\overset{\text{ind.}}{\sim}\calN\left(0,(s^2+1)^2\right)
~\text{ for }~1 \le i \le j \le n .
\end{equation*}
Thus, our first four scenarios will all follow the model in Equation~\eqref{eq:model}, differing only in the choice of coefficients $B\in\IR^{p\times d}$, node feature covariance matrix $\Sigma_Z\in\IR^{p\times p}$ and signal strength $s\geq0$.
By varying $s \ge 0$, we change the distance from the base hypothesis
\begin{equation*}
H_0: \text{No covariate impact on latent positions},
\end{equation*}
which is true when $s=0$.
In Remark~\ref{rem:impc} below, we discuss in more detail how $s$ measures this deviation in terms of signal to noise ratio (SNR).

We describe now the first four data generation scenarios.
Note that the random steps below are only executed once in our experiments.
Afterwards, the chosen values remain fixed for all replicates.
Moreover, after generating the matrix $B$, we rescale its columns such that $\|B_{\cdot r}\|_2\in\{0,1\}$ for all $r=1,2,\dots,d$.
\begin{itemize}
\item[(i)] \textbf{No Sparsity.}
\begin{equation*}
\Sigma_Z:=\begin{bmatrix}
1      & \covarcov   & \dots  & \covarcov   & \covarcov \\
\covarcov   & 1      & \ddots & \vdots & \vdots \\
\vdots & \ddots & \ddots & 1      & \covarcov \\
\covarcov   & \covarcov   & \dots  & \covarcov   & 1
\end{bmatrix}, \quad
\covarcov=\frac{1}{p-1},\quad 
B_{kr}\sim\calN\left(\frac{1}{k},\frac{1}{9k^2}\right).
\end{equation*}
The different entries of $B$ are assigned independently.
\item[(ii)] \textbf{Entry-wise sparsity.} $\Sigma_Z$ as in (i) and $B_{kr}=0$ except for $\lfloor (\log n)/2d \rfloor$ randomly selected entries, which are independently distributed as $\calN(0,1)$.
\item[(iii)] \textbf{Row-wise sparsity.} $\Sigma_Z$ as in (i) and $B_{kr}=0$ except for $\lfloor (\log n)/2\rfloor$ randomly selected rows. For each selected $k \in [p]$, we generate $B_{kr}\overset{\text{ind.}}{\sim}\calN(0,1)$ independently for $r\in[d]$.
\item[(iv)] \textbf{High correlation.} $B$ is generated as in (iii) and
\begin{equation*}
\Sigma_Z :=
\sigma_0^2
\begin{bmatrix}
1          & \covarcov       & \covarcov^2 & \dots  & \covarcov^{p-1} \\
\covarcov       & 1          & \covarcov   & \dots  & \covarcov^{p-2} \\
\vdots     &  \ddots    & \ddots & \ddots & \vdots \\
\covarcov^{p-2} & \dots      & \covarcov   & 1      & \covarcov \\
\covarcov^{p-1} & \covarcov^{p-2} & \dots  & \covarcov   & 1
\end{bmatrix}
,\,
\covarcov=0.95,
\end{equation*}
with $\sigma_0$ chosen such that the maximal eigenvalue of $\Sigma_Z$ equals $2$.
\end{itemize}
\begin{remark}
\label{rem:impc}
Note that, by construction, the largest eigenvalue of $\Sigma_Z$ equals $2$ in all four cases described above.
For any column $B_{\cdot r}$ of $B$, we hence have
\begin{equation*}
\IE\left(\left|Z_i^\top B_{\cdot r}s\right|^2\right)
=
s^2B_{\cdot r}^\top \Sigma_Z B_{\cdot r}\leq 2s^2,
\end{equation*}
since the columns $B_{\cdot r}$ are standardized. 
Since $E_{ij}\sim\calN(0,2)$, we have a signal-to-noise-ratio (SNR) of at most $s^2$ in the regression of $X$ on $Z$.
Thus, $s\ge 0$ quantifies the level of the impact of the covariates on the latent positions as desired.
However, an undesired consequence of this is that the latent positions $X$ are more dispersed on average. Thus, the edge-level noise $\netnoise$ becomes less important. To counteract this effect, we increase the strength of the noise accordingly. To argue that the edge noise $\netnoise$ has exactly the right strength, we first compute the strength of the signal $X$ to be
\begin{equation*} \begin{aligned}
\IE(\|X\|_F^2)
&= \IE \trace XX^\top 
=\sum_{i=1}^n\IE X_i^\top X_i
=\sum_{i=1}^n\IE (Z_i^\top Bs+E_{i\cdot})(Z_i^\top Bs+E_{i\cdot})^\top  \\
&= \sum_{i=1}^n\IE\left(s^2Z_i^\top BB^\top Z_i+E_{i\cdot}E_{i\cdot}^\top \right)
=\sum_{i=1}^ns^2\trace B^\top  \left( \IE Z_iZ_i^\top  \right) B + 2dn \\
&= n s^2 \underbrace{\trace B^\top \Sigma_Z B}_{\leq2d}+2dn\leq2dn(s^2+1).
\end{aligned} \end{equation*}
On the other hand, the noise level is given by $\sqrt{ \IE\|\netnoise\|_F^2 }=n(s^2+1)$, because the noise is applied to $XX^\top $.
Hence, for our choices, the SNR in the regression of $A$ on $X$ is bounded by $2d$, which is independent of $s$.
Thus, the variance of the edge noise $\netnoise$ is always of the same order as the variation of the latent positions, regardless of our choice of $s$. In this sense, $s$ affects the importance of the covariates on the latent positions, but does not effect the SNR in the regression of $A$ on $X$.
\end{remark}

Our remaining two scenarios, which we number (v) and (vi) below, will deviate from the model in Equation~\eqref{eq:model}, providing an illustration of how our methods behave in the presence of misspecification.
In particular, these two additional models violate the linearity assumption in Equation~\eqref{eq:model}.
First, we replace the model in Equation~\eqref{eq:model} with a generalized linear model for the covariates $Z$ and the latent positions $X$.
This allows us to, for example, model latent positions that can be used to generate Bernoulli edges.
\begin{itemize}
\item[(v)] \textbf{Network} Let $Z$ and $B$ be as in (iii). Define $\alpha_{ir}$ for $i \in [n]$ and $r \in [d]$ according to
\begin{equation*}
\log \alpha_{ir}:=Z_{i\cdot}B_{\cdot r}s+E_{ir},
\end{equation*}
where $E_{ir}\sim\calN(0,2)$. Generate the latent positions according to $X_i\sim \Dirichlet(\alpha_{i\cdot})$, and generate $A$ as a binary random dot product graph with latent positions $X$.
\end{itemize}
Similarly to Remark~\ref{rem:impc}, the SNR in the regression of $\alpha$ in the covariates $Z$ is bounded from above by $s^2$.
Moreover, in the Dirichlet distribution, the entries of $\alpha_{i\cdot}$ are scaled to sum to $1$. Therefore, in contrast to Scenarios (i)-(iv), we do not have to manually re-scale the variance of the error in the generation of $A$ from $X$.

Finally, we consider a scenario that is completely different from the model in Equation~\eqref{eq:model} in the sense that the latent positions are not derived from the covariates but rather the covariates reflect a structure in the latent positions.
\begin{itemize}
\item[(vi)] \textbf{Assortative Mixing} Let $g_1,g_2,\dots,g_n\sim\mathcal{U}(\{1,2,3,4\})$ be uniformly random independent group assignments. We now assign each vertex $i$, conditionally on $g_i$, a latent position $X_i\sim\Dirichlet(\alpha_{g_i})$, where $\alpha_k\in\IR^4$ is a vector with all entries equal to one except for its $k$-th position, which is $100$.
For example, when $k=2$, $\alpha_2=(1, 100, 1, 1)$. $A$ is then generated as a binary random dot product graph with these latent positions. When two latent positions $X_i$, $X_j$ take values close to the corresponding expected values, then the connection probability between $i$ and $j$ is $0.94$ if $i$ and $j$ are in the same group and around $0.02$ otherwise. The first covariate $Z_{i1}$ of each vertex $i$ is generated from the group assignments via
\begin{equation*}
Z_{i1}=s\left(g_i-\frac{5}{2}\right)+\epsilon_i,
\end{equation*}
with $\epsilon_i\overset{\text{ind.}}{\sim}\calN(0,5/4)$.
The remaining $p-1$ covariates are drawn independently from $\calN(0,5/2)$.
\end{itemize}
Similarly to the other cases above, we have chosen the parameters such that all covariates have the same mean. Moreover, for $s=1$, they all have the same variance. The SNR in the generation of $Z_{i1}$ is exactly $s^2$.

\subsubsection{Results}
In order to asses finite-sample performance, we consider $1,000$ Monte Carlo repetitions in each of the six scenarios outlined above. We begin by assessing if all tests hold the required level on $H_0: s=0$ in all six scenarios.
Note that scenarios (i)-(iii) are identical in this case.
We apply all permutation tests with nominal level $\alpha=0.05$ and consider each scenario with $n=100$ vertices and $p=200$ covariates.
We set the number of latent dimensions to be $d=4$ and we use $100$ permutations for the permutation tests. The results are shown in Table \ref{tab:H0level}. We note that the tests hold the level in essentially all cases. There is only one slight deviation in the case of high-correlation, scenario (iv), for the LASSO.
By investigating QQ-plots, we have also shown that the distribution of the test statistic is as it should be. We present details in Appendix~\ref{apx:additional_emp_res:simulations}.

\begin{table}
\centering
\footnotesize
\caption{\footnotesize Percentages of rejections (Est.) in the case of $s=0$, i.e., the hypothesis $H_0$, for the different tests in the respective scenarios (all calibrated to level $5\%$). Note that, for $s=0$, scenarions (i)-(iii) are identical. The lower (LB) and upper (UB) bounds correspond to $99\%$ pointwise confidence regions.} 
\begin{tabular}{l|rrr|rrr|rrr|rrr}
  \hline
Scenario & \multicolumn{3}{c}{(i)-(iii)} & \multicolumn{3}{c}{(iv)} & \multicolumn{3}{c}{(v)} & \multicolumn{3}{c}{(vi)} \\
 & Est. & LB & UB & Est. & LB & UB & Est. & LB & UB & Est. & LB & UB \\ 
  \hline
LASSO  & 0.066 & 0.046 & 0.086 & 0.088 & 0.065 & 0.111 & 0.047 & 0.030 & 0.064 & 0.055 & 0.036 & 0.074 \\ 
gLASSO & 0.064 & 0.044 & 0.084 & 0.062 & 0.042 & 0.082 & 0.048 & 0.031 & 0.065 & 0.052 & 0.034 & 0.070 \\ 
ridge  & 0.056 & 0.037 & 0.075 & 0.046 & 0.029 & 0.063 & 0.055 & 0.036 & 0.074 & 0.048 & 0.031 & 0.065 \\ 
CCA    & 0.060 & 0.041 & 0.079 & 0.048 & 0.031 & 0.065 & 0.058 & 0.039 & 0.077 & 0.045 & 0.028 & 0.062 \\ 
netCCA & 0.053 & 0.035 & 0.071 & 0.052 & 0.034 & 0.070 & 0.060 & 0.041 & 0.079 & 0.057 & 0.038 & 0.076 \\ 
   \hline
\end{tabular}
\label{tab:H0level}
\end{table}

We turn next to assessing the performance of the tests when $H_0$ does not hold, that is, when $s > 0$.
We have argued in Section \ref{subsubsec:DGPs} that, in each scenario, $s^2$ can be interpreted as the signal to noise ratio (SNR).
Therefore, we call $s$ the root signal to noise ratio (rSNR). 
The results are shown in Figure \ref{fig:power}. In Scenario (i), which corresponds to a linear model without sparsity, we see that all methods perform equally well (with the methods based on the linear model performing slightly better, and within those the LASSO-based methods seem to have a slight advantage over the ridge-based method).
Note that we obtain good power already for rSNR equal to one.
As soon as sparsity is introduced in the linear model, the LASSO-based methods perform best: For entry-wise sparsity, Scenario (ii), the regular LASSO strictly outperforms all other methods. In the case of row-wise sparsity, Scenario (iii), the group LASSO performs equally well as the regular LASSO. 
If, in addition, high correlation is introduced, as in Scenario (iv), all methods perform similarly with slight advantages for the group LASSO (since the latent dimension selection that we used produced a high number of dimensions in this case, we have capped the number of latent dimensions in this scenario by $20$ to reduce the computational burden).
When we turn to non-linear models, e.g., Scenario (v), the CCA-based methods show their appeal. Here, the network-based CCA method performs best among all methods with both LASSO based methods being close competitors. Comparing scenarios (iv) and (v) to Scenarios (i)-(iii), we observe that slightly higher rSNRs are required in Scenarios (iv) and (v) in order to obtain the same power as in Scenarios (i)-(iii).
In the case of Scenario (v), a similar effect was observed in \cite{FosHof2015}, and was attributed to the reduced amount of information available in a binary network compared to a weighted network.
Finally, if we reverse the order of dependence in Equation~\eqref{eq:model}, i.e., the covariates are generated from the network as in Scenario (vi), the LASSO-based methods perform well for low rSNRs but are unable to keep the trend and are eventually outperformed by the network CCA method.
Note lastly that CCA and ridge behave similarly in Scenarios (i), (iii), (v), and (vi). One explanation for this behavior is that both methods aim to identify linear relations between covariates and latent positions and that at the same time the regularization used for the CCA (see Remark \ref{rem:covar-reg}) is actually very similar to the ridge penalty, as can be seen in Equation~\eqref{eq:def:bhatk}.

\begin{figure}
\includegraphics[width=\textwidth]{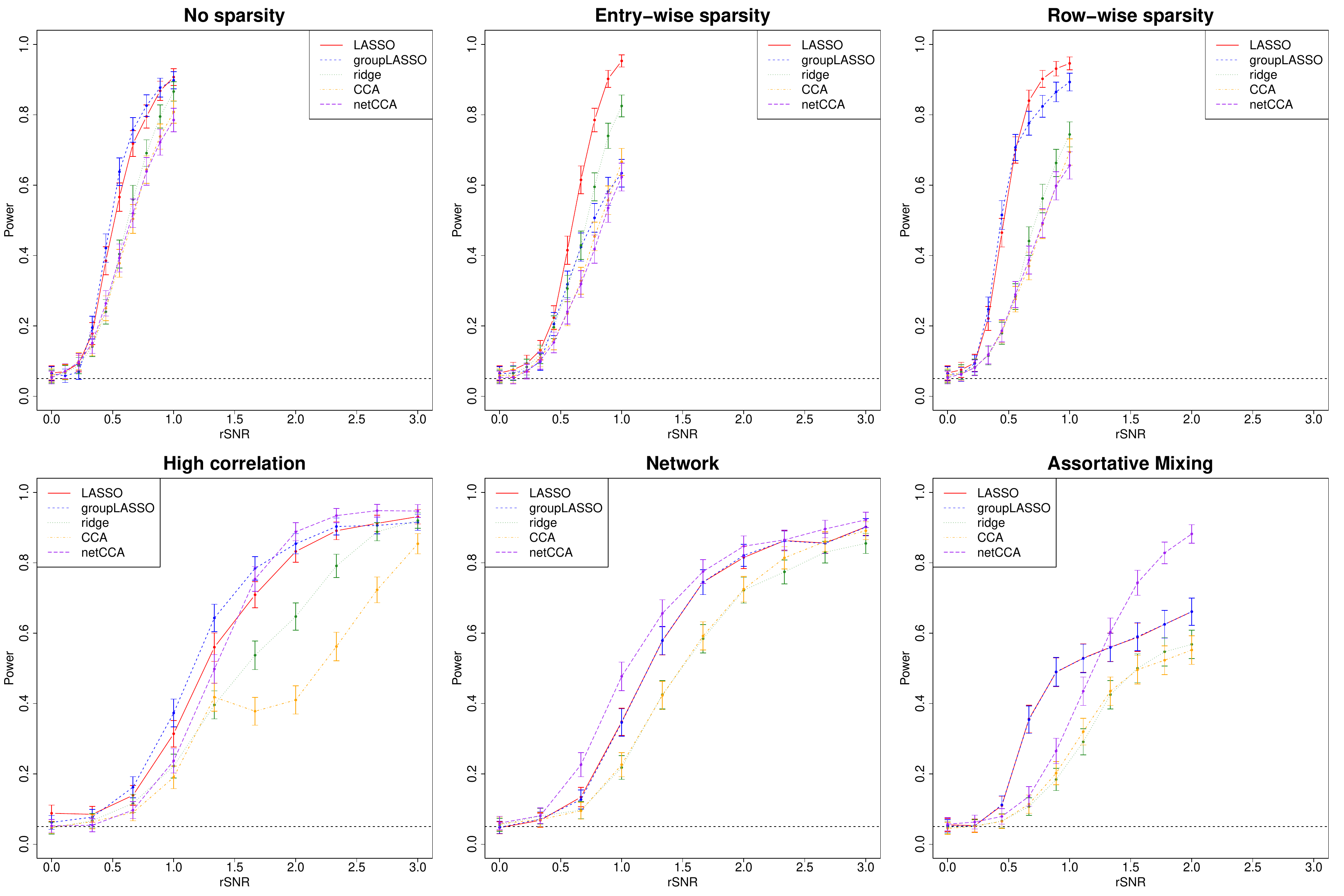}
\caption{Rejection rates as a function of rSNR for the five proposed methods in the six model scenarios. The error bars show $95\%$ point-wise confidence intervals for the rejection rate. The dashed line indicates the nominal $5\%$-level.}
\label{fig:power}
\end{figure}

Next, we consider the behavior of our proposed methods as the number of observations $n$ and the dimension $p$ vary.
We empirically investigate their behavior in the row-wise sparse setting, Scenario (iii), and in the assortative mixing case, Scenario (vi).
The results are shown in Figure \ref{fig:pn_stream}.
In Scenario (iii), the model assumptions of all methods are true, and thus all perform well even in low SNR settings.
On the other hand, Scenario (vi) was difficult for all tests.
When looking at the plots in the left column, it becomes visible that the performance of all tests improves as $n$ increases.
In addition, we note that for the high-dimensional case (i.e., when $n$ is not much larger than $p$), the LASSO-based methods perform slightly better in Scenario (iii). 
In the case $n>p$, all methods catch up with the LASSO-based methods. 
In contrast, in Scenario (vi), the network CCA method performs remarkably well in all cases other than $n<p$.
The right column of Figure~\ref{fig:pn_stream} shows the test performances in a very high-dimensional regime. Here we see that the LASSO-based methods roughly keep their performance while the CCA-based methods show their real strength primarily in the low-dimensional cases, in which they outperform the LASSO methods. Recall that the true data generating processes are very sparse in the covariates. Since the LASSO-based methods make specific use of this sparsity, it is plausible that they maintain higher power in the high-dimensional regimes. 
The CCA-based methods applied here do not promote sparsity, but it is to be expected that methods based on sparse CCA will be able to maintain high power in high-dimensional regimes.
We leave exploration of this to future work.

\begin{figure}
\includegraphics[width=\textwidth]{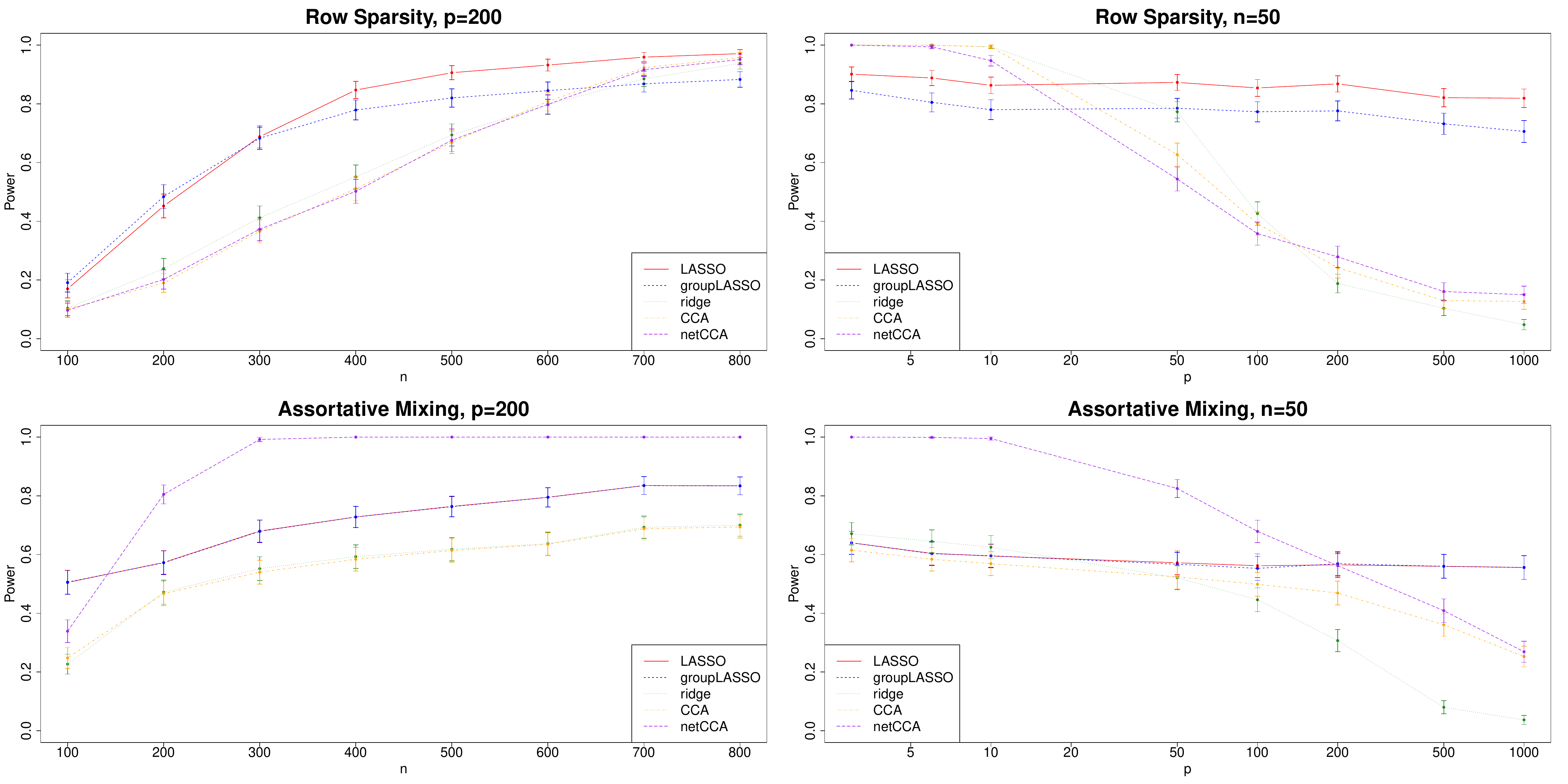}
\caption{Rejection rates of the five different tests for different choices of $p$ and $n$ in Scenarios (iii) and (vi). Top-left: Scenario (iii) with two non-zero rows and rSNR equal to $0.3$. Top-right: Scenario (iii) with two non-zero rows and rSNR equal to $1$. Bottom-left: Scenario (vi) with sSNR equal to $1$. Bottom-right: Scenario (vi) with rSNR equal to $2$.}
\label{fig:pn_stream}
\end{figure}

Finally, we compare our method to the distance correlation method presented in \cite{LeeShePriVog2019}. To this end, we compare the performance of all methods in Scenario (v) with $s=1$, that is, we have chosen a non-linear model.
We consider two different growth regimes: a high-dimensional regime in which we fix $n=50$ and consider increasing values of $p$, and a moderate-dimensional regime, where we fix $p=10$ and consider increasing values of $n$.
The left panel of Figure \ref{fig:runtime:pstream} shows the power of all methods as a function of $p$ when $n=50$.
We denote the method by \cite{LeeShePriVog2019} by \emph{Lee}. Note that we have chosen a fairly difficult situation to test the methods. As previously, the LASSO methods seem to hold their power when $p$ increases, while network CCA and Lee et al.'s method show superior performance for low values of $p$, which deteriorates when $p$ increases. In contrast, if we fix $p=10$ and let $n$ increase, as we did in the left panel of Figure \ref{fig:runtime:nstream}, all methods show an increasing performance, but again network CCA performs best out of the tested methods.
The right panels of Figures \ref{fig:runtime:pstream} and \ref{fig:runtime:nstream} show the average run-time of a single fit in the corresponding set-up on the desktop computer of one the authors (Intel Core Ultra 7 265K (3.90 GHz)). 
We can see that while network CCA and Lee et al.~perform similarly in terms of power, network CCA runs orders of magnitude faster.
Similarly, the LASSO-based methods also have a faster run-time. 
Interestingly, Lee et al.'s method does not exhibit an increasing run-time as $p$ increases. 
Hence, extrapolating, it is possible that our methods run slower than the method in \cite{LeeShePriVog2019} in extremely high-dimensional cases, but our simulations suggest that the LASSO-based methods will exhibit better power in this setting.
On the other hand, when $p$ is fixed and $n$ increases, Lee et al.~show a considerable increase in computation time.
Visually, it appears that from our methods, only network CCA shows a similar slope in the increase, likely because of the matrix-matrix products it requires, while the computation time for the other methods is only slowly increasing with $n$.

\begin{figure}
\begin{subfigure}{\textwidth}
\includegraphics[width=\textwidth]{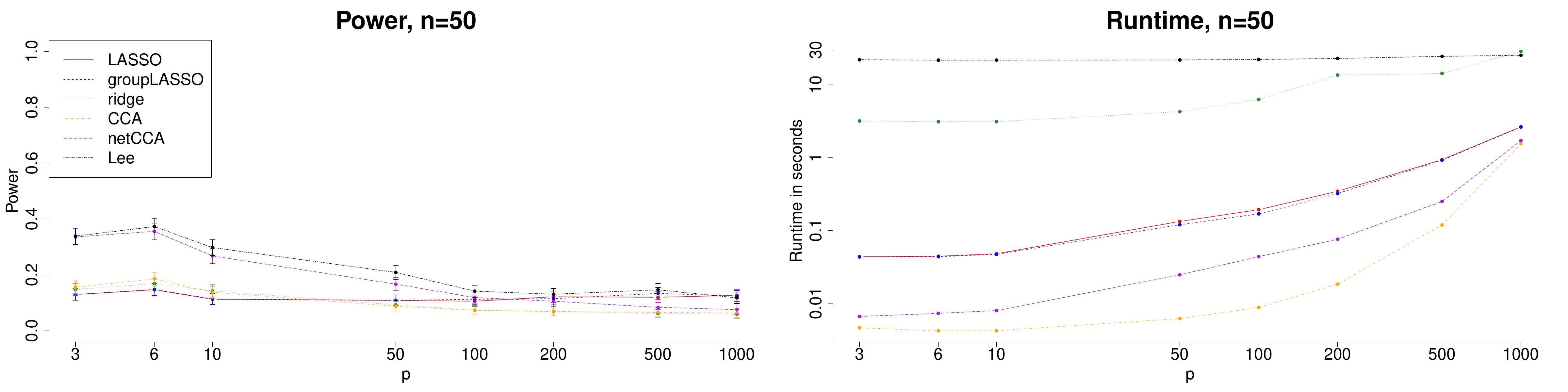}
\caption{$n=50$ is fixed and $p$ varies.}
\label{fig:runtime:pstream}
\end{subfigure}
\begin{subfigure}{\textwidth}
\includegraphics[width=\textwidth]{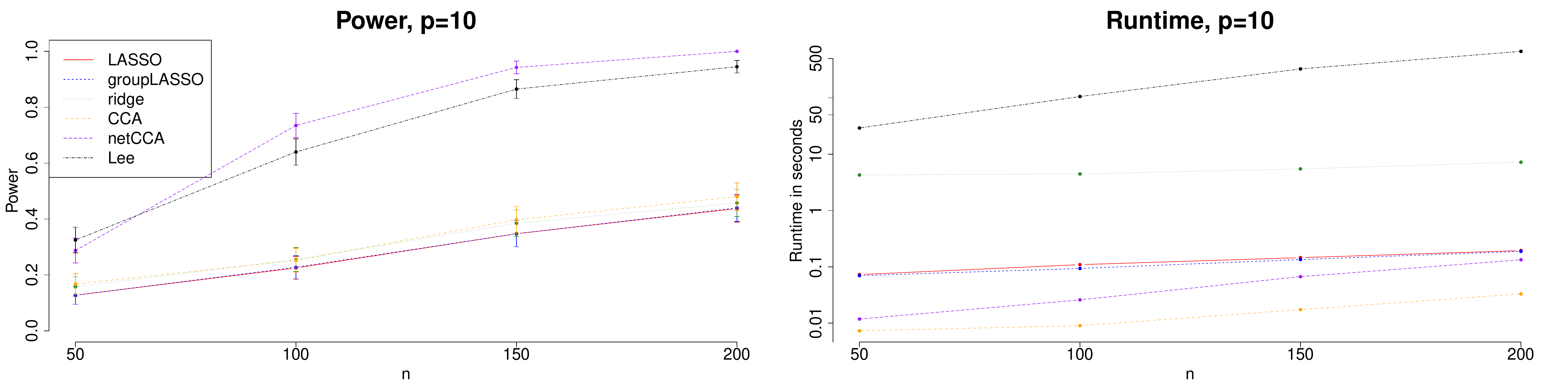}
\caption{$p=10$ is fixed and $n$ varies. To reduce the computational burden, we have considered $400$ permutation replicates to produce these figures.}
\label{fig:runtime:nstream}
\end{subfigure}
\caption{Plots show the rejection rates (left) and run-times (right) of the different tests for different choices of $p$ and $n$ in Scenario (v) with $s=1$. The method by \cite{LeeShePriVog2019} is labeled {\em Lee}.}
\label{fig:runtime}
\end{figure}

%Finally, we aim to present the estimation performance on the estimates of $B$ itself. We therefore focus on the row-wise sparsity setting. Since the estimate is only identified up to an orthogonal transformation, we first Procrustes align the latent positions and their ASE estimates. Then, we use the alignment matrix to transform the estimate $\hat{B}$ to make it comparable to the true $B$. In Figures \ref{fig:hist_LASSO}-\ref{fig:hist_ridge}, we show histograms of the resulting estimators. As expected, the LASSO estimates experience some shrinkage and therefore underestimate the true values. The ridge estimator seem to be not suitable for estimation. 
%
%\begin{figure}
%\includegraphics[width=\textwidth]{./Figures/hist_LASSO.pdf}
%\caption{Histograms of the LASSO estimates (after Procrustes alignment) in the row-wise sparsity scenario, vertical lines show true values of $B$}
%\label{fig:hist_LASSO}
%\end{figure}
%
%\begin{figure}
%\includegraphics[width=\textwidth]{./Figures/hist_gLASSO.pdf}
%\caption{Histograms of the Group-LASSO estimates (after Procrustes alignment) in the row-wise sparsity scenario, vertical lines show true values of $B$}
%\label{fig:hist_gLASSO}
%\end{figure}
%
%\begin{figure}
%\includegraphics[width=\textwidth]{./Figures/hist_ridge.pdf}
%\caption{Histograms of the ridge estimates (after Procrustes alignment) in the row-wise sparsity scenario, vertical lines show true values of $B$}
%\label{fig:hist_ridge}
%\end{figure}
%
%
%

\subsection{Data Analysis} \label{subsec:data_analysis}
In order to illustrate how our algorithms can be deployed in practice, we apply them to two data sets: one a protein-protein interaction network and one network of Wikipedia articles. All of the analyses presented below were carried out on the authors' desktop computers in reasonable wall-clock time.

\subsubsection{Protein-protein interactions}
\label{subsubsec:data_analysis:protein}
We study a network data set that was explored by \cite{MKSCOFB02}. The authors provide a data set of proteins produced by yeast cells and interactions among those proteins. Interactions have been found via different methods and we restrict to those interactions that the authors labeled with high confidence.
The network is depicted in Figure \ref{fig:protein_network}. As covariates, we take the $13$ different protein categories provided in the same data set. 
More precisely, for each protein, the covariate vector lies in $\{0,1\}^{13}$, and has a single entry equal to $1$ defining its category. There is one category \emph{uncharacterized}, in which we also manually put ten proteins that were present in the data set but were otherwise not labeled with a category by the original authors. The network comprises $988$ vertices and $2,455$ edges, with $13$ covariates. Thus, this data set is comparatively low-dimensional.

Before we can apply our methods, we must select the dimension of the latent space.
In order to choose the latent dimension in a data driven fashion, we employ here the method of \citet{LiLevZhu2020}, which selects $14$ latent dimensions.
A scree plot of the eigenvalues of the adjacency matrix is given in Figure \ref{fig:spectrum:protein} of Appendix \ref{apx:additional_emp_res:data}.
We apply the permutation tests based on group LASSO, ridge estimation, and the two CCA-based tests as described in Section~\ref{subsec:test}.
In each case, we perform $500$ random permutation repetitions to obtain a $p$-value. All tests reject the null-hypothesis of no covariate impact at all conventional levels. In fact, the $p$-value in all four tests is $0$.
Moreover, the group LASSO and ridge methods output an estimate $\hat{B}$ for the matrix $B$ in \eqref{eq:model}. We take the norm $\|\hat{B}_k\|_2$ of the $k$-th row of $\hat{B}$ as an indicator of the importance of the covariate $k$, i.e., in this case, of the $k$-th protein category. We choose in this way the four most relevant protein categories according to the group LASSO, and show in Figure \ref{fig:protein_network} the protein-protein interaction network with vertices colored according to these four categories. It visually appears that vertices from these categories are well-connected to one another compared to the background rate, suggesting that these covariates are indeed associated with variation in network structure. 
We investigate this further in Appendix~\ref{apx:additional_emp_res:data} by showing the projections of the features on the latent dimensions.

%To further illustrate the utility of our novel group LASSO method, we visualize the feature projections, that is, the columns of $Z\Bhat^{\gLASSO}$, or, in other words, the latent positions predicted by our model. $Z\hat{B}^{\gLASSO}$ contains in its $14$ columns a feature projection for each of the $14$ latent dimensions. Figures \ref{fig:protein_LD1} and \ref{fig:protein_LD2} in Appendix \ref{apx:additional_emp_res} show the network plot, where the vertex colors are indicative for the feature projection (there is one network plot for each of the $14$ feature projections). Different feature projections appear to be used to distinguish some clusters of proteins from the others, e.g., feature projection 1 in Figure \ref{fig:protein_LD1} identifies higher degree vertices in the central cluster. Feature projection 4 in the same Figure seems to differentiate between three smaller clusters below the central cluster.

\begin{figure}
\centering
\includegraphics[width=\textwidth]{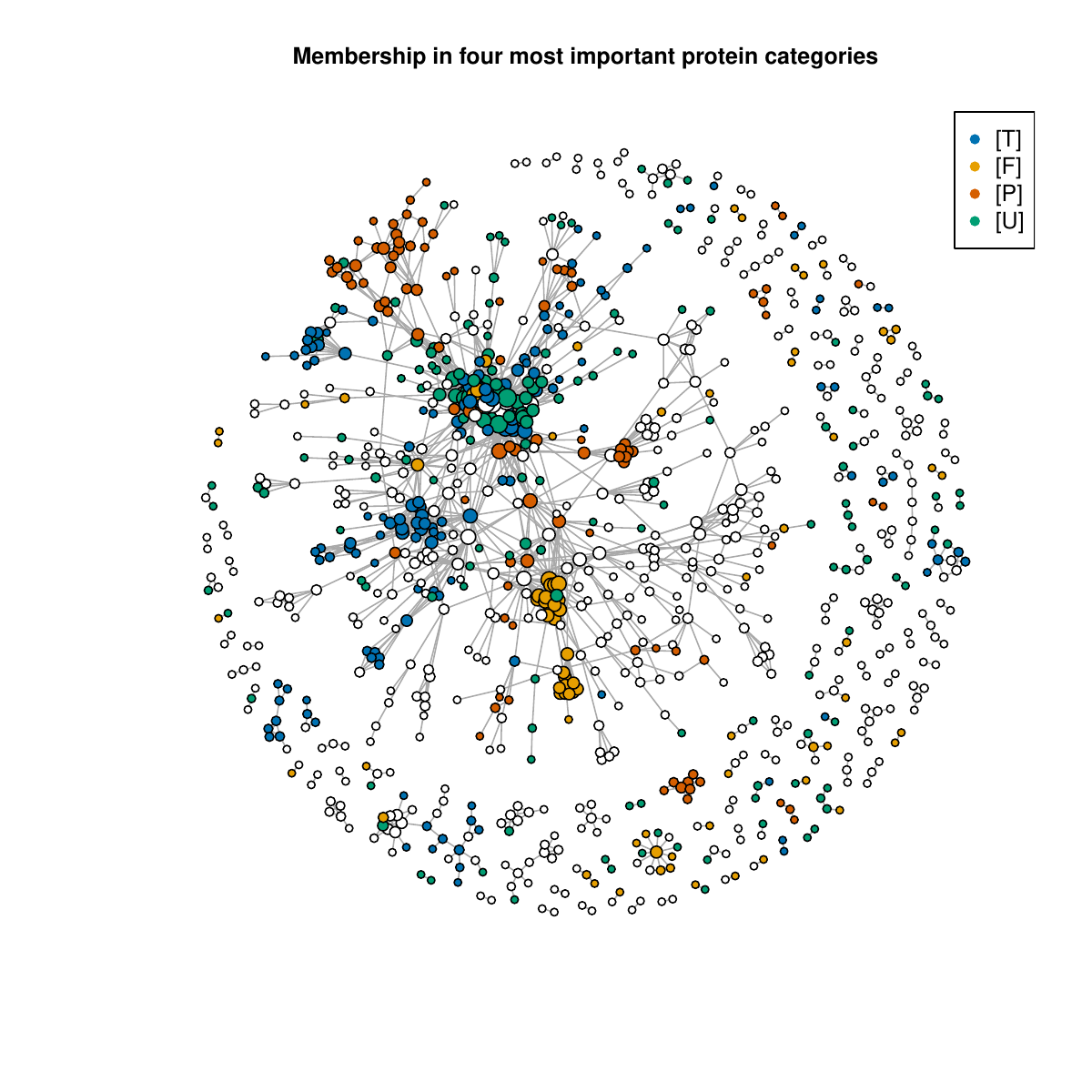}
\caption{Protein network analyzed in Section~\ref{subsubsec:data_analysis:protein}. Vertex sizes are proportional to their degree and the vertices are colored according to the four most relevant protein categories as identified by our LASSO-based method: [T] - Transcription, [F] - Protein fate, [P] - Translation, [U] - Uncharacterized}
\label{fig:protein_network}
\end{figure}

\subsubsection{Wikipedia}
We turn now to a network analyzed in \citet{RCS21}\footnote{The dataset is available for download from \url{http://snap.stanford.edu/data/wikipedia-article-networks.html}}. The vertices of the network correspond to Wikipedia pages about chameleons. 
An edge between two pages means that there is a mutual link.
The network has $2,277$ vertices and $31,371$ edges, and we have removed self-links.
For each page, the presence of certain nouns is recorded.
Moreover, the average monthly traffic on each page between October 2017 and November 2018 is available.
We construct from these data a feature vector for each Wikipedia page that comprises the log-average-monthly-traffic, the log-number-of-nouns, and a binary vector encoding which nouns are present in the article.
This yields $3,134$-dimensional node covariates.

Figure \ref{fig:spectrum:chameleon} in Appendix \ref{apx:additional_emp_res:data} shows the eigenvalues of the adjacency matrix of the network. We again used the model selection procedure of \citet{LiLevZhu2020}, which suggests $55$ as the latent dimension.
We apply here our tests based on group LASSO and both CCA tests. All three permutation tests reject the hypothesis of no relation between the features and the network with a p-value of $0$ based on $500$ permutation replicates.

We visually verify that the covariates capture some structure in the network by showing the feature projections of highest variance in Figure \ref{fig:chameleon_X}.
More precisely, note that $Z\Bhat^{\gLASSO}$ contains in its columns the estimated projections of the covariates on the feature space.
In this way, we obtain an estimate for each of the $55$ latent dimensions.
Instead of showing all $55$ projections, Figure \ref{fig:chameleon_X} shows those with the highest variance.
The colors of the vertices correspond to the values of the projection and the size of the vertices is proportional to their degree.
The plot suggests that the first two projections distinguish two clusters on the left of the network. The third and the fifth projections seem to identify some high-degree vertices in the center of the network.
Once again, this indicates that the features identified by our LASSO-based method do indeed vary with network structure.

\begin{figure}
\centering
\includegraphics[height=0.9\textheight]{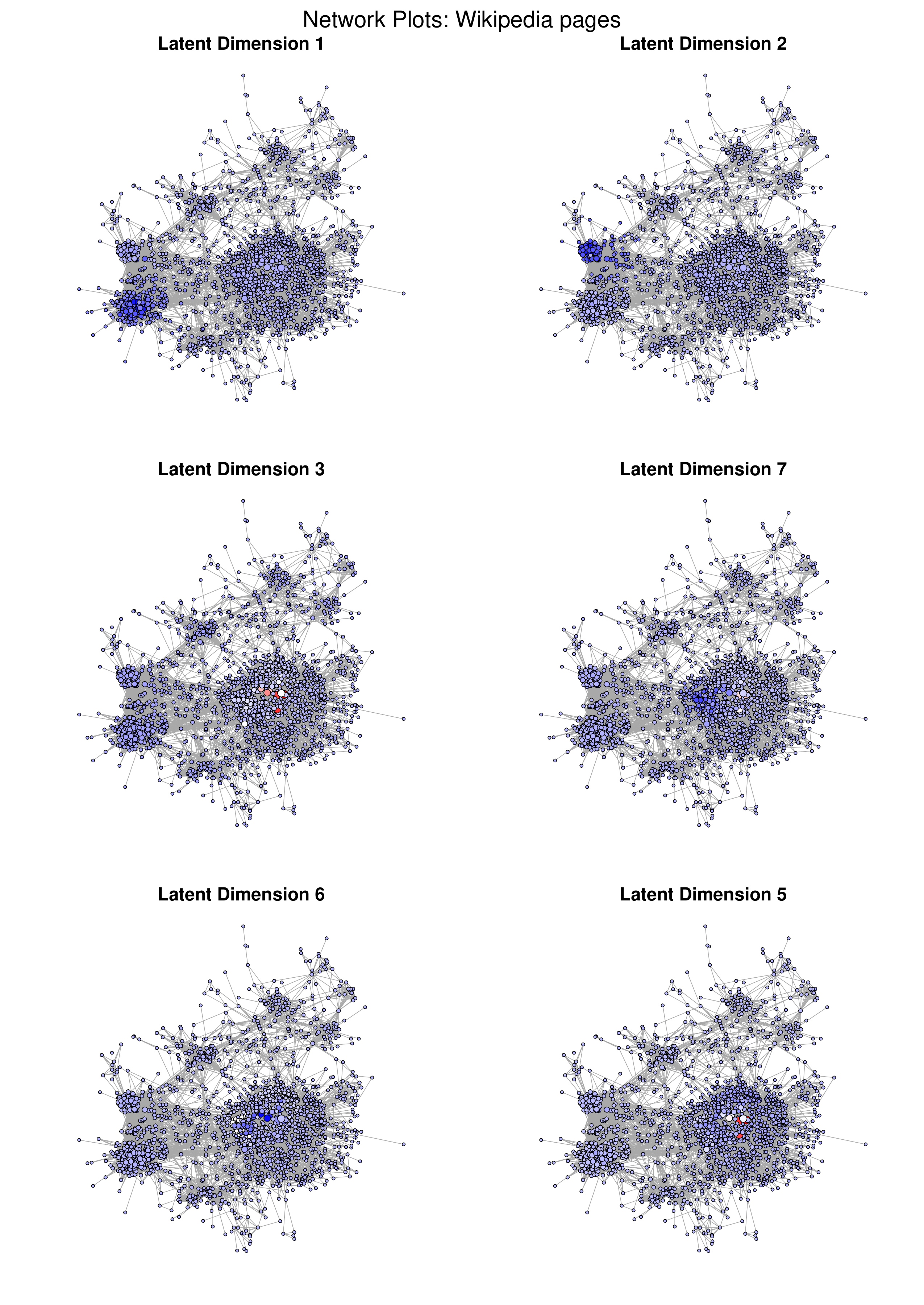}
\caption{Plots of the Wikipedia network with vertices colored according to four different feature projections (i.e., four different columns of $Z\Bhat^{\gLASSO}$. The blue end of the spectrum corresponds to negative values and the red end of the spectrum corresponds to positive values. Vertex sizes are proportional to their degrees.}
\label{fig:chameleon_X}
\end{figure}

\section{Discussion and Conclusion} \label{sec:conclusion}
In this paper, we have considered the problem of high-dimensional network association testing, in which we aim to understand whether observed network structure is associated with (potentially high-dimensional) vertex features.
To this end, we considered a model in which the network is generated via an RDPG, and we suggest two types of methods: model-based and model-free methods.
The model-based methods assume that the latent positions of the RDPG are linearly related to the vertex features.
We then estimate this relation using ridge, LASSO and group LASSO regression.
In the model-free approach, we are agnostic as to the relation between latent positions and features, and instead we aim to estimate potential relations between them using CCA between features and either latent positions or rows of the adjacency matrix.
All of our proposed methods operate via permutation tests, and we have applied them to both simulated and real world data in both low- and high-dimensional settings.
Our theoretical and empirical results suggest that when a linear model relating features to network structure is correct, our model-based methods perform quite well, while our CCA-based methods perform better when these linearity assumptions are violated.

The five methods that we have suggested come with different advantages and disadvantages. The main restriction for the model-based methods is that they rely on a linear model. However, in that case, we are able to prove performance guarantees in the form of an oracle inequality for the ridge regression and a convergence rate for the group LASSO, even in the high-dimensional case.
We relaxed the linearity assumption by using CCA to test for network association. These methods all require an initial estimate of the latent positions which comes with its own challenges, e.g., selecting the latent dimension. Therefore, we finally presented a network CCA method that performs CCA alignment between the adjacency matrix and network features. For these two methods, we have shown consistency in the low-dimensional regime.
Our empirical results show that their performance remains good in high-dimensional settings, provided one regularizes appropriately.

This research can be extended in several interesting directions.
Our focus here was correlation, that is, linear dependence.
A natural extension is to consider non-linear dependence, i.e., to test for dependence as opposed to correlation.
Distance correlation \citep{LeeShePriVog2019} actually tests for independence, but we saw in our empirical results that it is computationally expensive compared to our proposed methods.
One possible next step would be to study inference, e.g., using de-biasing or post-LASSO in our LASSO-based methods. Similarly, we have not discussed variable screening, i.e., selecting the most relevant features with high-probability. Likely, under a so-called beta-min condition \citep[see, e.g.,][Section 7.7]{vdGB11}, the LASSO theory can be extended to such results, but it is unclear how the estimated latent positions affect its behavior in this regard.

In another direction, it would be interesting to extend these methods to multiple network observations. This could either consider repeated independent realisations of the same network, or the observation of several networks in which the exact relation to the covariates is different in each network, but, for example, the identity of relevant covariates is the same across all networks. Lastly, it is worthwhile to consider Remark~\ref{rem:ZfromA}, i.e., settings in which the covariates derive from the network itself rather than latent positions.
This setting promises a more general set of tools, but will doubtless require more careful bookkeeping and analysis than that undertaken here.

\section*{Acknowledgements} \label{sec:acknowledgements}
The authors wish to thank
Cencheng Shen,
Elizaveta Levina, 
Youngser Park,
Carey E. Priebe,
Karl Rohe, 
Jonathan Stewart,
Ji Zhu
and the attendees of the University of Wisconsin--Madison IFDS Ideas Seminar
for helpful comments and discussion that greatly improved the paper.
The authors also thank Youjin Lee for sharing R implementations related to \cite{LeeShePriVog2019}.
Computations for this work were done using resources of the Leipzig University Computing Center.

\section*{Funding} \label{sec:Funding}
This research was funded by the Deutsche Forschungsgemeinschaft (DFG, German
Research Foundation) – 534099487.
KL was supported in part by the University of Wisconsin--Madison Office of the Vice Chancellor for Research and Graduate Education with funding from the Wisconsin Alumni Research Foundation.

\bibliographystyle{plainnat}
\bibliography{biblio}

\newpage

\appendix

\section{Proof of Theorem~\ref{thm:ridge:convergence}} \label{apx:ridge}

Here, we prove the convergence of our method based on ridge regression.

\begin{proof}[Proof of Theorem~\ref{thm:ridge:convergence}]
By Assumption~\ref{assum:ttiRate}, 
there is a sequence of random matrices $Q$ such that
\begin{equation} \label{eq:XhatX:uniform}
\max_{i \in [n]} \max_{k \in[d]} \left| \Xhat_{i, k} - (XQ)_{i, k} \right|
= \Op{ \ttiRate_n }.
\end{equation}
Recalling the definitions of $\bhat^{(k)}$ and $\btilde^{(k)}_Q$ 
from Equations~\eqref{eq:def:bhatk} and~\eqref{eq:def:btildek},
respectively,
\begin{equation*} 
\bhat^{(k)} - \btilde_Q^{(k)}
= \left( \frac{1}{n} Z^\top  Z + \lassoreg_k I_p \right)^{-1}
        \frac{1}{n} Z^\top  \left( \Xhat_{\cdot, k} - (XQ)_{\cdot, k} \right) .
\end{equation*}
By submultiplicativity of the norm,
\begin{equation} \label{eq:btildek:bhatk:diff}
\left\| \bhat^{(k)} - \btilde^{(k)}_Q \right\|
\le
\left\| \left( \frac{1}{n} Z^\top  Z + \lassoreg_k I_p \right)^{-1} \right\|
\left\| \frac{1}{n} Z^\top  \left(\Xhat_{\cdot,k} - (XQ)_{\cdot,k}\right) \right\|.
\end{equation}

Trivially lower-bounding the eigenvalues of $n^{-1} Z^\top  Z + \lassoreg_k I_p$ by $\lassoreg_k > 0$, we have
\begin{equation} \label{eq:covmx:spectral}
\max_{k \in [d]}
        \left\| \left( \frac{1}{n} Z^\top  Z + \lassoreg_k I_p \right)^{-1} \right\|
\le \frac{ 1 }{ \min_{k \in [d]} \lassoreg_k }.
\end{equation}
Unrolling the matrix-vector product and applying the triangle inequality, for $j \in [p]$,
\begin{equation} \label{eq:expand:entryj}
\begin{aligned}
\left|
\left[ \frac{1}{n} Z^\top  \left( \Xhat_{\cdot, k} - (XQ)_{\cdot, k} \right) \right]_j
\right|
&\le
\frac{1}{n} \sum_{i=1}^n \left| \Xhat_{i, k} - (XQ)_{i, k} \right| \left| Z_{ij} \right| \\
&\le \left( \max_{i \in [n]} \left| \Xhat_{i, k} - (XQ)_{i, k} \right| \right)
        \frac{1}{n} \sum_{i=1}^n \left| Z_{ij} \right|. 
\end{aligned} \end{equation}
Since by assumption the rows of $Z$ are independent and entries within each row are marginally $(\nu,b)$-sub-gamma, it follows by standard concentration inequalities that %with probability at least $1-O(p^{-1}n^{-2})$,
%\begin{equation*}
%\frac{1}{n} \sum_{i=1}^n \left| Z_{ij} \right|
%\le \E |Z_{1j}| + \Op{ \frac{ \sqrt{\nu + b^2} \log^{1/2} np }{ \sqrt{n} } }.
%\end{equation*}
%A union bound over all $j \in [p]$ yields that with probability at least $1-O(n^{-2})$,
\begin{equation} \label{eq:absZ:entrywise}
\begin{aligned} 
\max_{j \in [p]} \frac{1}{n} \sum_{i=1}^n \left| Z_{ij} \right|
&\le \max_{j \in [p]} \E |Z_{1j}| 
	+ \frac{ C\sqrt{\nu_Z + b_Z^2} \log^{1/2} np }{ \sqrt{n} } \\
&\le C + \Op{ \frac{ \sqrt{\nu_Z + b_Z^2} \log^{1/2} np }{ \sqrt{n} } }
= \Op{ 1 },
\end{aligned} \end{equation}
where we have used our growth assumptions on $n,p$ and $(\nu_Z,b_Z)$.
Noting that Equations~\eqref{eq:XhatX:uniform} and~\eqref{eq:absZ:entrywise} hold uniformly over all $k \in [d]$, we may bound Equation~\eqref{eq:expand:entryj} as
\begin{equation*} 
\max_{k \in [d]} \left|
\left[ \frac{1}{n} Z^\top  \left( \Xhat_{\cdot, k} - (XQ)_{\cdot, k} \right) \right]_j
\right|
= \Op{ \ttiRate_n }. 
\end{equation*}
Applying this and Equation~\eqref{eq:covmx:spectral} to Equation~\eqref{eq:btildek:bhatk:diff}, again noting that Equation~\eqref{eq:covmx:spectral} also does not depend on $k \in [d]$, it follows that
\begin{equation*} 
\max_{k \in [d]} \left\| \btilde^{(k)}_Q - \bhat^{(k)} \right\|
= \Op{ \frac{\ttiRate_n}{ \min_{k \in [d]} \lassoreg_k } } ,
\end{equation*}
as we set out to prove.
\end{proof}

\section{Proofs for Section~\ref{subsubsec:LASSO}} \label{app:proofsLASSO}
Below, we provide technical lemmas and proof details in support of Theorem~\ref{thm:LASSO:convergence_rate} for our results related to the group LASSO estimator in Equation~\eqref{eq:def:BhatgLASSO}.
For ease of notation, in this section, we simply denote the estimator by
\begin{equation*}
\Bhat
= \Bhat^{\gLASSO}
= \begin{bmatrix}
\Bhat_{\gLASSO,1}^\top  \\ \Bhat_{\gLASSO,2} \\ \vdots \\ \Bhat_{\gLASSO,p}^\top 
\end{bmatrix}
=\argmin{\beta\in\IR^{p\times d}} \frac{1}{nd}\left\|\Xhat-Z\beta\right\|_F^2 
	+\frac{\lassoreg}{\sqrt{d}}\sum_{j=1}^p\|\beta_j\| .
\end{equation*}              
In what follows, for ease of notation, we write
\begin{equation} \label{eq:def:Delta}
\Delta = \Xhat - XQ \in \IR^{n \times d} ,
\end{equation}
where $Q \in \IR^{d \times d}$ is the orthogonal matrix guaranteed by Assumption~\ref{assum:ttiRate}.
That is, $\Delta$ is the error between our estimate $\Xhat$ of the latent positions and our target $X$, after accounting for orthogonal non-identifiability of the latent positions.

\begin{lemma} \label{lem:basic_inequality}
Let $B \in \IR^{p \times d}$ denote the true matrix of coefficients in Equation~\eqref{eq:model}, and let $\Bhat \in \IR^{p \times d}$ be the estimator defined in Equation~\eqref{eq:def:BhatgLASSO}.
With $\Delta$ as defined in \eqref{eq:def:Delta},
\begin{equation*} \begin{aligned}
&\frac{1}{nd} \left\|Z\left(BQ-\Bhat\right)\right\|_F^2
        +\frac{\lassoreg}{\sqrt{d}}\sum_{j=1}^p\|\Bhat_j\| \\
&\qquad\qquad \leq
\frac{2}{nd}
  \trace \left[Z\left(\Bhat-BQ\right)\left(EQ+\Delta\right)^\top \right]
+\frac{\lassoreg}{\sqrt{d}}\sum_{j=1}^p\|B_j\| .
\end{aligned} \end{equation*}
\end{lemma}
\begin{proof}
The proof largely follows that of Lemma 6.1 in \citet{vdGB11}, modified to account for the rotational non-identifiability of the latent positions.
By definition of $\Bhat$ as a minimizer in Equation~\eqref{eq:def:BhatgLASSO},
\begin{equation*} \begin{aligned}
\frac{1}{nd}\left\|\Xhat -Z\Bhat\right\|_F^2
	+\frac{\lassoreg}{\sqrt{d}}\sum_{j=1}^p\|\Bhat_j\|
&\leq\frac{1}{nd}\left\|\Xhat-ZBQ\right\|_F^2
	+\frac{\lassoreg}{\sqrt{d}}\sum_{j=1}^p\left\|[BQ]_{j\cdot}\right\| \\
&= \frac{1}{nd}\left\|\Xhat-ZBQ\right\|_F^2
        +\frac{\lassoreg}{\sqrt{d}}\sum_{j=1}^p\left\| B_j \right\|,
\end{aligned} \end{equation*}
where the equality follows from the unitary invarince of the Euclidean norm. Plugging in $X=ZB+E$ from Equation~\eqref{eq:model} and recalling $\Delta = \Xhat - XQ$ from Equation~\eqref{eq:def:Delta},
\begin{equation*}
\frac{1}{nd}\!
\left\|Z\!\left(BQ-\!\Bhat\right)\!+\!EQ+\!\Delta\right\|_F^2
+\frac{\lassoreg}{\sqrt{d}}\sum_{j=1}^p\left\|\Bhat_j\right\|
\leq
\frac{1}{nd}\! \left\|EQ+\Delta\right\|_F^2
	+\frac{\lassoreg}{\sqrt{d}}\sum_{j=1}^p\|B_j\|,
\end{equation*}
and the result follows from the fact that $\|A+B\|_F^2=\|A\|_F^2+\|B\|_F^2+2\trace AB^\top $ for matrices $A$ and $B$ of the same dimensions.
\end{proof}

For a given sequence $(\lassoreg_{0,n})_{n\in\IN}\subseteq\IR^\IN$, define the event
\begin{equation} \label{eq:def:calTn}
\calT_n =
\left\{ \max_{j\in[p]}\frac{2}{nd}
			\left\|\left(EQ + \Delta \right)^\top  Z_{\cdot j} \right\|
        \le \frac{\lassoreg_{0,n}}{\sqrt{d}} \right\},
\end{equation}
where $Q \in \bbO_d$ is the orthogonal matrix guaranteed by Assumption~\ref{assum:ttiRate}.
The main motivation for considering this event is the following Lemma \ref{lem:Lasso2}, which will be used in the proof of Theorem \ref{thm:LASSO:convergence_rate} to show the inequality in Equation \eqref{eq:lasso_bound}.
This inequality, in turn, is a basic performance bound for the LASSO.

\begin{lemma} \label{lem:Lasso2}
Let $\lassoreg_{0,n}$ and $\lassoreg$ be arbitrary with $\lassoreg\geq2\lassoreg_{0,n}>0$, and let $\calT_n$ be defined as in Equation~\eqref{eq:def:calTn}. Furthermore, let $S\subseteq[p]$ denote the indices of the non-zero rows of $B \in \IR^{p \times d}$, the true parameter matrix from the model in Equation \eqref{eq:model}.
On the event $\calT_n$, it holds that
\begin{equation} \label{eq:Lasso2:bound}
\frac{2}{nd}\left\|Z\left(\Bhat-BQ\right)\right\|_F^2
        +\frac{\lassoreg}{\sqrt{d}}\sum_{j\in S^c}\|\Bhat_j\|
\leq \frac{3\lassoreg}{\sqrt{d}} \sum_{j\in S}\|\Bhat_j-Q^\top B_j\| ,
\end{equation}
where $Q \in \bbO_d$ is the orthogonal matrix guaranteed by Assumption~\ref{assum:ttiRate}.
\end{lemma}
\begin{proof}
The proof follows that of Lemma 6.3 in \citet{vdGB11}.
Recall our notation that $\beta_j$ denotes the $j$-th row of a matrix $\beta$ written as a column vector and recall $\Delta = \Xhat-XQ$ from Equation~\eqref{eq:def:Delta}.
For any matrix $\beta\in\IR^{p\times d}$, by cyclicity of the trace, the triangle inequality, and the Cauchy-Schwarz inequality,
\begin{equation*} \begin{aligned}
\left|\frac{2}{nd}\trace \left(Z\beta\left(EQ+\Delta\right)^\top \right)\right|
&=\left|\frac{2}{nd}\trace \left(\beta\left(EQ+\Delta\right)^\top Z\right)\right|
\le \frac{2}{nd}\sum_{j=1}^p
	\left| \beta_j^\top \left(EQ+\Delta\right)^\top Z_{\cdot j}\right| \\
&\leq\frac{2}{nd}\sum_{j=1}^p\|\beta_j\|
	~\left\|\left(EQ+\Delta\right)^\top Z_{\cdot j}\right\| .
\end{aligned} \end{equation*}
On the event $\calT_n$, we then have
\begin{equation*}
\left|\frac{2}{nd}\trace \left(Z\beta\left(EQ+\Delta\right)^\top \right)\right|
\leq \frac{\lassoreg_{0,n}}{\sqrt{d}} \sum_{j=1}^p\|\beta_j\| .
\end{equation*}
Taking $\beta = BQ-\Bhat$, applying the above inequality to the bound in Lemma~\ref{lem:basic_inequality}, and using our assumption $\lassoreg \geq 2\lassoreg_{0,n}$,
\begin{equation} \label{eq:Lasso2:checkpt} \begin{aligned}
&\frac{2}{nd} \left\|Z\left(BQ-\Bhat\right)\right\|_F^2
        +\frac{2\lassoreg}{\sqrt{d}}\sum_{j=1}^p\|\Bhat_j\| \\
&\qquad \qquad
\leq \frac{\lassoreg}{\sqrt{d}} \sum_{j=1}^p\|\Bhat_j-Q^\top B_j\| 
+\frac{2\lassoreg}{\sqrt{d}}\sum_{j=1}^p\|B_j\| .
\end{aligned} \end{equation}
Since $B_j=0$ for all $j \in S^c$, we have
\begin{equation*}
\sum_{j=1}^p\|\Bhat_j-Q^\top B_j\|
= \sum_{j\in S}\|\Bhat_j-Q^\top B_j\|+\sum_{j\in S^c}\|\Bhat_j\|.
\end{equation*}
Applying this to Equation~\eqref{eq:Lasso2:checkpt} and again using the fact that $B_j=0$ for $j \in S^c$,
\begin{equation} \label{eq:Lasso2:checkpt2} \begin{aligned}
&\frac{2}{nd} \left\|Z\left(BQ-\Bhat\right)\right\|_F^2
        +\frac{2\lassoreg}{\sqrt{d}}\sum_{j=1}^p\|\Bhat_j\| \\
&~~~
\leq \frac{\lassoreg}{\sqrt{d}} \sum_{j \in S} \|\Bhat_j-Q^\top B_j\| 
+ \frac{\lassoreg}{\sqrt{d}} \sum_{j \in S^c}  \|\Bhat_j\|
+\frac{2\lassoreg}{\sqrt{d}}\sum_{j\in S} \|B_j\| .
\end{aligned} \end{equation}

Applying the triangle inequality and using the fact that $Q$ is an orthogonal matrix, 
\begin{equation*} \begin{aligned}
\sum_{j=1}^p\|\Bhat_j\|
%&= \sum_{j\in S}\|\Bhat_j\|+\sum_{j\in S^c}\|\Bhat_j\| \\
&\ge \sum_{j\in S}\left(\|Q^\top B_j\|-\|\Bhat_j-Q^\top B_j\|\right)
	+ \sum_{j\in S^c} \|\Bhat_j\| \\
&= \sum_{j\in S}\left(\|B_j\|-\|\Bhat_j-Q^\top B_j\|\right)
	+ \sum_{j\in S^c} \|\Bhat_j\| .
\end{aligned} \end{equation*}
Using this to lower-bound the left-hand side of Equation~\eqref{eq:Lasso2:checkpt2} and rearranging,
\begin{equation*} \begin{aligned}
\frac{2}{nd} & \left\|Z\left(BQ-\Bhat\right)\right\|_F^2
+ \frac{\lassoreg}{\sqrt{d}} \sum_{j\in S^c} \|\Bhat_j\|
\leq \frac{3\lassoreg}{\sqrt{d}} \sum_{j\in S} \|\Bhat_j-Q^\top B_j\|, 
\end{aligned} \end{equation*}
as we set out to show.
\end{proof}

The following is a basic concentration inequality \citep[see, e.g.,][]{BLM13,Vershynin2020}.
We include a proof here for the sake of completeness.

\begin{lemma} \label{lem:cond_subGamma}
Let $(Z_i,\epsilon_i)_{i\in[n]}$ be a sequence of independent, identically distributed random variables with $Z_i$ and $\epsilon_i$ independent for each $i \in[n]$.
Suppose that the $\epsilon_i$ are centered $(\nu_{\epsilon},b_{\epsilon})$-subgamma random variables and that $Z_i^2-\IE(Z_1^2)$ is $(\nu_Z,b_Z)$-sub-gamma for all $i \in [n]$.
Then, for any $c_1,c_2,t>0$,
\begin{equation*} \begin{aligned}
\IP&\left[\left|\frac{1}{\sqrt{n}}\sum_{i=1}^nZ_i\epsilon_i\right| 
> \! \sqrt{2\left(\sqrt{2\nu_Zc_1}+b_Zc_1+\IE Z_1^2 \right)\nu_{\epsilon}t}
	+b_{\epsilon} t \sqrt{\frac{\sqrt{2\nu_Z c_2n}+b_Z c_2 n
	+\IE Z_1^2}{n}}\right] \\
&\leq e^{-t}+e^{-c_1 n}+ne^{-c_2 n} .
\end{aligned} \end{equation*}
\end{lemma}
\begin{proof}
For ease of notation, let $E_0$ denote the event that
\begin{equation} \label{eq:def:event:E0}
\left|\frac{1}{\sqrt{n}}\sum_{i=1}^nZ_i\epsilon_i\right|
\le \sqrt{2\left(\sqrt{2\nu_Zc_1}+b_Zc_1+\IE Z_1^2\right)\nu_{\epsilon}t}
        +b_{\epsilon} t \sqrt{\frac{\sqrt{2\nu_Z c_2n}+b_Z c_2 n
        +\IE Z_1^2}{n}} 
\end{equation}
and write $Z:=(Z_1,Z_2,\dots,Z_n)$.

We first note that conditionally on $Z_i$, $Z_i\epsilon_i$ is a $(Z_i^2\nu_{\epsilon},|Z_i|b_{\epsilon})$-sub-gamma random variable.
Thus, our independence assumptions imply that $\sum_{i=1}^n Z_i\epsilon_i$ is,
conditional on $Z_1,Z_2,\dots,Z_n$, a sub-gamma random variable with parameters
\begin{equation*}
\left( \nu_\star, b_\star \right)
=
\left(\sum_{i=1}^nZ_i^2\nu_{\epsilon},~b_{\epsilon}\max_{i\in[n]}|Z_i|\right) .
\end{equation*}
By standard concentration inequalities \citep[see][Section 2.4]{BLM13}, for any $t>0$,
\begin{equation} \label{eq:Zepsbound}
\IP\left[ \left|\frac{1}{\sqrt{n}}\sum_{i=1}^nZ_i\epsilon_i\right|
	>\sqrt{\frac{2 \nu_{\epsilon}t}{n}\sum_{i=1}^nZ_i^2 }
	+ \frac{ b_{\epsilon}t \max_{i\in[n]}|Z_i|}{\sqrt{n}}\Bigg| Z \right]
\leq e^{-t}.
\end{equation}
Define the events
\begin{equation*} \begin{aligned}
E_1 &= \left\{ \frac{1}{n}\sum_{i=1}^nZ_i^2
	\le \sqrt{2\nu_Zc_1}+b_Zc_1+\IE Z_1^2
	\right\} \\
\text{ and }~
E_2 &= \left\{ \max_{i\in[n]}|Z_i|^2
        \le \sqrt{2\nu_Z c_2 n}+b_Z c_2n+\IE Z_1^2 \right\} ,
\end{aligned} \end{equation*}
where $c_1$ and $c_2$ are positive constants.
We have
\begin{equation} \label{eq:gLASSO:splitevents} 
\IP\left[ E_0^c \right]
\le \IP\left[ E_0^c , E_1 \cap E_2 \right]
+ \IP\left[ E_1^c \right]
+ \IP\left[ E_1^c \right] .
\end{equation}

Since since the $Z_i^2-\IE Z_1^2$ are $(\nu_Z,b_Z)$-sub-gamma by assumption, $\sum_{i=1}^n\left(Z_i^2-\IE Z_1^2\right)$ is $(n\nu_Z,b_Z)$-sub-gamma, and thus
\begin{equation*} \label{eq:Zsquared}
\begin{aligned}
\IP\left[ E_1^c \right]
&= \IP\left[
	\sum_{i=1}^n\left(Z_i^2-\IE Z_1^2 \right)
	> \sqrt{2n\nu_Z\cdot c_1n}+b_Z c_1n\right]
\leq e^{-c_1n}. 
\end{aligned} \end{equation*}
Similarly, by the union bound,
\begin{equation} \label{eq:Zmax}
\begin{aligned}
\IP\left[ E_2^c \right]
&\leq
n\IP\left[|Z_i|^2\!-\!\IE Z_1^2>\sqrt{2\nu_Z c_2n}+b_Z c_2n\right]
\leq ne^{-c_2n}.
\end{aligned} \end{equation}
Applying the above two displays to Equation~\eqref{eq:gLASSO:splitevents},
\begin{equation} \label{eq:gLASSO:splitevents2} 
\IP\left[ E_0^c \right]
\le \IP\left[ E_0^c , E_1 \cap E_2 \right]
+ e^{-c_1n} + ne^{-c_2n}.
\end{equation}

Note that on the event $E_1 \cap E_2$,
\begin{equation*} \begin{aligned}
\sqrt{\frac{2 \nu_{\epsilon}t}{n}\sum_{i=1}^nZ_i^2 }
+ \frac{ b_{\epsilon}t \max_{i\in[n]}|Z_i|}{\sqrt{n}}
&\le
\sqrt{ 2  \nu_{\epsilon}t }
\sqrt{ \sqrt{2\nu_Zc_1}+b_Zc_1+\IE Z_1^2 } \\
&~~~~~~+ \frac{ b_{\epsilon}t }{ \sqrt{n} }
	\sqrt{ \sqrt{2\nu_Z c_2 n}+b_Z c_2n+\IE Z_1^2 } ,
\end{aligned} \end{equation*}
from which Equation~\eqref{eq:Zepsbound} and the definition of event $E_0$ in Equation~\eqref{eq:def:event:E0} imply
\begin{equation*}
\IP\left[ E_0^c , E_1 \cap E_2 \right] \le e^{-t} .
\end{equation*}
Applying this bound to Equation~\eqref{eq:gLASSO:splitevents2} completes the proof.
\end{proof}

The following lemma is an adaptation of Lemma A.1 in \cite{T09}, which holds under general settings for $L_2$ functions.
To provide a self-contained exposition, we state and prove a simpler version for vectors.
\begin{lemma} \label{lem:gen_triangle}
Let $d,n\in\IN$ and and suppose that $a_{i,k}\in\IR$ for all $i\in [n]$ and $k \in [d]$. 
Then it holds that
\begin{equation*}
\sqrt{\sum_{k=1}^d\left(\sum_{i=1}^na_{i,k}\right)^2}
\leq
\sum_{i=1}^n \sqrt{\sum_{k=1}^da_{i,k}^2} .
\end{equation*}
\end{lemma}
\begin{proof}
The proof is a simplified version of Lemma A.1 of \cite{T09}.
For each $k \in [d]$, define
\begin{equation*}
b_k:=\sum_{i=1}^na_{i,k},
\end{equation*}
and let $b=(b_1,b_2,\dots,b_d)^\top \in\IR^d$.
Then by the Cauchy-Schwarz inequality,
\begin{equation*}
\|b\|^2
= \langle b,b\rangle
\leq\sum_{k=1}^d|b_k|\sum_{i=1}^n|a_{i,k}|=\sum_{i=1}^n\sum_{k=1}^d|b_k||a_{i,k}|\leq\sum_{i=1}^n\|b\|\sqrt{\sum_{k=1}^d|a_{i,k}|^2}.
\end{equation*}
Dividing by $\|b\|$ and recalling the definition of $b_k$ yields
\begin{equation*}
\sqrt{\sum_{k=1}^d\left(\sum_{i=1}^na_{i,k}\right)^2}
\leq\sum_{i=1}^n\sqrt{\sum_{k=1}^d|a_{i,k}|^2},
\end{equation*}
as we set out to prove.
\end{proof}

With the above technical lemmas in place, we are ready to prove our main result on convergence of the group LASSO-based method.
\begin{proof}[Proof of Theorem~\ref{thm:LASSO:convergence_rate}]
Let $\epsilon>0$ be arbitrary and choose $c_0,c_1,c_2,c_3,N_{\epsilon}>0$ such that
\begin{equation} \label{eq:LASSO:chooseConstants} \begin{aligned}
\left( \max\left\{ pd,n \right\} \right)^{1-c_0}&\leq\epsilon,
\qquad 
&\exp\left\{ -n\left(c_1-\frac{\log pd}{n}\right) \right\}
\leq \epsilon, \\
\exp\left\{-n\left(c_2-\frac{\log pd}{n}-\frac{\log n}{n}\right) \right\}
	&\leq\epsilon,
~~~\text{ and } 
&\exp\left\{-n\left(c_3-\frac{\log p}{n}\right) \right\}  \leq \epsilon
\end{aligned}
\end{equation}
for all $n\geq N_{\epsilon}$.

Note that this is possible because $\log pd = o(n)$ by assumption.
Choose $C_1 > 0$ such that for all $n\geq N_{\epsilon}$ (after potentially increasing $N_{\epsilon}$, if needed)
\begin{equation*}
\IP\left[\max_{i\in[n]}\left\|\Delta_i\right\|>C_1\ttiRate_n\right] \leq \epsilon ,
\end{equation*}
which is possible in light of Assumption~\ref{assum:ttiRate}.
Recall that, by assumption, $E_{1k}$ is $(\nu,b)$-sub-gamma and $Z_{1j}^2-\IE(Z_{1j}^2)$ are $(\nu_Z,b_Z)$-sub-gamma.
With this notation, we may choose $C_2>0$ so that for all $n\geq N_{\epsilon}$ (again increasing $N_{\epsilon}$, if needed)
\begin{equation*} \begin{aligned}
C_2\log\max\left\{ pd,n \right\}
&\ge
\sqrt{2\left(\sqrt{2\nu_Zc_1}+b_Zc_1+ \max_{j\in[p]}\IE Z_{1j}^2 \right)
	\nu c_0\log\max\left\{pd,n\right\} } \\
&~~~~~~+ bc_0
\sqrt{ \frac{\sqrt{2\nu_Zc_2n}+b_Zc_2n+\max_{j\in[p]}\IE Z_{1j}^2 }{n}}
\log\max\left\{ pd,n \right\} ,
\end{aligned} \end{equation*}
which is possible since $\max_{j\in[p]}\IE Z_{1j}^2$ is bounded by assumption.
We assume for the remainder of the proof that $n>N_{\epsilon}$ so that all of the above bounds hold.
Finally, define
\begin{equation} \label{eq:def:Keps}
\begin{aligned}
\tilde{C}_1
&= \sqrt{\sqrt{2\nu_Zc_3}+b_Zc_3+\sup_{n\in\IN}\max_{j\in[p]}\IE(Z_{1j}^2)} \\
\text{and}~
K_{\epsilon} &= \max\left\{4\tilde{C}_1C_1,4C_2\right\}.
\end{aligned} \end{equation}
Recall the event $\calT_n$ from Equation~\eqref{eq:def:calTn}.
By Lemma~\ref{lem:Lasso2}, when the event $\calT_n$ holds,
\begin{equation*}
\sum_{j\in S^c}\|\Bhat_j- Q^\top B_j\|
\le 3  \sum_{j\in S}\|\Bhat_j-Q^\top B_j\|,
\end{equation*}
which can be seen from Equation~\eqref{eq:Lasso2:bound} using the non-negativity of the Frobenius norm and recalling that $B_j = 0$ for $j \in S^c$.
It follows that on the event $\calT_n$, the matrix $\beta=\Bhat-BQ$ is included in the set $\calB_{p,d}(S)$ in Definition~\ref{def:compatibility}.
Therefore, using Lemma \ref{lem:Lasso2} again, we obtain that on the event $\calT_n$ if $\phi_S>0$,
\begin{equation*} \begin{aligned}
\frac{2}{nd}&\left\|Z\left(\Bhat-BQ\right)\right\|_F^2
	+\frac{\lassoreg}{\sqrt{d}}\sum_{j=1}^p\|\Bhat_j-Q^\top B_j\| \\
&\leq \frac{4 \lassoreg}{\sqrt{d}}\sum_{j\in S}\|\Bhat_j-Q^\top B_j\|
\leq \frac{4\lassoreg\sqrt{|S|}}{\phi_S\sqrt{nd}}
	\left\|Z\left(\Bhat-BQ\right)\right\|_F\\
&\leq \frac{1}{nd}\left\|Z\left(\Bhat-BQ\right)\right\|_F^2
	+\frac{4\lassoreg^2|S|}{\phi_S^2}.
\end{aligned}
\end{equation*}
where the last inequality follows from the fact that $4ab\leq a^2+4b^2$.
Rearranging, when $\phi_S>0$ and the event $\calT_n$ and $\phi_S>0$ holds,
\begin{equation} \label{eq:lasso_bound}
\frac{1}{nd}\left\|Z\left(\Bhat-BQ\right)\right\|_F^2
	+ \frac{\lassoreg}{\sqrt{d}}\sum_{j=1}^p\|\Bhat_j-Q^\top B_j\|
\le \frac{4\lassoreg^2|S|}{\phi_S^2}.
\end{equation}
Note that if $\phi_S=0$, the inequality holds trivially. Equation~\eqref{eq:lasso_bound} mirrors the bounds established in Theorem 8.4 of \cite{vdGB11}, adapted to our setting and simplified slightly. Rearranging, we can deduce convergence rates from this result once we have selected $\lassoreg$ and provided that we know that $\phi_S^2$ is suitably bounded away from zero. So far, the only substantial difference between our setup and the classical LASSO \citep[e.g., as discussed in Chapters 6 through 8 of][]{vdGB11} is the definition of the set $\calT_n$, as our definition of this event must account for estimating the latent positions $X$. We have so far shown that the event $\calT_n$ implies Equation~\eqref{eq:lasso_bound}, and hence Equation~\eqref{eq:thm:LASSO:convergence_rate:1}. Thus, in order to finish the proof of Equation~\eqref{eq:thm:LASSO:convergence_rate:1}, it remains for us to establish that $\IP(\calT_n)\geq1-5\epsilon$.

By the triangle inequality,
\begin{equation} \label{eq:max_split}
\max_{j\in[p]} \frac{2}{nd}
	\left\|\left(EQ+\Delta\right)^\top  Z_{\cdot j}\right\|
\leq
\underbrace{ \max_{j\in[p]}\frac{2}{nd}\left\|\Delta^\top Z_{\cdot j}\right\|
	}_{=M_1}
+ \underbrace{\max_{j\in[p]}\frac{2}{nd}\left\|Q^\top E^\top Z_{\cdot j}\right\|
	}_{=M_2},
\end{equation}
where $\Delta$ is defined as in Equation~\eqref{eq:def:Delta} and $Q \in \IR^{d \times d}$ is the orthogonal matrix guaranteed by Assumption~\ref{assum:ttiRate}. Note that
$$\IP\left(M_1\leq\frac{\lassoreg_{0,n}}{2\sqrt{d}}, M_2\leq\frac{\lassoreg_{0,n}}{2\sqrt{d}}\right)\leq\IP\left(M_1+M_2\leq\frac{\lassoreg_{0,n}}{\sqrt{d}}\right)\leq\IP(\calT_n),$$
and, hence,
\begin{equation}
\label{eq:Tnc_bound}
\IP(\calT_n^c)\leq \IP\left[ M_1 + M_2 \ge \frac{\lassoreg_{0,n}}{\sqrt{d}} \right]
\leq \IP\left[ M_1 \ge \frac{\lassoreg_{0,n}}{2\sqrt{d}} \right]
        + \IP\left[ M_2 \ge \frac{\lassoreg_{0,n}}{2\sqrt{d}} \right].
\end{equation}
We will prove that each of the probabilities on the right-hand side of Equation~\eqref{eq:Tnc_bound} is bounded by multiples of $\epsilon$.
We first control $M_2$.
Applying a union bound over $j \in [p]$ and using the fact that $\|Q^\top E^\top Z_{\cdot j}\| = \|E^\top Z_{\cdot j}\|$ since $Q$ is orthogonal, we obtain
\begin{equation*} \begin{aligned}
\IP\left[ M_2\geq \frac{2C_2}{\sqrt{dn}}\log\max\{pd,n\} \right]
&\leq p\max_{j\in[p]}
        \IP\left[\frac{2}{nd} \left\|Q^\top E^\top Z_{\cdot j}\right\|
                \geq \frac{2C_2}{\sqrt{dn}}\log\max\{pd,n\} \right] \\
&= p\max_{j\in[p]}\IP\left[ \frac{1}{n^2d^2}\sum_{k=1}^d
        \left (E_{\cdot k}^\top Z_{\cdot j}\right)^2
	\geq\frac{C_2^2}{dn} \log^2\max\{pd,n\} \right].
\end{aligned} \end{equation*}
Applying a second union bound, this time over $k \in [d]$, and taking square roots,
\begin{equation*} \begin{aligned}
\IP\left[ M_2\geq \frac{2C_2}{\sqrt{dn}}\log\max\{pd,n\} \right]
&\le 
pd\max_{j\in[p]}\max_{k\in[d]}
\IP\left[ \frac{1}{n^2}\left(E_{\cdot k}^\top Z_{\cdot j}\right)^2
        \geq \frac{C_2^2}{n}\log^2\max\{pd,n\} \right] \\ 
&=
pd\max_{j\in[p]}\max_{k\in[d]}
\IP\left[ \left|\frac{1}{\sqrt{n}}\sum_{i=1}^nE_{ik}Z_{ij}\right|
\geq C_2\log\max\{pd,n\} \right].
\end{aligned}
\end{equation*}
Since $(E_{ik}Z_{ij})_{i\in[n]}$ is a sequence of independent random variables with $E_{ik}$ and $Z_{ij}$ being independent, $E_{ik}$ is $(\nu,b)$-sub-gamma and $Z_{ij}^2-\IE Z_{ij}^2$ is $(\nu_Z,b_Z)$-sub-gamma (both for all $i\in[n]$), the assumptions of Lemma~\ref{lem:cond_subGamma} are fulfilled. We, hence, conclude by the choice of $C_2$ that
\begin{equation*} \begin{aligned}
\IP&\left[ M_2\geq \frac{2C_2}{\sqrt{dn}}\log\max(pd,n) \right] \\
&\le pd\max_{j\in[p]}\max_{k\in[d]}
	\IP\Bigg[\left|\frac{1}{\sqrt{n}}\sum_{i=1}^nE_{ik}Z_{ij}\right|
		\geq \sqrt{2\left(\sqrt{2\nu_Zc_1}+b_Zc_1+\IE Z_{1j}^2 \right)
		\nu c_0\log\max(pd,n)} \\
&~~~~~~~~~~~~~~~~~~~~~~~~~~~~~~~~~~~~~~~~~~~~~~~~~+
b c_0\log\max(pd,n)\sqrt{\frac{\sqrt{2\nu_Zc_2n}+b_zc_2n+\IE(Z_{1j}^2)}{n}}
	\Bigg] \\
&\leq pd\left(e^{-c_0\log\max(pd,n)}+e^{-c_1n}+ne^{-c_2n}\right) .
\end{aligned} \end{equation*}
By our choice of $c_0,c_1,c_2$ in Equation~\eqref{eq:LASSO:chooseConstants}, it follows that
\begin{equation} \label{eq:M2}
\begin{aligned}
\IP \left[ M_2\geq \frac{2C_2}{\sqrt{dn}}\log\max\{pd,n\} \right]
&\leq
\max\{pd,n\}^{1-c_0}+e^{-n\left(c_1-\frac{\log pd}{n}\right)}
	+e^{-n\left(c_2-\frac{\log pd}{n}-\frac{\log n}{n}\right)} \\
&\leq 3\epsilon.
\end{aligned} \end{equation}
%\begin{equation*}
%M_2 = \Op{ \frac{1}{\sqrt{nd}} \log \max\{pd,n\} }.
%\end{equation*}

To control $M_1$, we begin by applying Jensen's inequality and using the fact that $Z_{1j}^2-\IE Z_{1j}^2$ is $(\nu_Z,b_Z)$-sub-gamma by assumption to estimate
\begin{equation*} \begin{aligned}
\IP\left[ \frac{1}{n}\sum_{i=1}^n|Z_{ij}|
	>\sqrt{\sqrt{2\nu_Zc_3}+b_Zc_3 +\IE Z_{1j}^2 } \right]
&\le
\IP\left[ \frac{1}{n}\sum_{i=1}^nZ_{ij}^2>\sqrt{2\nu_Zc_3}+b_Zc_3+\IE Z_{1j}^2
	\right]  \\
&=
\IP\left[\sum_{i=1}^nuZ_{ij}^2-\IE Z_{1j}^2)
	>\sqrt{2\nu_Zn~ c_3n}+b_Z c_3n \right] \\
&\leq e^{-c_3n}.
\end{aligned} \end{equation*}
Using $\tilde{C}_1\geq\sqrt{\sqrt{2\nu_Zc_3}+b_Zc_3+\max_{j\in[p]}\IE(Z_{1j}^2)}$, taking a union bound over $j \in [p]$ and recalling our choice of $c_3$ from Equation~\eqref{eq:LASSO:chooseConstants},
\begin{equation} \label{eq:M1:Zbound}
\IP\left[\max_{j\in[p]}\frac{1}{n}\sum_{i=1}^n|Z_{ij}|>\tilde{C}_1\right]
\leq pe^{-c_3n}=e^{-n\left(c_3-\frac{\log p}{n}\right)}
\leq\epsilon.
\end{equation}
Recalling the definition of $M_1$ from Equation~\eqref{eq:max_split}, expanding the matrix product and applying Lemma~\ref{lem:gen_triangle} with $a_{i,k}=\Delta_{ik}Z_{ij}$,
\begin{equation*} \begin{aligned}
M_1
&= \max_{j\in[p]}\frac{2}{nd}\left\|\Delta^\top Z_{\cdot j}\right\|
=\frac{2}{nd}\max_{j\in[p]}\sqrt{\sum_{k=1}^d                                                   \left(\sum_{i=1}^n\Delta_{ik}Z_{ij}\right)^2} \\
&\le \frac{2}{nd}
        \max_{j\in[p]}\sum_{i=1}^n\sqrt{\sum_{k=1}^d\Delta_{ik}^2Z_{ij}^2}
=\frac{2}{nd}\max_{j\in[p]}\sum_{i=1}^n|Z_{ij}|\cdot\left\|\Delta_i\right\| \\
&\le \frac{2}{d} \max_{i\in[n]}\left\|\Delta_i\right\| \max_{j\in[p]}\frac{1}{n}\sum_{i=1}^n\left|Z_{ij}\right|.
\end{aligned} \end{equation*}                                                   Applying Equation~\eqref{eq:M1:Zbound} and using Assumption~\ref{assum:ttiRate} to control the rows of $\Delta$ ,
\begin{align} 
\IP\left(M_1>\frac{2\tilde{C}_1C_1}{d}\ttiRate_n\right)\leq&\IP\left(\frac{2}{d} \max_{i\in[n]}\left\|\Delta_i\right\| \max_{j\in[p]}\frac{1}{n}\sum_{i=1}^n\left|Z_{ij}\right|>\frac{2\tilde{C}_1C_1}{d}\ttiRate_n\right) \nonumber \\
\leq&\IP\left(\frac{2}{d} \max_{i\in[n]}\left\|\Delta_i\right\| \tilde{C}_1>\frac{2\tilde{C}_1C_1}{d}\ttiRate_n\right)+\IP\left(\max_{j\in[p]}\frac{1}{n}\sum_{i=1}^n|Z_{ij}|>\tilde{C}_1\right) \nonumber \\
\leq&\IP\left( \max_{i\in[n]}\left\|\Delta_i\right\| >C_1\ttiRate_n\right)+\epsilon\leq2\epsilon, \label{eq:M1}
\end{align}
by choice of $C_1$.
Now, by choice of $\lassoreg_{0,n}$ and recalling our choice of $K_{\epsilon}$ from Equation~\eqref{eq:def:Keps}, we conclude from Equations~\eqref{eq:M1} and \eqref{eq:M2} that
\begin{align*}
\IP\left(M_1>\frac{\lassoreg_{0,n}}{2\sqrt{d}}\right)
&\leq \IP\left(M_1>\frac{2\tilde{C}_1C_1}{d}\ttiRate_n\right)\leq2\epsilon \\ 
\text{and }~ 
\IP\left(M_2>\frac{\lassoreg_{0,n}}{2\sqrt{d}}\right)
&\leq \IP\left(M_2>\frac{2C_2}{\sqrt{dn}}\log\max(pd,n)\right)
\leq 3 \epsilon.
\end{align*}
Applying the above to bound Equation~\eqref{eq:Tnc_bound},
\begin{equation*} \begin{aligned}
\IP\left[ \calT_n^c \right]\leq2\epsilon+3\epsilon=5\epsilon.
\end{aligned} \end{equation*}
Hence, with probability at least $1-5\epsilon$ the event $\calT_n$ holds.
Hence, we obtain from \eqref{eq:lasso_bound} that, with probability $1-5\epsilon$,
\begin{equation*}
\frac{1}{nd}\left\| Z\left(\Bhat-BQ\right) \right\|_F^2
+\frac{\lassoreg}{\sqrt{d}} \sum_{j=1}^p\left\|\Bhat_j-Q^\top B_j\right\|
\leq \frac{4\lassoreg^2|S|}{\phi_S^2},
\end{equation*}
yielding Equation~\eqref{eq:thm:LASSO:convergence_rate:1}.

We turn now to the proof of Equation~\eqref{eq:thm:LASSO:convergence_rate:2}.
Note that, since $\log pd = o(n)$ by assumption, the constants $c_0,c_1,c_2,c_3$ as defined in \eqref{eq:LASSO:chooseConstants} do not depend on $\epsilon$.
As a consequence, $C_2$ and $\tilde{C}_1$ do also not depend on $\epsilon$.
However, $C_1$ depends, in general on $\epsilon$.
Assumption \ref{assum:ttiRate} implies that $\IP[\max_{i\in[n]}\|\Delta_i\|>\extradeltarate\ttiRate_n]\leq\epsilon$ for any $\epsilon>0$, provided $n$ is large enough, that is, for all $n\geq N_{\epsilon}$ after potentially increasing $N_{\epsilon}$.
We may therefore repeat the previous proof with $\extradeltarate$ in place of $C_1$ and using
\begin{equation*}
\lassoreg_{0,n}
=\max\left\{\frac{4\tilde{C}_1\extradeltarate}{\sqrt{d}}\ttiRate_n,
		\frac{4C_2}{\sqrt{n}}\log\max(pd,n)\right\}.
\end{equation*}
Thus, we have shown that, for any $\epsilon>0$, Equation~\eqref{eq:thm:LASSO:convergence_rate:1} holds with probability $1-5\epsilon$ for $n\geq N_{\epsilon}$ also for $\lassoreg=2\lassoreg_{0,n}$ with $\lassoreg_{0,n}$ as above.
Moreover,
\begin{equation*}
\lassoreg
=
\Op{\max\left\{\frac{\extradeltarate}{\sqrt{d}}\ttiRate_n,
		\frac{\log\max\{pd,n\}}{\sqrt{n}}\right\}}. 
\end{equation*}
Let now $\epsilon>0$ be fixed and suppose that $n\geq N_{\epsilon}$.
We obtain from Equation~\eqref{eq:thm:LASSO:convergence_rate:1} that
\begin{equation*}
\frac{1}{\sqrt{d}}\sum_{j=1}^p\left\|\Bhat_j-Q^\top B_j\right\|
\leq\frac{4\lassoreg|S|}{\phi_S^2}
\end{equation*}
with probability $1-5\epsilon$ for $n$ large enough.
Since $\epsilon>0$ was arbitrary, we conclude
\begin{equation*}
\frac{1}{\sqrt{d}} \sum_{j=1}^p\left\|\Bhat_j-Q^\top B_j\right\|           
=\Op{\frac{|S|}{\phi_S^2}\max\left\{\frac{\extradeltarate}{\sqrt{d}}\ttiRate_n,\frac{\log\max(pd,n)}{\sqrt{n}}\right\}}.
\end{equation*}
This implies \eqref{eq:thm:LASSO:convergence_rate:2} for the special case $\phi_S^{-1}=\Op{ 1 }$, which completes the proof.
\end{proof}

\section{Supporting Results for Spectral Norm Bounds} \label{apx:subspace}

Here we collect results that will be useful in proving our results concerning the CCA-based methods presented in Section~\ref{subsec:cca-results}.
These largely concern the behavior of the matrix of latent positions $X$ and related eigenvalues and subspaces.
Our first such result characterizes the growth rate of the singular values of $X$.

\begin{lemma} \label{lem:svals:X}
Under Assumption~\ref{assum:invertible}, letting $\sigma_k(X)$ denotes the $k$-th singular value of $X$,
\begin{equation*} 
\sigma_k( X ) = \Thetap{ \sqrt{n} }~\text{ uniformly over } k \in [d].
\end{equation*}
\end{lemma}
\begin{proof}
By the law of large numbers,
\begin{equation*}
\frac{1}{n} X^\top X \rightarrow \Delta_X~\text{ in probability,}
\end{equation*}
with $\Delta_X = \Sigma_X + (\E X_1)(\E X_1)^\top$ being of full rank by Assumption~\ref{assum:invertible}.
The result follows after observing that the $d$ non-zero eigenvalues of $X X^\top$ are the same as those of $X^\top X$.
\end{proof}

\begin{lemma} \label{lem:control:SigmaX}
% Lemma~\ref{lem:svals:X} assumptions.
Under Assumption~\ref{assum:invertible}, with $\Sigmatilde_X$ as defined in Equation~\eqref{eq:def:SigmatildeMarginal}, it holds that
\begin{equation} \label{eq:SigmatildeX:SigmaX}
\left\| \Sigmatilde_X - \Sigma_X \right\| = \op{ 1 } .
\end{equation}
Further, suppose that
%Used in course of proof.
Assumption~\ref{assum:ttiRate} holds with $\ttiRate_n = o(1)$.
Then
\begin{equation} \label{eq:SigmahatX:SigmatildeX}
\left\| \frac{1}{n} \Xhat^\top M_n \Xhat - \Sigmatilde_{XQ} \right\|
= \Op{ \ttiRate_n } ,
\end{equation}
where $\ttiRate_n$ and $Q$ are the rate and orthogonal matrix guaranteed by Assumption~\ref{assum:ttiRate}, respectively.
\end{lemma}
\begin{proof}
Recalling the definition of $\Sigmatilde_X$ from Equation~\eqref{eq:def:SigmatildeMarginal}, Equation~\eqref{eq:SigmatildeX:SigmaX} follows immediately from the law of large numbers applied to
\begin{equation*}
\frac{1}{n} X^\top M_n X
= \frac{1}{n} \sum_{i=1}^n \left( X_i - \xbar \right)
		\left( X_i - \xbar \right)^\top .
\end{equation*}

To establish Equation~\eqref{eq:SigmahatX:SigmatildeX}, we begin by writing
\begin{equation*} 
\frac{1}{n} \Xhat^\top M_n \Xhat - \Sigmatilde_{XQ}
= \frac{1}{n} \left( \Xhat - XQ \right)^\top M_n \Xhat
        + \frac{1}{n} X^\top M_n \left( \Xhat - XQ \right).
\end{equation*}
Applying the triangle inequality and using the fact that $M_n$ is a projection and that $Q$ is orthogonal,
\begin{equation*}
\left\| \frac{1}{n} \Xhat^\top M_n \Xhat - \Sigmatilde_{XQ} \right\|
\le \frac{1}{n} \left\| \Xhat - QX \right\|^2
        + \frac{2}{n} \left\| \Xhat - QX \right\| \| X \|.
\end{equation*}
Applying the trivial upper bound
\begin{equation*}
\left\| \Xhat - QX \right\| \le \sqrt{n} \left\| \Xhat - QX \right\|_{\tti}
\end{equation*}
and using Lemma~\ref{lem:svals:X}, it follows that
\begin{equation*}
\left\| \frac{1}{n} \Xhat^\top M_n \Xhat - \Sigmatilde_{XQ} \right\|
=O_P\left( \left\| \Xhat - QX \right\|_{\tti}^2
        + \left\| \Xhat - QX \right\|_{\tti} \right).
\end{equation*}
Controlling the $(\tti)$-norm with Assumption~\ref{assum:ttiRate} and using our assumption that $\ttiRate_n = o(1)$ yields Equation~\eqref{eq:SigmahatX:SigmatildeX}.
\end{proof}

\begin{lemma} \label{lem:svalgrowth:SigmatildeXX}
% Assumptions for Lemma~\ref{lem:svals:X}
Suppose that Assumption~\ref{assum:invertible} holds and define
\begin{equation} \label{eq:def:SigmatildeP}
\Sigmatilde_{XX^\top} = \frac{1}{n} X X^\top M_n X X^\top .
\end{equation}
Then the $d$-th singular value of $\Sigmatilde_{XX^\top}$ obeys
\begin{equation} \label{eq:SigmatildeP:spectrum}
\sigma_d\left( \Sigmatilde_{XX^\top} \right) = \Omegap{ n }.
\end{equation}
\end{lemma}
\begin{proof}
Recalling the definition of $\Sigmatilde_X$ from Equation~\eqref{eq:def:SigmatildeMarginal},
\begin{equation*}
\Sigmatilde_{XX^\top} 
= X X^\top \frac{1}{n} M_n X X^\top
= X \Sigmatilde_X X^\top.
\end{equation*}
Writing $P = X X^\top = U S U^\top$ via the spectral decomposition, we may take $X = U S^{1/2}$ without loss of generality \citep[see, e.g.,][for further discussion]{AthFisLevLyzParQinSusTanVogPri2018,LLL22},
\begin{equation} \label{eq:SigmatildeP:factorize}
\Sigmatilde_{XX^\top} = X \Sigmatilde_X X^\top
= U S^{1/2} \Sigmatilde_X S^{1/2} U^\top.
\end{equation}

By Lemma~\ref{lem:control:SigmaX}, all $d$ eigenvalues of $\Sigmatilde_X$ are of constant order and bounded away from zero.
Lemma~\ref{lem:svals:X} implies that the diagonal entries of $S = \diag(s_1,s_2,\dots,s_d)$ satisfy
\begin{equation} \label{eq:svalsX:finerpoint}
\min_{k \in [d]} s_k = \Omegap{ n }.
\end{equation}
Write
\begin{equation*}
S^{1/2} \Sigmatilde_X S^{1/2} = \Qtilde \Stilde \Qtilde^\top
\end{equation*}
for $\Stilde \in \R^{d \times d}$ diagonal and $\Qtilde \in \bbO_d$.
The invertibility of $\Sigmatilde_X$ and Equation~\eqref{eq:svalsX:finerpoint} ensure that the entries of $\Stilde$ all grow at rate $\Omegap{n}$, and it follows that, plugging the above display into Equation~\eqref{eq:SigmatildeP:factorize},
\begin{equation*}
\Sigmatilde_{XX^\top} U \Qtilde = U \Qtilde \Stilde,
\end{equation*}
which is to say that $\Sigmatilde_{XX^\top}$ has at least $d$ eigenvalues that grow as $\Omegap{n}$, as we set out to show.
\end{proof}

\section{Proof of Theorem~\ref{thm:CCA:XhatZ}} \label{apx:cca:XtoZ}

Here, we prove Theorem~\ref{thm:CCA:XhatZ}, showing that the CCA coefficient using $\Xhat$ is, in the large-$n$ limit, equivalent to using the true but unknown latent positions $X$.

\begin{proof}[Proof of Theorem~\ref{thm:CCA:XhatZ}]
By definition, $\rho_{\Xhat,Z}$ and $\rho_{X,Z}$ are the leading singular values of, respectively,
\begin{equation*}
\CCA( \Xhat, Z )~\text{ and }~\CCA( X, Z ) .
\end{equation*}
Since singular values are invariant to orthogonal transformation, it will suffice for us to bound
\begin{equation*} %\label{eq:thm:CCA:XhatZ:goal}
\left\| Q^\top \CCA( \Xhat, Z ) - \CCA( X, Z ) \right\| 
\end{equation*}
for some choice of $Q \in \bbO_d$ to be specified below.

Applying the singular value decomposition, 
\begin{equation} \label{eq:MX:SVD}
\frac{ 1 }{ \sqrt{n} } M_n X = W R V^\top,
\end{equation}
where $W \in \R^{n \times d}$ has orthonormal columns, $R=\diag(r_1,r_2,\dots,r_d)$ for $r_1 \ge r_2 \ge \cdots r_d \ge 0$, and $V \in \bbO_d$.
Analogously, write
\begin{equation} \label{eq:MXhat:SVD}
\frac{ 1 }{ \sqrt{n} } M_n \Xhat = \What \Rhat \Vhat^\top,
\end{equation}
where $\What \in \R^{n \times d}$, $\Vhat \in \bbO_d$ and $\Rhat = \diag( \rhat_1,\rhat_2,\dots,\rhat_d)$ with $\rhat_1 \ge \rhat_2 \ge \cdots \rhat_d \ge 0$.
We note that by Lemma~\ref{lem:control:SigmaX}, along with the fact that $\|\tilde{\Sigma}_{XQ_X}\|=\|\tilde{\Sigma}_X\|$ for $Q_X=Q$ from Assumption \ref{assum:ttiRate} and our assumption that $\Sigma_X$ is invertible, the diagonal entries of both $\Rhat$ and $R$ converge to non-zero constants.
Thus, recalling the definitions from Equations~\eqref{eq:def:CCA:XtoZ} and~\eqref{eq:def:CCA:XhatZ} and recalling that $M_n$ is idempotent, we have 
\begin{equation*} \begin{aligned}
Q^\top \CCA( \Xhat, Z ) &- \CCA( X, Z ) \\
&= \left[ Q^\top \left( \frac{\Xhat^\top M_n \Xhat}{n} \right)^{-1/2}
	\frac{ \Xhat^\top M_n }{ \sqrt{ n } }
	- \left( \frac{ X^\top M_n X}{n} \right)^{-1/2}
        \frac{ X^\top M_n }{ \sqrt{ n } }
	\right] \frac{ M_n Z }{ \sqrt{n} } \Sigmatilde_Z^{-1/2} \\
&= \left[ Q^\top \Vhat \Rhat^{-1} \Vhat^\top \Vhat \Rhat \What^\top
	- V R^{-1} V^\top V R W^\top 
	\right] \frac{ M_n Z }{ \sqrt{n} } \Sigmatilde_Z^{-1/2} \\
&= \left( Q^\top \Vhat \What^\top - V W^\top \right)
	 \frac{ M_n Z }{ \sqrt{n} } \Sigmatilde_Z^{-1/2} .
\end{aligned} \end{equation*}
By construction,
\begin{equation*}
\left\| \frac{ M_n Z }{ \sqrt{n} } \Sigmatilde_Z^{-1/2} \right\|
\le 1,
\end{equation*}
and thus submultiplicativity of the norm implies
\begin{equation} \label{eq:plugQhere}
\left\| Q^\top \CCA( \Xhat, Z ) - \CCA( X, Z ) \right\|
\le \left\| Q^\top \Vhat \What^\top - V W^\top \right\|.
\end{equation}

Let $O_W \in \bbO_d$ be the matrix satisfying
\begin{equation*}
O_W \in \arg\min_{O \in \bbO_d} \left\| \What O - W \right\|.
\end{equation*}
By Lemma 1 in \cite{CaiZha2018},
\begin{equation} \label{eq:WhatW:CaiZha:lem1}
\left\| \What O_W - W \right\|
\le \sqrt{2} \left\| \sin \Theta( \What, W ) \right\|,
\end{equation}
where $\Theta( \What, W ) \in \R^d$ denotes the diagonal matrix containing the canonical angles between the subspaces spanned by $\What$ and $W$.
We note that the diagonal entries of $\sin \Theta( \What, W )$ are precisely the non-zero singular values of $(I-\What \What^\top) W$ \citep[see, e.g., Exercise VII.1.10 in][]{Bhatia1997}.

Taking $Q$ in Equation~\eqref{eq:plugQhere} to be equal to
\begin{equation*} 
Q_W = V O_W^\top \Vhat^\top,
\end{equation*}
we then have
\begin{equation} \label{eq:CCA:XtoZ:pluginHere}
\left\| Q_W^\top \CCA( \Xhat, Z ) - \CCA( X, Z ) \right\|
\le \left\| V \left( \What O_W \right)^\top - V W^\top \right\|
%\le \left\| \What O_W - W \right\|
\le \sqrt{2} \left\| \sin \Theta( \What, W ) \right\|,
\end{equation}
where the second inequality follows from unitary invariance of the norm and Equation~\eqref{eq:WhatW:CaiZha:lem1}

Observe that since $M_n$ is idempotent, $\What$ and $W$ also span the leading rank-$d$ eigenspaces of, respectively,
\begin{equation*}
\frac{1}{n} M_n \Xhat \Xhat^\top M_n
= \What \Rhat^2 \What^\top \in \R^{n \times n}
~\text{ and }~
\frac{1}{n} M_n X X^\top M_n
= W R^2 W^\top \in \R^{n \times n} .
\end{equation*}
Thus, Theorem 2 from \cite{YuWanSam2014} implies
\begin{equation*}
\left\| \sin \Theta( \What, W ) \right\|
\le \frac{ \sqrt{d} }{ r_d^2 }
\left\| \frac{1}{n} M_n \Xhat \Xhat^\top M_n
	- \frac{1}{n} M_n X X^\top M_n \right\|
\le \frac{ \sqrt{d} }{ n r_d^2 }
\left\| \Xhat \Xhat^\top - X X^\top \right\|,
\end{equation*}
where the second inequality follows from submultiplicativity of the norm and the fact that $M_n$ is idempotent.
Let $Q_X \in \bbO_d$ denote the matrix guaranteed by Assumption~\ref{assum:ttiRate}.
Adding and subtracting appropriate quantities and applying the triangle inequality, 
\begin{equation*}
\left\| \sin \Theta( \What, W ) \right\|
\le \frac{ \sqrt{d} }{ n r_d^2 }
	\left\| \Xhat Q_X Q_X^\top \Xhat^\top \!-\! X X^\top \right\|
\le \frac{ \sqrt{d} }{ n r_d^2 }
\left[ 2\| \Xhat Q_X \! -\! X \| \|X\| 
	+ \left\| \Xhat Q_X \!-\! X \right\|^2 \right].
\end{equation*}
Applying the trivial upper bound
\begin{equation*}
\| \Xhat Q_X  - X \| \le \sqrt{n} \left\| \Xhat Q_X - X \right\|_{\tti}
\end{equation*}
and applying Lemma~\ref{lem:svals:X},
\begin{equation*}
\left\| \sin \Theta( \What, W ) \right\|
\le \frac{ \sqrt{d} }{ r_d^2 }
\left[ 2\left\| \Xhat Q_X  - X \right\|_{\tti}
+ \left\| \Xhat Q_X - X \right\|^2_{\tti} \right].
\end{equation*}
Using Assumption~\ref{assum:ttiRate}, it follows that 
\begin{equation*}
\left\| \sin \Theta( \What, W ) \right\|
\le \frac{ \sqrt{d} }{ r_d^2 }
	\left( 2\ttiRate_n + \ttiRate_n^2 \right).
\end{equation*}
Plugging this back into Equation~\eqref{eq:CCA:XtoZ:pluginHere}, for suitably-chosen constant $C>0$,
\begin{equation*} 
\left\| Q_W^\top \CCA( \Xhat, Z ) - \CCA( X, Z ) \right\|
\le \frac{ C }{ r_d^2 }
	\left( \ttiRate_n + \ttiRate_n^2 \right),
\end{equation*}
where we have used our assumption that $d$ is constant.
The result now follows from the fact that, since $\Sigma_X$ is invertible by assumption,
\begin{equation*}
r_d^2 = \lambda_d\left( \frac{1}{n} X^\top M_n X \right) 
\rightarrow \lambda_d\left( \Sigma_X \right) > 0,
\end{equation*}
in combination with our assumption that $\ttiRate_n = o(1)$.
\end{proof}

\section{Proof of Theorem~\ref{thm:CCA:fullnet}} \label{apx:cca:fullnet}

In light of the definitions of 
$\rho_{X,Z}$ and $\rho_{A,Z}^{(\regu)}$ 
from Equations
~\eqref{eq:def:rho:XtoZ} and ~\eqref{eq:def:rho:AtoZ}, 
respectively, our proof of Theorem~\ref{thm:CCA:fullnet} requires that we show that the leading singular value of $\CCAregu(A,Z)$ approaches that of $\CCA(X,Z)$.
Directly relating these matrices is a challenge, owing to their dimension mismatch and the fact that $A$ varies randomly about $XX^\top$.
Rather than relating $\CCAregu(A,Z)$ directly to $\CCA(X,Z)$, we will relate it to 
\begin{equation} \label{eq:def:CCA:dagger}
\CCAdagger( XX^\top, Z )
= \Sigmatilde_{XX^\top}^{\dagger/2}
	XX^\top \frac{1}{n} M_n Z \Sigmatilde_Z^{-1/2},
\end{equation}
where for a symmetric, positive semi-definite square matrix $H$ with spectral decomposition $H = UDU^\top$, we write $H^{\dagger/2} = U \tilde{D} U^\top$, where $\tilde{D}$ is diagonal with
\begin{equation*}
\tilde{D}_{ii} = \begin{cases}
		1/\sqrt{D_{ii}} &\mbox{ if } D_{ii}\neq0 \\
		0 &\mbox{ otherwise.} \end{cases}
\end{equation*}
We show in Lemma~\ref{lem:fullnettrue:singval} that the non-zero singular values of $\CCAdagger(XX^\top, Z)$ are precisely those of $\CCA(X,Z)$.
Then, we show that $\CCAregu(A,Z)$ is close to $\CCAdagger(XX^\top, Z)$ in spectral norm to complete the proof of Theorem~\ref{thm:CCA:fullnet}.

\begin{lemma} \label{lem:fullnettrue:singval}
Under Assumption~\ref{assum:invertible}, % $\Sigma_X$ converges in prob to invertible thing and $X \in \R^{n \times d}$ and $Z \in \R^{n \times p}$ are such that $\Sigmatilde_X$ and $\Sigmatilde_Z$ are invertible. 
%Lemma~\ref{lem:svals:X}
the non-zero singular values of $\CCAdagger( XX^\top, Z )$, as defined in Equation~\eqref{eq:def:CCA:dagger}, are the same as those of
\begin{equation*}
\CCA( X, Z ) = \Sigmatilde_{X}^{-1/2} 
	X^\top \frac{1}{n} M_n Z \Sigmatilde_Z^{-1/2}.
\end{equation*}
\end{lemma}
\begin{proof}
Recalling the definition of $\Sigmatilde_{XX^\top}$ from Equation~\eqref{eq:def:SigmatildeP} and the definition of $\Sigmatilde_X$ from Equation~\eqref{eq:def:SigmatildeMarginal},
\begin{equation} \label{eq:SigmatildeP:rewrite}
\Sigmatilde_{XX^\top}
= \frac{1}{n} XX^\top M_n X X^\top
= X \Sigmatilde_X X^\top .
\end{equation}
Applying this identity,
\begin{equation*}
\CCAdagger( XX^\top, Z )
= \left( X \Sigmatilde_X X^\top \right)^{\dagger/2}
	\frac{1}{n} XX^\top M_n Z \Sigmatilde_Z^{-1/2}
= \left( X \Sigmatilde_X X^\top \right)^{\dagger/2} X
	\Sigmatilde_{X,Z} \Sigmatilde_Z^{-1/2},
\end{equation*}
where $\Sigmatilde_{X,Z}$ is as defined in Equation~\eqref{eq:def:SigmatildeXZ}.
By the law of large numbers and our assumption that $\Sigma_X$ is invertible, we may write
\begin{equation} \label{eq:truefullnet:trueXZ}
\begin{aligned}
\CCAdagger( XX^\top, Z )
&= \left( X \Sigmatilde_X X^\top \right)^{\dagger/2} X \Sigmatilde_X^{1/2}
	\Sigmatilde_X^{-1/2} 
        \Sigmatilde_{X,Z} \Sigmatilde_Z^{-1/2} \\
&= \left( X \Sigmatilde_X X^\top \right)^{\dagger/2} X \Sigmatilde_X^{1/2}
	\CCA(X, Z).	
\end{aligned} \end{equation}

Applying the singular value decomposition, write
\begin{equation*} 
X \Sigmatilde_X^{1/2} = W_1 D W_2^\top,
\end{equation*}
where $W_1 \in \R^{n \times d}$ and $W_2 \in \R^{d \times d}$ have orthonormal columns, and $D \in \R^{d \times d}$ is diagonal.
By Lemma~\ref{lem:svals:X}, the leading $d$ singular values of $X$ are bounded away from zero.
Again using the fact that $\Sigmatilde_X$ converges in probability to the inverible $\Sigma_X$, it follows that the diagonal entries of $D$ are strictly positive (for suitably large $n$), and we have
\begin{equation*}
\left( X \Sigmatilde_X X^\top \right)^{\dagger/2} X \Sigmatilde_X^{1/2}
= \left( W_1 D^2 W_1^\top \right)^{\dagger/2} W_1 D W_2^\top
= W_1 W_2^\top .
\end{equation*}
Plugging this into Equation~\eqref{eq:truefullnet:trueXZ},
\begin{equation*}
\CCAdagger( XX^\top, Z ) = W_1 W_2^\top \CCA(X, Z).
\end{equation*}
Noting that $W_1 W_2^\top$ has orthonormal columns, applying the singular value decomposition of $\CCA(X,Z)$ completes the proof.
\end{proof}

In light of Lemma~\ref{lem:fullnettrue:singval}, our proof of Theorem~\ref{thm:CCA:fullnet} will rely on relating the spectrum of $\CCAregu(A,Z)$ to that of $\CCAdagger(XX^\top, Z)$.
Applying the singular value decomposition, write
\begin{equation} \label{eq:Pproj:SVD}
\frac{1}{\sqrt{n}} X X^\top M_n
= \Wtildepara \Rtildepara \Vtildepara^\top ,
\end{equation}
where $\Wtildepara,\Vtildepara \in \R^{n \times d}$ both have orthonormal columns, and
\begin{equation} \label{eq:def:Rtildepara}
\Rtildepara = \diag(\rtilde_1,\rtilde_2,\dots,\rtilde_d) \in \R^{d \times d} .
\end{equation}

We also define the following sample analogue of the SVD in Equation~\eqref{eq:Pproj:SVD},
\begin{equation} \label{eq:Aproj:decomp}
\frac{1}{\sqrt{n}} A M_n
= \What \Rhat \Vhat^\top
= \Whatpara \Rhatpara \Vhatpara^\top + \Whatperp \Rhatperp \Vhatperp^\top,
\end{equation}
where
\begin{equation} \label{eq:def:Rhat}
\begin{aligned}
\Rhatpara &= \diag( \rhat_1, \rhat_2,\dots,\rhat_d)
\in \R^{d \times d} \\
\Rhatperp &= \diag( \rhat_{d+1}, \rhat_{d+2},\dots, \rhat_n )
\in \R^{(n-d) \times (n-d)},
\end{aligned} \end{equation}
$\Whatpara, \Vhatpara \in R^{n \times d}$ and $\Whatperp, \Vhatperp \in \R^{n \times (n-d)}$ all have orthonormal columns, the columns of $\Vhatpara$ are orthogonal to those of $\Vhatperp$, and similarly for $\Whatpara$ and $\Whatperp$.

Our next two technical lemmas establish the asymptotic behavior of $\Rtildepara$ and $\Rhatpara$.

\begin{lemma} \label{lem:fullnet:rtildegrowth}
Under Assumption~\ref{assum:invertible}, %for Lemma~\ref{lem:svals:X} and~\ref{lem:svalgrowth:SigmatildeXX}
with $\rtilde_1,\rtilde_2,\dots,\rtilde_d$ as defined in Equation~\eqref{eq:def:Rtildepara}, it holds for all $k \in [d]$ that
\begin{equation} \label{eq:rtilde:growth}
\rtilde_k = \Thetap{ \sqrt{n} }.
\end{equation}
\end{lemma}
\begin{proof}
From Equation~\eqref{eq:Pproj:SVD}, we observe that the diagonal entries of $\Rtildepara^2$ are precisely the non-zero eigenvalues of
\begin{equation*}
\Sigmatilde_{XX^\top}
=
\frac{ XX^\top M_n }{\sqrt{n}}
\left[ \frac{ XX^\top M_n }{\sqrt{n}} \right]^\top ,
\end{equation*}
where we have used the fact that $M_n$ is idempotent.
Applying Lemma~\ref{lem:svalgrowth:SigmatildeXX} to control the spectrum of this matrix and taking square roots, it follows that $\rtilde_d = \Omegap{ \sqrt{n} }$, and thus $\rtilde_k = \Omegap{ \sqrt{n} }$ for any $k \in [d]$.
On the other hand, using the fact that $M_n$ is a projection, for any $k \in [d]$,
\begin{equation*}
\rtilde_d \le \rtilde_k
\le \left\| XX^\top \frac{1}{\sqrt{n}} M_n \right\|
\le \frac{1}{\sqrt{n}} \left\| X \right\|^2
= \Op{ \sqrt{n} },
\end{equation*}
where the last bound follows from Lemma~\ref{lem:svals:X}, and this yields Equation~\eqref{eq:rtilde:growth}, completing the proof.
\end{proof}

\begin{lemma} \label{lem:Rhat:spectrum}
Under Assumption~\ref{assum:invertible}, %for Lemma lem:fullnet:rtildegrowth
suppose that Assumption~\ref{assum:spectralRate} holds with $\spectralRate_n$ obeying the growth rate in Equation~\eqref{eq:assum:fullnet:spectralRate:growth}.
Then, with $\rhat_1,\rhat_2,\dots,\rhat_n$ as defined in Equation~\eqref{eq:def:Rhat},
\begin{equation} \label{eq:Rhat:spectrum:leading}
\max_{i \in [d]} \rhat_i^2 = \Thetap{ n }
\end{equation}
and
\begin{equation} \label{eq:Rhat:spectrum:trailing}
\max_{i \in \{d+1,d+2,\dots,n\} } \rhat_i^2
= \Op{ \frac{ \spectralRate_n^2 }{ n } } .
\end{equation}
\end{lemma}
\begin{proof}
Applying submultiplicativity of the norm and the fact that $M_n$ is a projection,
\begin{equation*}
\left\|  \frac{1}{\sqrt{n}} A M_n - \frac{1}{\sqrt{n}} XX^\top M_n \right\|
\le \frac{1}{\sqrt{n}} \left\| A - X X^\top \right\|.
\end{equation*}
Applying Assumption~\ref{assum:spectralRate},
\begin{equation} \label{eq:SVDdiff}
\left\|  \frac{1}{\sqrt{n}} A M_n - \frac{1}{\sqrt{n}} XX^\top M_n \right\|
= \Op{ \frac{\spectralRate_n}{\sqrt{n}} }.
\end{equation}
It follows that for any $i \in [n]$, writing $\rtilde_i$ for the singular values of $XX^\top M_n/\sqrt{n}$, we have
\begin{equation*}
\rhat_i^2 = \left[ \rtilde_i + (\rhat_i - \rtilde_i) \right]^2
= \rtilde_i^2 + \Op{ \frac{\spectralRate_n \rtilde_i}{\sqrt{n}} } 
	+ \Op{ \frac{ \spectralRate_n^2 }{ n } }.
\end{equation*}
Applying Lemma~\ref{lem:fullnet:rtildegrowth}, our growth assumption in Equation~\eqref{eq:assum:fullnet:spectralRate:growth} yields Equation~\eqref{eq:Rhat:spectrum:leading}.

On the other hand, since $\rtilde_i=0$ for $i > d$, again using Equation~\eqref{eq:SVDdiff}, we have
\begin{equation} \label{eq:rhatctrl:cases}
\rhat_i^2 = \left[ \rtilde_i + (\rhat_i - \rtilde_i) \right]^2
= \Op{ \frac{ \spectralRate_n^2 }{ n } },
\end{equation}
yielding Equation~\eqref{eq:Rhat:spectrum:trailing}, completing the proof.
\end{proof}

% https://projecteuclid.org/journals/annals-of-statistics/volume-45/issue-1/Rate-optimal-perturbation-bounds-for-singular-subspaces-with-applications-to/10.1214/17-AOS1541.full
The following result is a specialization of Theorem 1 in \cite{CaiZha2018} to the present setting.
We state and prove it here in our notation for the sake of completeness.

\begin{lemma} \label{lem:CaiZha2018}
Under Assumption~\ref{assum:invertible}, %for lem:svalgrowth:SigmatildeXX
suppose that
Assumption~\ref{assum:spectralRate} holds with $\spectralRate_n$ obeying the growth rate assumption in Equation~\eqref{eq:assum:fullnet:spectralRate:growth}
and that
Assumption~\ref{assum:subspaceRate} holds with $\subspaceRate_n$ obeying the growth rate assumption in Equation~\eqref{eq:assum:fullnet:subspaceRate:growth}.
Then, with $\Whatpara$ as given by Equation~\eqref{eq:Aproj:decomp} and $\Wtildepara$ as given by Equation~\eqref{eq:Pproj:SVD},
\begin{equation} \label{eq:sinTheta:W}
\left\| \sin \Theta( \Whatpara, \Wtildepara ) \right\|
= \Op{ \frac{ \spectralRate_n }{ n } } .
\end{equation}
Further, there exists $Q_W \in \bbO_d$ such that
\begin{equation} \label{eq:Qclose}
\left\| \Wtildepara^\top \Whatpara - Q_W \right\|
= \Op{ \frac{\spectralRate_n^2}{n^2} }
\end{equation} 
and
\begin{equation} \label{eq:WhatWtilde:close}
\left\| \Whatpara Q_W - \Wtildepara \right\|
= \Op{ \frac{\spectralRate_n}{n} } .
\end{equation}
\end{lemma}
\begin{proof}
Let $\Wtildeperp \in \R^{n \times (n-d)}$ be a matrix whose columns span the orthogonal complement of the column space of $\Wtildepara$, and let $\Vtildeperp \in \R^{n \times (n-d)}$ be a matrix whose columns span the orthogonal column space of $\Vtildepara$.
Define $\alpha,\beta \in \R$ according to
\begin{equation} \label{eq:def:alphabeta} \begin{aligned}
\alpha &= \sigma_d\left( \Wtildepara^\top A
                \frac{1}{\sqrt{n}} M_n
                \Vtildepara \right)
~\text{ and }~
\beta = \left\| \Wtildeperp^\top A \frac{1}{\sqrt{n}} M_n
                \Vtildeperp \right\|
\end{aligned} \end{equation}
and define
\begin{equation*} \begin{aligned}
z_{12} &= \left\| \frac{1}{\sqrt{n}} \Wtildepara \Wtildepara^\top
                (A-XX^\top) M_n \Vtildeperp \Vtildeperp^\top \right\|
~\text{ and }\\
z_{21} &= \left\| \frac{1}{\sqrt{n}} \Wtildeperp \Wtildeperp^\top
                (A-XX^\top) M_n \Vtildepara \Vtildepara^\top \right\|,
\end{aligned} \end{equation*}
where $\Vtildepara$ and $\Wtildepara$ are as in Equation~\eqref{eq:Pproj:SVD}.
By Theorem 1 in \cite{CaiZha2018}, we have
\begin{equation} \label{eq:CaiZha:Uterms}
\left\| \sin \Theta( \Whatpara, \Wtildepara ) \right\|
\le \frac{ \alpha z_{21} + \beta z_{12} }
        { \alpha^2 - \beta^2 - z_{21}^2 \vee z_{12}^2 } .
\end{equation}

By Assumption~\ref{assum:subspaceRate},
\begin{equation*}
\left\| \frac{1}{\sqrt{n}}
	\Wtildepara^\top \left( A - XX^\top \right) M_n \Vtildepara \right\|
= \Op{ \frac{\subspaceRate_n}{\sqrt{n}} } .
\end{equation*}
Using basic properties of singular values, it follows that
\begin{equation*}
\alpha = \sigma_d\left( \frac{1}{\sqrt{n}}
                \Wtildepara^\top XX^\top M_n \Vtildepara \right)
	+ \Op{ \frac{ \subspaceRate_n }{ \sqrt{n} } } .
\end{equation*}
Since $\Wtildepara^\top XX^\top M_n \Vtildepara^\top/\sqrt{n} = \Rtildepara$, Lemma~\ref{lem:fullnet:rtildegrowth} implies
\begin{equation*}
\alpha = \Thetap{ \sqrt{n} } + \Op{ \frac{ \subspaceRate_n }{ \sqrt{n} } } .
\end{equation*}
Our growth assumption in Equation~\eqref{eq:assum:fullnet:subspaceRate:growth} % $\subspaceRate_n = o(n)$
implies 
\begin{equation} \label{eq:ctrl:alpha}
\alpha = \Thetap{ \sqrt{n} } .
\end{equation}

Applying submultiplicativity followed by Assumption~\ref{assum:spectralRate},
\begin{equation} \label{eq:z:ctrl:close}
z_{12} \vee z_{21}
\le \frac{1}{\sqrt{n}} \left\| A - XX^\top \right\|
= \Op{ \frac{ \spectralRate_n }{ \sqrt{ n } } } .
\end{equation}
Recalling Equation~\eqref{eq:ctrl:alpha} and applying our growth assumption in Equation~\eqref{eq:assum:fullnet:spectralRate:growth},
\begin{equation} \label{eq:z:ctrl}
z_{12} \vee z_{21} = \op{ \alpha } .
\end{equation}

For ease of notation, write $E = A - XX^\top$.
Recalling the definition of $\beta$ from Equation~\eqref{eq:def:alphabeta} and using basic properties of the norm,
\begin{equation*}
\beta \le
\left\| \Wtildeperp^\top X X^\top \frac{1}{\sqrt{n}} M_n \Vtildeperp \right\|
+ \left\| \Wtildeperp^\top E \frac{1}{\sqrt{n}} M_n \Vtildeperp \right\|.
\end{equation*}
Observing that $\Vtildeperp$ and $\Wtildeperp$ correspond precisely to the null singular values,
\begin{equation*}
\beta \le \left\| \Wtildeperp^\top E
         \frac{1}{\sqrt{n}} M_n \Vtildeperp \right\|.
\end{equation*}
Applying submultiplicativity and Assumption~\ref{assum:spectralRate},
\begin{equation} \label{eq:ctrl:beta:close}
\beta \le \frac{1}{\sqrt{n}} \left\| E \right\|
= \Op{ \frac{ \spectralRate_n }{ \sqrt{n} } } .
\end{equation}
Comparing this with Equation~\eqref{eq:ctrl:alpha}, our growth assumption in Equation~\eqref{eq:assum:fullnet:spectralRate:growth} implies that 
\begin{equation} \label{eq:ctrl:beta}
\beta = \op{ \alpha } .
\end{equation}

By Equations~\eqref{eq:z:ctrl} and~\eqref{eq:ctrl:beta},
\begin{equation*}
\alpha^2 - \beta^2 - z_{12}^2 \vee z_{21}^2
= \alpha^2 \left( 1 - \op{1} \right)
= \Omegap{ \alpha^2 } .
\end{equation*}
Applying this to Equation~\eqref{eq:CaiZha:Uterms}, 
\begin{equation*}
\left\| \sin \Theta( \Whatpara, \Wtildepara ) \right\|
=
\Op{ \frac{ z_{21} }{ \alpha } }
+ \Op{ \frac{ \beta z_{12} }{ \alpha^2 } } .
\end{equation*}
Applying Equations~\eqref{eq:z:ctrl:close},~\eqref{eq:ctrl:beta:close} and~\eqref{eq:ctrl:alpha}, it follows that
\begin{equation*}
\left\| \sin \Theta( \Whatpara, \Wtildepara ) \right\|
= \Op{ \frac{ \spectralRate_n }{ n } } 
+ \Op{ \frac{ \spectralRate_n^2 }{ n^2 } }.
\end{equation*}
Our assumption in Equation~\eqref{eq:assum:fullnet:spectralRate:growth} then yields Equation~\eqref{eq:sinTheta:W}.

As elsewhere in the literature \citep[see discussion in][]{CapTanPri2019}, we take
\begin{equation*}
Q_W = U_1 U_2^\top,
\end{equation*}
where $U_1,U_2 \in \IR^{n \times d}$ are the singular subspaces in the SVD $\Wtildepara^\top \Whatpara = U_1 D U_2^\top$. 
To establish Equation~\eqref{eq:Qclose}, standard arguments from the subspace geometry literature \citep[see, e.g.,][]{CaiZha2018,CapTanPri2019} imply
\begin{equation*} 
\left\| \Wtildepara^\top \Whatpara - Q_W \right\|
\le \left\| \sin  \Theta( \Whatpara, \Wtildepara ) \right\|^2
= \Op{ \frac{ \spectralRate_n^2 }{ n^2 } } ,
\end{equation*}
where the equality follows from Equation~\eqref{eq:sinTheta:W}.

By the triangle inequality and submultiplicativity of the norm, again using basic results on subspace perturbation,
\begin{equation*} \begin{aligned}
\left\| \Whatpara - \Wtildepara Q_W \right\|
&\le \left\| \Whatpara - \Wtildepara \Wtildepara^\top \Whatpara \right\|
	+ \left\| \Wtildepara^\top \Whatpara - Q_W \right\| \\
&= \left\| \sin \Theta( \Whatpara, \Wtildepara ) \right\|
+ \left\| \sin \Theta( \Whatpara, \Wtildepara ) \right\|^2,
\end{aligned} \end{equation*}
and Equation~\eqref{eq:sinTheta:W} along with our assumption in Equation~\eqref{eq:assum:fullnet:spectralRate:growth} implies Equation~\eqref{eq:WhatWtilde:close}.
\end{proof}

\begin{lemma} \label{lem:Rinterchange}
Under Assumption~\ref{assum:invertible}, %for lem:fullnet:rtildegrowth, lem:Rhat:spectru
suppose 
that Assumption~\ref{assum:spectralRate} holds with $\spectralRate_n$ obeying Equation~\eqref{eq:assum:fullnet:spectralRate:growth} %for lem:Rhat:spectrum
and
that Assumption~\ref{assum:subspaceRate} holds with $\subspaceRate_n$ obeying Equation~\eqref{eq:assum:fullnet:subspaceRate:growth}. %for lem:CaiZha2018
Then
\begin{equation*}
\left\| Q_W \Rhatpara^2 - \Rtildepara^2 Q_W \right\|
= 
\Op{ \frac{ \spectralRate_n^2 }{ n } } 
+ \Op{ \frac{ \subspaceRate_n }{ \sqrt{n} } } ,
\end{equation*}
where $Q_W \in \bbO_d$ is the orthogonal matrix guaranteed by Lemma~\ref{lem:CaiZha2018}.
\end{lemma}
\begin{proof}
We follow a standard argument developed previously in \cite{LTAPP17} \citep[see also][]{LevAthTanLyzYouPri2017}.
Recalling the SVDs from Equations~\eqref{eq:Pproj:SVD} and~\eqref{eq:Aproj:decomp} and adding and subtracting appropriate quantities, writing $P = XX^\top$ for ease of notation,
\begin{equation} \label{eq:interchange:bigexpand} \begin{aligned}
Q_W \Rhatpara^2 - \Rtildepara^2 Q_W
&= \left( Q_W \!-\! \Wtildepara^\top \Whatpara \right) \Rhatpara^2
+ \Wtildepara^\top \!\left( \frac{A M_n A}{n} - \frac{P M_n P}{n} \right)
	\! \left( \Whatpara - \Wtildepara \Wtildepara^\top \Whatpara \right) \\
&~~~~~~+ \Wtildepara^\top \left( \frac{A M_n A}{n} - \frac{P M_n P}{n} \right)
	\Wtildepara \Wtildepara^\top \Whatpara
+ \Rtildepara^2 \left( \Wtildepara^\top \Whatpara - Q_W \right) .
\end{aligned} \end{equation} 

By submultiplicativity, Lemma~\ref{lem:Rhat:spectrum} and Equation~\eqref{eq:Qclose} in Lemma~\ref{lem:CaiZha2018},
\begin{equation} \label{eq:interchange:term1}
\left\| \left( Q_W - \Wtildepara^\top \Whatpara \right) \Rhatpara^2 \right\|
\le
\left\| Q_W - \Wtildepara^\top \Whatpara \right\|
\left\| \Rhatpara^2 \right\|
= \Op{ \frac{ \spectralRate_n^2 }{ n } }
\end{equation}
and similarly, this time using Lemma~\ref{lem:fullnet:rtildegrowth},
\begin{equation} \label{eq:interchange:term4}
\left\| \Rtildepara^2 \left( \Wtildepara^\top \Whatpara - Q_W \right)
	\right\|
= \Op{ \frac{ \spectralRate_n^2 }{ n } } .
\end{equation}

Again applying basic properties of subspace geometry \citep[see, e.g., Lemma 6.7 in][]{CapTanPri2019} and using Lemma~\ref{lem:CaiZha2018},
\begin{equation} \label{eq:subspace:projection}
\left\| \Whatpara - \Wtildepara \Wtildepara^\top \Whatpara \right\|
= \left\| \sin  \Theta( \Whatpara, \Wtildepara ) \right\|
= \Op{ \frac{ \spectralRate_n }{ n } } .
\end{equation}
By the triangle inequality,
\begin{equation*} \begin{aligned}
\left\| \frac{A M_n A}{n} - \frac{P M_n P}{n} \right\|
&\le 2\left\| \frac{ (A-P) M_n P }{ n } \right\|
	+ \left\| \frac{ (A-P) M_n (A-P) }{ n } \right\|  \\
&\le \frac{ 2 \left\| A-P \right\| }{ \sqrt{n} }
	\left\| \frac{M_n P}{ \sqrt{n} } \right\|
	+ \frac{ \left\| A-P \right\|^2 }{ n } .
\end{aligned} \end{equation*}
Applying Assumption~\ref{assum:spectralRate}, the definition of $\Rtildepara$ in Equation~\eqref{eq:def:Rtildepara}, and Lemma~\ref{lem:fullnet:rtildegrowth},
\begin{equation} \label{eq:AMA:PMP:spectral}
\left\| \frac{A M_n A}{n} - \frac{P M_n P}{n} \right\|
= \Op{ \spectralRate_n } + \Op{ \frac{ \spectralRate^2 }{ n } }
= \Op{ \spectralRate_n },
\end{equation}
where the second equality follows from our growth assumption in Equation~\eqref{eq:assum:fullnet:spectralRate:growth}.

Using submultiplicativity followed by Equations~\eqref{eq:subspace:projection} and~\eqref{eq:AMA:PMP:spectral},
\begin{equation} \label{eq:interchange:term2}
\begin{aligned}
\left\|\Wtildepara^\top \!\!
	\left( \! \frac{A M_n A}{n} \!-\! \frac{P M_n P}{n} \! \right) \!\!
	\left( \! \Whatpara \!-\! \Wtildepara \Wtildepara^\top \Whatpara 
	\!\right) \right\|
&\le
\left\| \frac{A M_n A}{n} \!-\! \frac{P M_n P}{n} \right\|
\left\| \Whatpara \!-\! \Wtildepara \Wtildepara^\top \Whatpara \right\| \\
&= \Op{ \frac{ \spectralRate_n^2 }{ n } } 
\end{aligned} \end{equation}

Similarly, by the triangle inequality and submultiplicativity and Assumption~\ref{assum:spectralRate},
\begin{equation} \label{eq:interchange:term3:prelim} \begin{aligned}
\left\| \Wtildepara^\top \left( \frac{A M_n A}{n} - \frac{P M_n P}{n} \right)
	\Wtildepara \right\|
%&\le
%2 \left\| \frac{1}{n} \Wtildepara^\top (A-P) M_n P \Wtildepara \right\|
%	+ 
%\left\| \frac{1}{n} \Wtildepara^\top (A-P) M_n (A-P) \Wtildepara \right\| \\
&\le 2 \left\| \frac{1}{n} \Wtildepara^\top (A-P) M_n P \Wtildepara \right\|
        + \Op{ \frac{ \spectralRate_n^2 }{ n } } .
\end{aligned} \end{equation}
Using the fact that $M_n$ and $P$ are symmetric and recalling the SVD from Equation~\eqref{eq:Pproj:SVD},
\begin{equation*}
 \frac{1}{n} \Wtildepara^\top (A-P) M_n P \Wtildepara
= \frac{1}{\sqrt{n}} \Wtildepara^\top (A-P) \Vtildepara \Rtildepara ,
\end{equation*}
and Assumption~\ref{assum:subspaceRate} implies
\begin{equation*}
\left\| \frac{1}{n} \Wtildepara^\top (A-P) M_n P \Wtildepara \right\|
= \Op{ \frac{ \subspaceRate_n }{ \sqrt{n} } } .
\end{equation*}
Applying this to Equation~\eqref{eq:interchange:term3:prelim},
\begin{equation} \label{eq:interchange:term3} \begin{aligned}
\left\| \Wtildepara^\top \left( \frac{A M_n A}{n} - \frac{P M_n P}{n} \right)
	\Wtildepara \right\|
&\le \Op{ \frac{ \subspaceRate_n }{ \sqrt{n} } } 
        + \Op{ \frac{ \spectralRate_n^2 }{ n } } .
\end{aligned} \end{equation}

Applying the triangle inequality in Equation~\eqref{eq:interchange:bigexpand} followed by Equations
~\eqref{eq:interchange:term1},
~\eqref{eq:interchange:term4},
~\eqref{eq:interchange:term2},
and~\eqref{eq:interchange:term3},
\begin{equation*}
\left\| \Rhatpara^2 Q_W - \Rtildepara^2 Q_W \right\|
\le
\Op{ \frac{ \spectralRate_n^2 }{ n } } 
+ \Op{ \frac{ \subspaceRate_n }{ \sqrt{n} } } ,
\end{equation*}
which completes the proof.
\end{proof}

\begin{lemma} \label{lem:Rinterchange:sqrt}
Under Assumption~\ref{assum:invertible}, %for lem:fullnet:rtildegrowth, lem:Rhat:spectru
suppose 
that Assumption~\ref{assum:spectralRate} holds with $\spectralRate_n$ obeying Equation~\eqref{eq:assum:fullnet:spectralRate:growth} %for lem:Rhat:spectrum
and
that Assumption~\ref{assum:subspaceRate} holds with $\subspaceRate_n$ obeying Equation~\eqref{eq:assum:fullnet:subspaceRate:growth}. %for lem:CaiZha2018
Then
\begin{equation*}
\left\| \Rhatpara^{-1} Q_W - Q_W \Rtildepara^{-1} \right\|_F
= \Op{ \frac{ \spectralRate_n^2 }{ n^{5/2} } 
	+ \frac{ \subspaceRate_n }{ n^2 } } .
\end{equation*}
\end{lemma}
\begin{proof}
Observe that for $k,\ell \in [d]$,
\begin{equation*} \begin{aligned}
\left| \Rhatpara^{-1} Q_W - Q_W \Rtildepara^{-1} \right|_{k,\ell}
&= 
\left| [Q_W]_{k,\ell} \left( \frac{1}{\rhat_k} - \frac{1}{\rtilde_\ell} \right)
	\right|
= \left| [Q_W]_{k,\ell} \frac{ \rtilde_\ell - \rhat_k }{ \rhat_k \rtilde_\ell}
	\right| \\
&= \left| [Q_W]_{k,\ell} \frac{ \rtilde_\ell^2 - \rhat_k^2 }
	{ \rhat_k \rtilde_\ell (\rtilde_\ell + \rhat_k) } \right| .
\end{aligned} \end{equation*}
Applying Lemmas~\ref{lem:fullnet:rtildegrowth} and~\ref{lem:Rhat:spectrum},
\begin{equation*}
\left| \Rhatpara^{-1} Q_W - Q_W \Rtildepara^{-1} \right|_{k,\ell}
\le \frac{ C }{ n^{3/2} }
\left| [Q_W]_{k,\ell} \left( \rtilde_\ell^2 - \rhat_k^2 \right) \right|,
\end{equation*}
Squaring and summing over all $k,\ell \in [d]$,
\begin{equation*}
\left\|  \Rhatpara^{-1} Q_W - Q_W \Rtildepara^{-1} \right\|_F^2
\le
\frac{ C }{n^3} \left\| \Rhatpara^2 Q_W - Q_W \Rtildepara^2 \right\|_F^2 .
\end{equation*}
Upper-bounding
\begin{equation*}
\left\| \Rhatpara^2 Q_W - Q_W \Rtildepara^2 \right\|_F^2
\le d \left\| \Rhatpara^2 Q_W - Q_W \Rtildepara^2 \right\|^2,
\end{equation*}
applying Lemma~\ref{lem:Rinterchange} and taking square roots completes the proof.
\end{proof}

\begin{lemma} \label{lem:paradiff}
Under Assumption~\ref{assum:invertible}, %for lem:fullnet:rtildegrowth, lem:Rhat:spectru
suppose 
that Assumption~\ref{assum:spectralRate} holds with $\spectralRate_n$ obeying Equation~\eqref{eq:assum:fullnet:spectralRate:growth} %for lem:Rhat:spectrum
and
that Assumption~\ref{assum:subspaceRate} holds with $\subspaceRate_n$ obeying Equation~\eqref{eq:assum:fullnet:subspaceRate:growth}. %for lem:CaiZha2018
Then
\begin{equation*}
\left\| \Whatpara \left( \Rhatpara^2 + \regu I \right)^{-1/2} \Whatpara^\top 
	- \Wtildepara \Rtildepara^{-1} \Wtildepara^\top \right\|
= 
\Op{ \frac{|\regu|}{n^{3/2}}
+ \frac{ \spectralRate_n }{ n^{3/2} } + \frac{ \subspaceRate_n }{ n^2 } } .
\end{equation*}
\end{lemma}
\begin{proof}
By the triangle inequality and submultiplicativity of the norm,
\begin{equation} \label{eq:paradiff:split} \begin{aligned}
\left\| \Whatpara \left( \Rhatpara^2 + \regu I \right)^{-1/2} \Whatpara^\top
	- \Wtildepara \Rtildepara^{-1} \Wtildepara^\top \right\|
&\le \left\| \left( \Rhatpara^2 + \regu I \right)^{-1/2} - \Rhatpara^{-1}
	\right\| \\
&~~~~~~+ \left\| \Whatpara \Rhatpara^{-1} \Whatpara^\top
	- \Wtildepara \Rtildepara^{-1} \Wtildepara^\top \right\| .
\end{aligned} \end{equation}

By definition,
\begin{equation*} \begin{aligned}
\left\| \left( \Rhatpara^2 + \regu I \right)^{-1/2} - \Rhatpara^{-1} \right\|
&= \max_{k \in [d]} \left| \frac{ 1 }{ \sqrt{ \rhat_k^2 + \regu} }
			- \frac{1}{ \rhat_k } \right| 
= \max_{k \in [d]}
\left| \frac{ \rhat_k - \sqrt{ \rhat_k^2 + \regu} }
		{ \rhat_k \sqrt{ \rhat_k^2 + \regu} } \right| .
\end{aligned} \end{equation*}
Multiplying through by appropriate quantities,
\begin{equation*}
\left\| \left( \Rhatpara^2 + \regu I \right)^{-1/2} - \Rhatpara^{-1} \right\|
= \max_{k \in [d]}
\frac{ | \regu | }
	{ \left( \sqrt{ \rhat_k^2 + \regu} + \rhat_k \right) 
		\rhat_k \sqrt{ \rhat_k^2 + \regu} } .
\end{equation*}
Applying Lemma~\ref{lem:Rhat:spectrum} and the fact that $\regu \ge 0$,
\begin{equation} \label{eq:Rhatpara:removeregu}
\left\| \left( \Rhatpara^2 + \regu I \right)^{-1/2} - \Rhatpara^{-1} \right\|
= \Op{ \frac{ | \regu | }{ n^{3/2} } } .
\end{equation}

Adding and subtracting appropriate quantities,
\begin{equation*} \begin{aligned}
\Whatpara \Rhatpara^{-1} \Whatpara^\top
	- \Wtildepara \Rtildepara^{-1} \Wtildepara^\top 
&=
\left( \Whatpara Q_W - \Wtildepara \right)
Q_W^\top \Rhatpara^{-1} Q_W \left(\Whatpara Q_W -\Wtildepara \right)^\top \\
&~~~+ \left( \Whatpara Q_W - \Wtildepara \right)
	Q_W^\top \Rhatpara^{-1} Q_W \Wtildepara^\top \\
&~~~+ 
\Wtildepara Q_W^\top \Rhatpara^{-1} Q_W \!\!
	\left( \Whatpara Q_W \! - \! \Wtildepara \! \right)^{\!\!\top}
\!+\! \Wtildepara \! 
	\left( Q_W^\top \Rhatpara^{-1} Q_W  \!-\! \Rtildepara^{-1} \!\right) 
	\! \Wtildepara^\top .
\end{aligned} \end{equation*}
Applying the triangle inequality and submultiplicativity of the norm,
\begin{equation*} \begin{aligned}
\left\| \Whatpara \Rhatpara^{-1} \Whatpara^\top
	- \Wtildepara \Rtildepara^{-1} \Wtildepara^\top \right\|
&\le 
\left\| \Whatpara Q_W - \Wtildepara \right\|^2 
	\left\| \Rhatpara^{-1} \right\| \\
&~~~+ 2 \left\| \Whatpara Q_W - \Wtildepara \right\|
	\left\| \Rhatpara^{-1} \right\| \\
&~~~+ \left\| Q_W^\top \Rhatpara^{-1} Q_W  \!-\! \Rtildepara^{-1} \right\| .
\end{aligned} \end{equation*}
Applying Lemmas
~\ref{lem:fullnet:rtildegrowth},
~\ref{lem:CaiZha2018}
and~\ref{lem:Rinterchange:sqrt},
\begin{equation*}
\left\| \Whatpara \Rhatpara^{-1} \Whatpara^\top
	- \Wtildepara \Rtildepara^{-1} \Wtildepara^\top \right\|
\le
\Op{ \frac{ \spectralRate_n^2 }{ n^{5/2} } }
+
\Op{ \frac{ \spectralRate_n }{ n^{3/2} } }
%+ \Op{ \frac{ \spectralRate_n^2 }{ n^{5/2} } }
+ \Op{ \frac{ \subspaceRate_n }{ n^2 } } .
\end{equation*}
Applying our growth assumption in Equation~\eqref{eq:assum:fullnet:spectralRate:growth},
\begin{equation*}
\left\| \Whatpara \Rhatpara^{-1} \Whatpara^\top
	- \Wtildepara \Rtildepara^{-1} \Wtildepara^\top \right\|
\le
\Op{ \frac{ \spectralRate_n }{ n^{3/2} } }
+ \Op{ \frac{ \subspaceRate_n }{ n^2 } } .
\end{equation*}
Applying this and Equation~\eqref{eq:Rhatpara:removeregu} to Equation~\eqref{eq:paradiff:split} and simplifying completes the proof.
\end{proof}

With the above results in hand, we are prepared to prove Theorem~\ref{thm:CCA:fullnet}.
\begin{proof}[Proof of Theorem~\ref{thm:CCA:fullnet}]
By Lemma~\ref{lem:fullnettrue:singval}, the non-zero singular values of $\CCA(X,Z)$ are the same as those of
\begin{equation} \label{eq:def:CCA:PtoZ}
\CCAdagger( XX^\top, Z )
= \Sigmatilde_{XX^\top}^{\dagger/2}
	\frac{1}{n} XX^\top\! M_n ~
	Z ~ \Sigmatilde_Z^{-1/2} \in \R^{n \times p}.
\end{equation}
Thus, recalling the definitions of $\rho_{X,Z}$ and $\rho_{A,Z}^{(\regu)}$ from Equations~\eqref{eq:def:rho:XtoZ} and~\eqref{eq:def:rho:AtoZ}, respectively, we have
\begin{equation} \label{eq:CCA:fullnet:target}
\left|  \rho_{A,Z}^{(\regu)} - \rho_{X,Z} \right|
\le 
\left\| \CCAregu(A,Z) - \CCAdagger( XX^\top, Z ) \right\| .
\end{equation}

Recalling the definition of $\Sigmatilde_{XX^\top}$ from Equation~\eqref{eq:def:SigmatildeP}, we can rewrite Equation~\eqref{eq:def:CCA:PtoZ} as
\begin{equation} \label{eq:CCAdagger:expand}
\CCAdagger( XX^\top, Z )
= \left( \frac{1}{n} XX^\top M_n XX^\top \right)^{\dagger/2}
	\frac{1}{n} XX^\top M_n Z ~\Sigmatilde_Z^{-1/2}.
\end{equation}
Similarly, recalling the definition of $\CCAregu(A,Z)$ from Equation~\eqref{eq:def:CCA:fullnet},
\begin{equation*}
\CCAregu( A, Z )
= \left( \Sigmatilde_A + \regu I \right)^{-1/2}
        \frac{1}{n} A M_n Z \Sigmatilde_Z^{-1/2} 
= \left( \frac{1}{n} A M_n A + \regu I \right)^{-1/2}
	\frac{1}{n} A M_n Z \Sigmatilde_Z^{-1/2} .
\end{equation*}
Adding and subtracting appropriate quantities,
Combining the above two displays and writing $P = XX^\top$ for ease of notation,
\begin{equation} \label{eq:CCAnet:addsub} \begin{aligned}
\CCAregu( \!A, Z ) \!-\! \CCAdagger( XX^\top\!\!, Z ) 
&=
\left( \!\frac{ A M_n A }{n} + \regu I \!\right)^{\!\!-1/2}
	\frac{ \left(A \!-\! P\right) \! M_n Z }{ n } \Sigmatilde_Z^{-1/2} \\
&~~~+
\left[ \! \left( \!\frac{ A M_n A }{n} + \regu I \!\right)^{\!\!-1/2}
		\!\!\!\!-\! 
		\left( \!\frac{P M_n P }{n} \right)^{\!\!-1/2} \right]
	\!\! 	\frac{ PM_n Z }{n} \Sigmatilde_Z^{-1/2} .
\end{aligned} \end{equation}

We note that by construction,
\begin{equation} \label{eq:MZSig:construction}
\left\| \frac{M_n Z}{\sqrt{n}} \Sigmatilde_Z^{-1/2} \right\| = 1 .
\end{equation}
By submultiplicativity of the norm, Assumption~\ref{assum:spectralRate}, and the above display,
\begin{equation} \label{eq:AMAreguAP:submul}
\left\| \left( \!\frac{ A M_n A }{n} + \regu I \!\right)^{\!\!-1/2}
	\frac{ \left(A \!-\! P\right) \! M_n Z }{ n } \Sigmatilde_Z^{-1/2}
\right\|
\le
\left\| \left( \!\frac{ A M_n A }{n} + \regu I \!\right)^{\!\!-1/2} \right\|
\frac{ \spectralRate_n }{ \sqrt{n} } .
\end{equation}
Using the fact that $M_n$ is idempotent and symmetric, and recalling the SVD from Equation~\eqref{eq:Aproj:decomp},
\begin{equation*}
\left\| \left( \!\frac{ A M_n A }{n} + \regu I \!\right)^{\!\!-1/2} \right\|
=
\left\| \left( \!\frac{ A M_n M_n A }{n} 
	+ \regu I \!\right)^{\!\!-1/2} \right\|
= \max_{i \in [n]} \frac{ 1 }{ \sqrt{ \rhat_i^2 + \regu } } .
\end{equation*}
Applying Lemma~\ref{lem:Rhat:spectrum} and using the fact that $\regu > 0$ by assumption,
\begin{equation*}
\left\| \left( \!\frac{ A M_n A }{n} + \regu I \!\right)^{\!\!-1/2} \right\|
= \Op{ \max\left\{ \frac{1}{\sqrt{n}} , \frac{1}{\sqrt{\regu} }
\right\} }
= \Op{ \frac{1}{\sqrt{\regu} } },
\end{equation*}
where the second equality follows from our assumption that $\regu = O(n)$.
Applying this to Equation~\eqref{eq:AMAreguAP:submul},
\begin{equation} \label{eq:AMAreguAP:final}
\left\| \left( \!\frac{ A M_n A }{n} + \regu I \!\right)^{\!\!-1/2}
	\frac{ \left(A \!-\! P\right) \! M_n Z }{ n } \Sigmatilde_Z^{-1/2}
\right\|
= \Op{  \frac{\spectralRate_n}{\sqrt{\regu n} } } .
\end{equation}

Applying the SVDs from Equations~\eqref{eq:Aproj:decomp} and~\eqref{eq:Pproj:SVD} and the fact that $M_n$ is idempotent,
\begin{equation} \label{eq:AMAregPMP:expand} \begin{aligned}
&\left[ \left( \frac{ A M_n A }{n} + \regu I \right)^{-1/2}
		-
		\left( \frac{P M_n P }{n} \right)^{-1/2} \right]
		\frac{ PM_n Z }{n} \Sigmatilde_Z^{-1/2} \\
&~~~~~~~~~= \left[ \Whatpara  
	\left( \Rhatpara^2 + \regu I \right)^{-1/2} \Whatpara^\top
	- \Wtildepara \Rtildepara^{-1} \Wtildepara \right]
	\frac{ PM_n Z }{n} \Sigmatilde_Z^{-1/2} \\
&~~~~~~~~~~~~~~~~~~+ 
\Whatperp \left( \Rhatperp^2 + \regu I \right)^{-1/2}
	 \frac{ \Whatperp^\top P M_n Z \Sigmatilde_Z^{-1/2} }{ n } .
\end{aligned} \end{equation}
Applying Lemma~\ref{lem:paradiff},
\begin{equation*}
\left\| \Whatpara \left( \Rhatpara^2 + \regu I \right)^{-1/2} \Whatpara^\top
	- \Wtildepara \Rtildepara^{-1} \Wtildepara^\top \right\|
=
\Op{ \frac{|\regu|}{n^{3/2}}
+ \frac{ \spectralRate_n }{ n^{3/2} } + \frac{ \subspaceRate_n }{ n^2 } } .
\end{equation*}
Applying submultiplicativity along with the above bound, Lemma~\ref{lem:fullnet:rtildegrowth} and Equation~\eqref{eq:MZSig:construction},
\begin{equation} \label{eq:AMAPMP:leading} \begin{aligned}
& \left\| \left[ \Whatpara  
	\left( \Rhatpara^2 + \regu I \right)^{-1/2} \Whatpara^\top
	- \Wtildepara \Rtildepara^{-1} \Wtildepara \right]
	\frac{ PM_n Z }{n} \Sigmatilde_Z^{-1/2} \right\| \\
&~~~~~~~~~~~~~~~~~~~~~=
\Op{ \frac{|\regu|}{n}
+ \frac{ \spectralRate_n }{ n } + \frac{ \subspaceRate_n }{ n^{3/2} } } .
\end{aligned} \end{equation}

Recalling the SVD in Equation~\eqref{eq:Pproj:SVD}, using the fact that $M_n$ is idempotent and using submultiplicativity of the norm,
\begin{equation*}
\left\| \Whatperp \!\! \left(\! \Rhatperp^2 \!+\! \regu I \!\right)^{-1/2}
	\! \frac{ \Whatperp^\top \! P M_n Z \Sigmatilde_Z^{-1/2} }{ n } 
\right\|
\le
\left\| \left(\! \Rhatperp^2 \!+\! \regu I \!\right)^{-1/2} \right\|
\left\| \Whatperp^\top \Wtildepara \right\|
\left\| \Rtildepara \right\|
\left\| \frac{ M_n Z \Sigmatilde_Z^{-1/2} }{ \sqrt{n} } \right\| .
\end{equation*}
Applying Equation~\eqref{eq:MZSig:construction} and trivially lower-bounding the elements of $\Rhatperp^2$,
\begin{equation} \label{eq:AMAtrailing:IP}
\left\| \Whatperp \!\! \left(\! \Rhatperp^2 \!+\! \regu I \!\right)^{-1/2}
	\! \frac{ \Whatperp^\top \! P M_n Z \Sigmatilde_Z^{-1/2} }{ n } 
\right\|
\le \frac{1}{\sqrt{\regu}} \left\| \Whatperp^\top \Wtildepara \right\|
	\left\| \Rtildepara \right\|.
\end{equation}

By basic results in subspace geometry \citep[see, e.g.,][Chapter VII]{Bhatia1997} and Lemma~\ref{lem:CaiZha2018}, %Ex VII.1.9; Thm VII.3.1
\begin{equation*}
\left\| \Whatperp^\top \Wtildepara \right\|
= \left\| \sin \Theta( \Whatpara, \Wtildepara ) \right\|
= \Op{ \frac{ \spectralRate_n }{ n } } .
\end{equation*}
Applying this bound to Equation~\eqref{eq:AMAtrailing:IP} along with Lemma~\ref{lem:fullnet:rtildegrowth}
\begin{equation} \label{eq:AMA:trailing:final}
\left\| \Whatperp \left( \Rhatperp^2 + \regu I \right)^{-1/2}
	\frac{ \Whatperp^\top P M_n Z \Sigmatilde_Z^{-1/2} }{ n } 
\right\|
= \Op{ \frac{\spectralRate_n}{\sqrt{\regu n} } } .
\end{equation}

Applying the triangle inequality to Equation~\eqref{eq:AMAregPMP:expand} followed by Equations~ \eqref{eq:AMAPMP:leading}, and~ \eqref{eq:AMA:trailing:final},
\begin{equation} \label{eq:AMAPMP:finalbound} \begin{aligned}
& \left\|
\left[ \left( \frac{ A M_n A }{n} + \regu I \right)^{-1/2}
                -
                \left( \frac{P M_n P }{n} \right)^{-1/2} \right]
                \frac{ PM_n Z }{n} \Sigmatilde_Z^{-1/2}
\right\| \\
&~~~~~~~~~=
\Op{ \frac{\spectralRate_n}{\sqrt{\regu n} }
+ 
\frac{|\regu|}{n}
+ \frac{ \spectralRate_n }{ n } + \frac{ \subspaceRate_n }{ n^{3/2} } } .
\end{aligned} \end{equation}

Applying the triangle inequality to Equation~\eqref{eq:CCAnet:addsub} followed by Equations~\eqref{eq:AMAreguAP:final} and~\eqref{eq:AMAPMP:finalbound},
\begin{equation*}
\left\| \CCAregu( \!A, Z ) \!-\! \CCAdagger( XX^\top\!\!, Z ) \right\|
=  \Op{ \frac{\spectralRate_n}{\sqrt{\regu n} }
+ \frac{|\regu|}{n} + \frac{ \spectralRate_n }{ n } 
+ \frac{ \subspaceRate_n }{ n^{3/2} } } .
\end{equation*}
Applying our assumption that $\regu = O(n)$, the first term dominates the third, and we have
\begin{equation*}
\left\| \CCAregu( \!A, Z ) \!-\! \CCAdagger( XX^\top\!\!, Z ) \right\|
=  \Op{ \frac{\spectralRate_n}{\sqrt{\regu n} }
+ \frac{|\regu|}{n} + \frac{ \subspaceRate_n }{ n^{3/2} } } .
\end{equation*}
Finally, applying this to Equation~\eqref{eq:CCA:fullnet:target} completes the proof.
\end{proof}

\section{Covariance Test for group LASSO} \label{apx:ctest_glasso}
The general idea of the covariance test in the case of the regular LASSO is to compute, for each $k\in[d]$, the covariance of $\Xhat_{\cdot k}$ with the predicted values $Z\Bhat_{\textrm{rLASSO},k}(\lassoreg)$, where $\Bhat_{\textrm{rLASSO},k}(\lassoreg)$ is the solution to the optimization problem in Equation~\eqref{eq:model-based:separateregs} with tuning parameter $\lassoreg>0$ chosen suitably.
If the covariates have no impact, this covariance is likely small.
\cite{LTTT14} argue that a good choice of $\lassoreg$ is the lowest value for which only one covariate is selected. This value of $\lassoreg$ can be computed and therefore no computationally intensive cross-validation is required.

In the following, we briefly discuss how to generalize the covariance test from \cite{LTTT14} to \emph{multivariate regression}, in the case where the goal is to test the hypothesis $H_0:\beta=0$. 
More precisely, we consider
\begin{equation}
\label{eq:mrl}
\hat{\beta}(\lassoreg)=\argmin{\beta\in\IR^{p\times d}}\frac{1}{nd}\left\|\hat{X}-Z\beta\right\|_F^2+\frac{\lassoreg}{\sqrt{d}}\sum_{j=1}^p\|\beta_j\|,
\end{equation}
where $\beta_j$ denotes the $j$-th row of $\beta$. This problem is exactly as introduced in Section \ref{subsec:mod_approach}, but we emphasize that the \emph{multivariate regression} settting as above might also be of independent interest.

In this setting, the covariance test statistic is is given by
\begin{equation} \label{eq:cov_test}
T_{\textrm{cov}}
=\frac{1}{\hat{\sigma}^2}
\sum_{i=1}^n\sum_{k=1}^d\hat{X}_{ik}
		\left[Z\hat{\beta}(\lassoreg_2)\right]_{ik},
\end{equation}
where $\lassoreg_2>0$ is the smallest value of $\lassoreg$ at which $\hat{\beta}_j(\lassoreg)\neq0$ for exactly one $j \in [p]$ and $\hat{\sigma}^2$ is an estimator for the error variance (see Remark \ref{rem:var} below).
We fix further notations, which are chosen for easy comparison with the results of \cite{LTTT14} for the standard LASSO.
We denote by $\lassoreg_1$ the smallest value of $\lassoreg$ for which $\hat{\beta}_j(\lassoreg)=0$ for all $j \in [p]$.
Let furthermore
\begin{equation*}
U_k:=\frac{1}{n}\hat{X}^\top Z_{\cdot k},\quad R_{\ell,k}:=\frac{1}{n}[Z^\top Z]_{\ell,k} \text{ for } \ell,k\in[p].
\end{equation*}
To use the test statistic $T_{\textrm{cov}}$ that was introduced in Equation~\eqref{eq:cov_test}, we have to provide a formula for $\lassoreg_2$. Such a formula is tedious to find, but it can be done in this case by characterizing $\hat{\beta}(\lassoreg)$ as defined in Equation~\eqref{eq:mrl} using sub-differential calculus. The objective function on the right hand side of Equation~\eqref{eq:mrl} is not differentiable but convex. Therefore, we may compute its sub-differential, which is a set-valued function. We denote by $B_1(0):=\{x\in\IR^d:\|x\|\leq1\}$ the unit ball around $0$. Then, the sub-differential with respect to $\beta_k$ is given by
$$-\frac{2}{nd}\hat{X}^\top Z_{\cdot k}+\frac{2}{nd}\sum_{\ell=1}^p\left[Z^\top Z\right]_{\ell,k}\beta_{\ell}+\frac{\lassoreg}{\sqrt{d}}\left\{\begin{array}{lll}
\|\beta_k\|^{-1}\beta_k & \textrm{ if }\beta_k\neq0 \\
B_1(0) & \textrm{ if }\beta_k=0
\end{array}\right..$$
Note that the above is a proper set if $\beta_k(\lassoreg)=0$ (namely a shifted and scaled ball). Otherwise, the above is a single number which we interpret as a set with exactly one element. Sub-differential calculus implies that the $0$-vector must be contained in this set when we plug-in the optimizer $\hat{\beta}_k(\lassoreg)$. This has to be true for all $k\in[p]$. In other words, using the previously introduced notation,
\begin{equation} \label{eq:foc_gl}
0\in-\frac{2}{d}U_k+\frac{2}{d}\sum_{\ell=1}^pR_{\ell,k}\hat{\beta}_{\ell}(\lassoreg)+\frac{\lassoreg}{\sqrt{d}}\left\{\begin{array}{lll}
\|\hat{\beta}_k(\lassoreg)\|^{-1}\hat{\beta}_k(\lassoreg) & \textrm{ if }\hat{\beta}_k(\lassoreg)\neq0 \\
B_1(0) & \textrm{ if }\hat{\beta}_k(\lassoreg)=0
\end{array}\right..
\end{equation} 
In the case that $\hat{\beta}_k(\lassoreg)\neq0$, the right hand side above is just a set with a single element and the $\in$-relation can in fact be read as an equality.

We first establish the value of $\lassoreg_1$.
\begin{lemma}
\label{lem:lassoreg1}
Suppose that
$$m:=\argmax{k\in[p]}\frac{2}{\sqrt{d}}\|U_k\|$$
is unique. In the situation above, we have
\begin{equation*}
\lassoreg_1
=\frac{2}{\sqrt{d}}\|U_m\|.
\end{equation*}
\end{lemma}
\begin{proof}
We must show that $\hat{\beta}_k(\lassoreg)=0$ for all $k\in[p]$ is a solution to \eqref{eq:mrl} if and only if $\lassoreg\geq\lassoreg_1$. In view of \eqref{eq:foc_gl}, $\hat{\beta}_k(\lassoreg)=0$ for all $k\in[p]$ is equivalent to
$$0\in-\frac{2}{d}U_k+\frac{\lassoreg}{\sqrt{d}}B_1(0)\text{ for all } k\in[p],$$
which is true if and only if 
$$\left\|\frac{2}{\sqrt{d}\lassoreg}U_k\right\|\leq 1\text{ for all } k\in[p],$$
which implies the statement.
\end{proof}

In Section 4.3.1 of \cite{HTW15}, some computational details for the group LASSO are considered. A main challenge is that the solution of the group LASSO problem has no closed-form solution.
However, in the multivariate regression setting that we consider here, we can formulate the solution explicitly for the special case of the first variable entering the regression.
We will show in the following that $\hat{\beta}_m(\lassoreg)$ is the first non-zero variable for $\lassoreg<\lassoreg_1$. In the first step, we assume that this is the case, and compute an explicit formula for $\hat{\beta}_m(\lassoreg)$.

\begin{lemma} \label{lem:beta_one}
Suppose $\lassoreg<\lassoreg_1$ is such that $\hat{\beta}_m(\lassoreg)\neq0$ and $\hat{\beta}_k(\lassoreg)=0$ for all $k\neq m$. Then
\begin{equation*}
\hat{\beta}_m(\lassoreg)
=\frac{1}{R_{m,m}}\left(1-\lassoreg\frac{\sqrt{d}}{2\|U_m\|}\right)U_m.
\end{equation*}
\end{lemma}
\begin{proof}
By Equation~\eqref{eq:foc_gl}, for $\hat{\beta}_m(\lassoreg)$, we have
\begin{equation*}
0=-\frac{2}{d}U_m+\frac{2}{d}R_{m,k}\hat{\beta}_m(\lassoreg)
+\frac{\lassoreg}{\sqrt{d}\|\hat{\beta}_m(\lassoreg)\|}
	\hat{\beta}_m(\lassoreg) .
\end{equation*}
Rearranging the above equation yields
\begin{equation} \label{eq:beta_raw}
\hat{\beta}_m(\lassoreg)=\frac{2}{d}\left(\frac{2}{d}R_{m,m}+\frac{\lassoreg}{\sqrt{d}\|\hat{\beta}_m(\lassoreg)\|}\right)^{-1}U_m.
\end{equation}
Taking the norm on both sides, noting that $R_{m,m}\geq0$,
\begin{equation*}
\|\hat{\beta}_m(\lassoreg)\|=\frac{2}{d}\left(\frac{2}{d}R_{m,m}+\frac{\lassoreg}{\sqrt{d}\|\hat{\beta}_m(\lassoreg)\|}\right)^{-1}\|U_m\|.
\end{equation*}
Solving the above equation for $\|\hat{\beta}_m(\lassoreg)\|$ yields
\begin{equation*}
\|\hat{\beta}_m(\lassoreg)\|=\frac{2\|U_m\|-\lassoreg\sqrt{d}}{2R_{m,m}}.
\end{equation*}
Replacing the above in Equation~\eqref{eq:beta_raw} yields
\begin{equation*} \begin{aligned}
\hat{\beta}_m(\lassoreg)&=\frac{2}{d}\left(\frac{2}{d}R_{m,m}+\frac{\lassoreg}{\sqrt{d}\frac{2\|U_m\|-\lassoreg\sqrt{d}}{2R_{m,m}}}\right)^{-1}U_m \\
&=\frac{1}{R_{m,m}}\left(1+\frac{\lassoreg\sqrt{d}}{2\|U_m\|-\lassoreg\sqrt{d}}\right)^{-1}U_m
=\frac{1}{R_{m,m}}\left(\frac{2\|U_m\|}{2\|U_m\|-\lassoreg\sqrt{d}}\right)^{-1}U_m \\
&=\frac{1}{R_{m,m}}\left(1-\frac{\lassoreg\sqrt{d}}{2\|U_m\|}\right)U_m,
\end{aligned} \end{equation*}
which is the statement we wanted to prove.
\end{proof}

We can now show that there is indeed $\lassoreg_2$ such that, for all $\lassoreg\in[\lassoreg_2,\lassoreg_1)$, $hat{\beta}_m(\lassoreg)\neq0$ and $\hat{\beta}_k(\lassoreg)=0$ for all $k\neq m$.
We will also provide a formula for $\lassoreg_2$.
We will achieve this by showing that $\hat{\beta}_m(\lassoreg)$ as in Lemma \ref{lem:beta_one} and $\hat{\beta}_k(\lassoreg)=0$ for $k\neq m$ provides a set of solution for the system in Equation~\eqref{eq:foc_gl} when $\lassoreg$ is chosen suitably.
Note firstly that, by construction, Equation~\eqref{eq:foc_gl} always holds true for $k=m$.
For $k\neq m$, Equation~\eqref{eq:foc_gl} reads (replacing the formula from Lemma \ref{lem:beta_one})
$$0\in-\frac{2}{d}U_k+\frac{2}{d}\frac{R_{m,k}}{R_{m,m}}\left(1-\frac{\lassoreg\sqrt{d}}{2\|U_m\|}\right)U_m+\frac{\lassoreg}{\sqrt{d}}B_1(0).$$
Noting that the right hand side above is a shifted and scaled ball, it is clear that the above relation holds if and only if
$$\left\|\frac{2}{\lassoreg\sqrt{d}}\left(U_k-\frac{R_{m,k}}{R_{m,m}}U_m\right)+\frac{R_{m,k}}{R_{m,m}}\cdot\frac{U_m}{\|U_m\|}\right\|^2\leq1.$$
Expanding the norm and rearranging the equation, finally yields
\begin{align}
\lassoreg^2\left(\frac{R_{m,k}^2}{R_{m,m}^2}-1\right)+\lassoreg\cdot\frac{R_{m,k}}{R_{m,m}}\cdot\frac{4\left(U_k^\top U_m-\frac{R_{m,k}}{R_{m,m}}\|U_m\|^2\right)}{\sqrt{d}\|U_m\|}+\frac{4\left\|U_k-\frac{R_{m,k}}{R_{m,m}}U_m\right\|^2}{d}\leq0. \label{eq:sol_condition}
\end{align}
We are going to prove that the above statement, and thus \eqref{eq:foc_gl}, can always be made true by choosing $\lassoreg<\lassoreg_1$ suitably. We define for ease of notation
$$A_k:=\frac{R_{m,k}}{R_{m,m}}\text{ and }B_k:=U_k^\top U_m-\frac{R_{m,k}}{R_{m,m}}\|U_m\|^2.$$
Now, viewing the left hand side of Inequality~\eqref{eq:sol_condition} as a second order polynomial in $\lassoreg$, the roots are given by
\begin{equation} \label{eq:abc_solution}
\frac{-A_k\cdot\frac{4B_k}{\sqrt{d}\|U_m\|}\pm\sqrt{A_k^2\cdot\frac{16B_k^2}{d\|U_m\|^2}-\frac{16\left(A_k^2-1\right)\left\|U_k-A_kU_m\right\|^2}{d}}}{2\left(A_k^2-1\right)}
\end{equation}
Note at this point that, in the case where $R_{m,k}^2\leq R_{m,m}^2$, that is, $A_k\in[-1,1]$, the term inside the square root is clearly non-negative and hence there are two solutions. In order to see that there are two solutions also in the case where $R_{m,k}^2\geq R_{m,m}^2$, note that, by tedious but elementary algebra, the term inside of the root in Equation~\eqref{eq:abc_solution} can be rewritten as
\begin{align*}
\frac{16}{d\|U_m\|^2R_{m,m}^2}\left(\left[U_k^\top U_mR_{m,k}-R_{m,m}\|U_m\|^2\right]^2+\left(R_{m,k}^2-R_{m,m}^2\right)\left(\|U_m\|^4-\|U_k\|^2\|U_m\|^2\right)\right).
\end{align*}
In this notation, it is evident that for $R_{m,k}^2\geq R_{m,m}^2$ the above term and thus the term inside the root of Equation~\eqref{eq:abc_solution} is also non-negative by choice of $m$.
Therefore, \eqref{eq:sol_condition} holds always with equality for exactly two values of $\lassoreg$.
Hence, it it is always fulfilled for some range of $\lassoreg$.
Moreover, by plugging in $\lassoreg_1$ from Lemma \ref{lem:lassoreg1} into Inequality~\eqref{eq:sol_condition}, we can see after some elementary algebra that $\lassoreg_1$ fulfills Inequality~\eqref{eq:sol_condition} for all $k\neq m$.
Therefore, we conclude the existence of $\lassoreg_2<\lassoreg_1$ such that for all $\lassoreg\in[\lassoreg_2,\lassoreg_1)$, $\hat{\beta}_k(\lassoreg)=0$ for $k\neq m$ and $\hat{\beta}_m(\lassoreg)$ is given by the Formula from Lemma \ref{lem:beta_one}.

The following lemma provides a specific formula for $\lassoreg_2$.

\begin{lemma} \label{lem:lassoreg2}
In the situation of Lemma \ref{lem:lassoreg1}, it is true that
$$\lassoreg_2=\max_{k\neq m}\frac{-A_k\frac{4B_k}{\sqrt{d}\|U_m\|}-\sqrt{\frac{16A_k^2B_k^2}{d\|U_m\|^2}-\frac{16\left(A_k^2-1\right)\left\|U_k-A_kU_m\right\|^2}{d}}}{2\left(A_k^2-1\right)}.$$
\end{lemma}
\begin{proof}
By Lemma \ref{lem:lassoreg1}, we have that $\hat{\beta}_k=0$ for all $k\in[p]$ for $\lassoreg\geq\lassoreg_1$.
Hence, we have left to show that \eqref{eq:sol_condition} holds true for all $k\neq m$ and $\lassoreg\in[\lassoreg_2,\lassoreg_1)$ but not for $\lassoreg<\lassoreg_2$.

Let $\lassoreg_{k,-}$ denote the root of the left hand side of Inequality~\eqref{eq:sol_condition} understood as a polynomial in $\lassoreg$ given by Equation~\eqref{eq:abc_solution} when choosing $-$ in place of $\pm$.
Analogously, $\lassoreg_{k,+}$ denotes the same for $+$ in place of $\pm$.
With this notation, the statement of the Lemma can be restated as $\lassoreg_2=\max_{k\neq m}\lassoreg_{k,-}$.
We prove that this number has the right properties.
We distinguish two cases: \\
\underline{$R_{m,k}^2<R_{m,m}^2$:} \\
In this case, the parabola (in $\lassoreg$) to the left of Inequality~\eqref{eq:sol_condition} is downwoards facing.
Hence, every $\lassoreg$ larger than the larger root of the parabola provides a valid solution.
Since $A_k^2-1<0$, we see that $\lassoreg_{k,-}$ is the larger root.
Therefore, we have to choose $\lassoreg\geq\lassoreg_{k,-}$ for all $k$ for which $R_{m,k}^2<R_{m,m}^2$.
In particular, for $\lassoreg\in[\lassoreg_2,\lassoreg_1)$, Inequality~\eqref{eq:sol_condition} holds true. \\

\underline{$R_{m,k}^2\geq R_{m,m}^2$:} \\
In this case, the parabola on the left of \eqref{eq:sol_condition} is upwards facing, and Inequality~\eqref{eq:sol_condition} holds true for $\lassoreg\in[\lassoreg_{k,-},\lassoreg_{k,+}]$ (note that in this case $\lassoreg_{k,-}\leq\lassoreg_{k,+}$).
Hence, in order to prove that \eqref{eq:sol_condition} holds true for all $\lassoreg\in[\lassoreg_2,\lassoreg_1)$, we have to show that $\lassoreg_{k,+}\geq\lassoreg_1$.
Denote the parabola (in $\lassoreg$) to the left of Inequality~\eqref{eq:sol_condition} by $P(\lassoreg)$, and recall that it is upwards facing.
By elementary computations, we see that $P(\lassoreg_1)<0$ if $m$ is the unique maximizer of $k\mapsto\|U_k\|$, which we assume.
It follows that $\lassoreg_1\in[\lassoreg_{k,-},\lassoreg_{k,+}]$.

\vspace{0.5cm}

Hence, in summary, we have proven that Inequality~\eqref{eq:sol_condition} holds true for all $\lassoreg\in[\lassoreg_2,\lassoreg_1)$.
We have left to argue that, in the case $\lassoreg<\lassoreg_2$, Inequality~\eqref{eq:sol_condition} is violated for at least one $k\neq m$.
If $\lassoreg<\lassoreg_2$, then $\lassoreg<\lassoreg_{k_0,-}$ for some $k_0\neq m$.
The discussion about the location of the roots of $P(\lassoreg)$ and the orientation of the parabolas in either of the previous steps shows that Inequality~\eqref{eq:sol_condition} is violated for this choice $k_0$.
The result follows.
\end{proof}
We finally mention a simple form of $T_{\textrm{cov}}$ that is analogue to what \cite{LTTT14} show in their Lemma 1.
Using Lemma \ref{lem:beta_one} and Lemma \ref{lem:lassoreg1} in Equality~\eqref{eq:cov_test}, we get
\begin{align*}
T_{\textrm{cov}}=&\frac{1}{\hat{\sigma}^2}\sum_{i=1}^n\sum_{j=1}^d\hat{X}_{ij}\left[Z\hat{\beta}(\lassoreg_2)\right]_{i,j}=\frac{1}{\hat{\sigma}^2}\sum_{i=1}^n\sum_{j=1}^d\hat{X}_{ij}Z_{im}\frac{1}{R_{m,m}}\left(1-\frac{\lassoreg_2}{\lassoreg_1}\right)U_{m,j} \\
=&\frac{n}{\hat{\sigma}^2R_{m,m}}\left(1-\frac{\lassoreg_2}{\lassoreg_1}\right)\|U_m\|^2=\frac{n}{\hat{\sigma}^2R_{m,m}}\left(1-\frac{\lassoreg_2}{\lassoreg_1}\right)\frac{\lassoreg_1^2d}{4}=\frac{nd}{4\hat{\sigma}^2R_{m,m}}\lassoreg_1\left(\lassoreg_1-\lassoreg_2\right).
\end{align*}

\begin{remark}
\label{rem:var}
For a given dataset, we estimate $\hat{\sigma}^2$ by computing the empirical variance of the residuals of a group LASSO fitted with a cross-validated $\lassoreg$. For a single data set this is computationally feasible. For the permutation test, however, we want to understand the behavior of the test statistic on the hypothesis. In this case, we therefore use an estimator that is valid under $H_0$, i.e., $B=0$. On the hypothesis, no covariate is active, and therefore we choose simply the empirical variance of the observations as $\hat{\sigma}^2$.
\end{remark}

\section{Additional empirical results}
\label{apx:additional_emp_res}

Here we collect a handful of additional experimental results that complement our main findings in Section~\ref{sec:expts}.

\subsection{Additional visualisations of the simulations in Section \ref{subsec:simulations}}
\label{apx:additional_emp_res:simulations}
In order to investigate the distribution of the test statistic on the hypothesis in more detail, we compute the distributions of the $p$-values generated by the different permutation tests.
When $H_0$ holds, the $p$-values should be uniformly distributed on $[0,1]$. 
Figure \ref{fig:qq_scen_iv} shows, in the case of Scenario (iv) (i.e., the high-correlation setting), QQ-plots that compare the empirical distribution of the $p$-values to a uniform distribution.
Recall that the LASSO permutation test involves a Bonferroni-correction. Therefore, we have multiplied the resulting $p$-value by the estimated number of latent dimensions, and hence it can take values larger than one.
Thus, we cannot expect that the QQ-plot, in which we compare to a uniform distribution on $[0,1]$, lies on the diagonal.
The lower quantiles do lie on the diagonal, however, and therefore we see that the permutation test will hold conventional levels.
Moreover, the slight deviation of the LASSO test that we have observed in Table \ref{tab:H0level} in Scenario (iv) appears visually not relevant in Figure \ref{fig:qq_scen_iv}.
For all other methods, we see that the QQ-plots lie very closely to the diagonal indicating a correct distribution of the p-values.
For other scenarios, we observe the same behavior. We show the QQ-plots also for the network Scenario (v) in Figure \ref{fig:qq_scen_v}.
The effect of the Bonferroni-correction is not so strong in this scenario because the latent dimension is on average chosen smaller. Since each latent dimension is tested separately, a lower number of latent dimensions means that the multiple testing issue is not so pronounced.

\begin{figure}
\centering
\includegraphics[height=0.9\textheight]{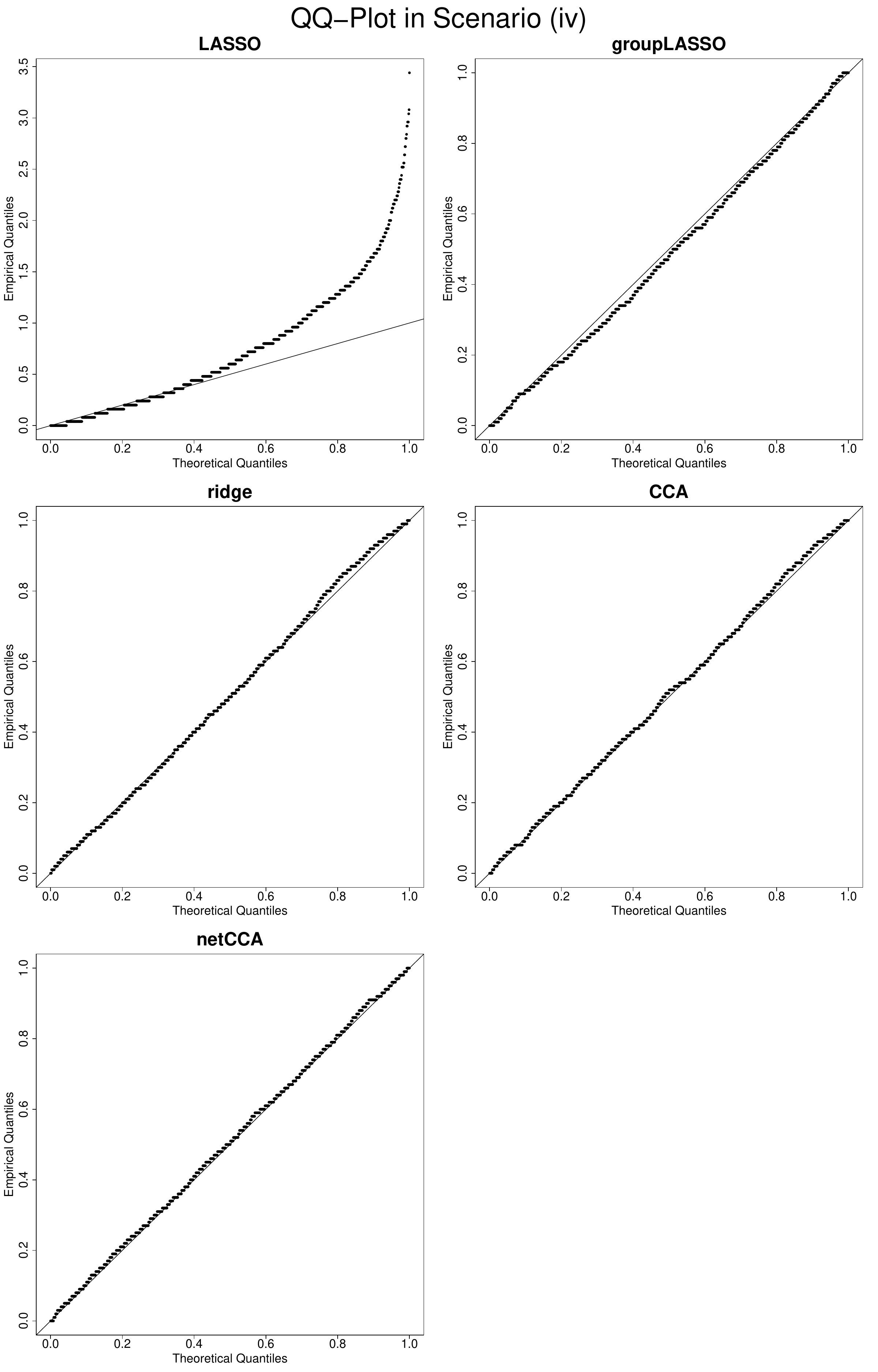}
\caption{QQ-plot of the empirical distributions of the p-values generated from the different methods in Scenario (iv) vs. a uniform distribution.}
\label{fig:qq_scen_iv}
\end{figure}

\begin{figure}
\centering
\includegraphics[height=0.9\textheight]{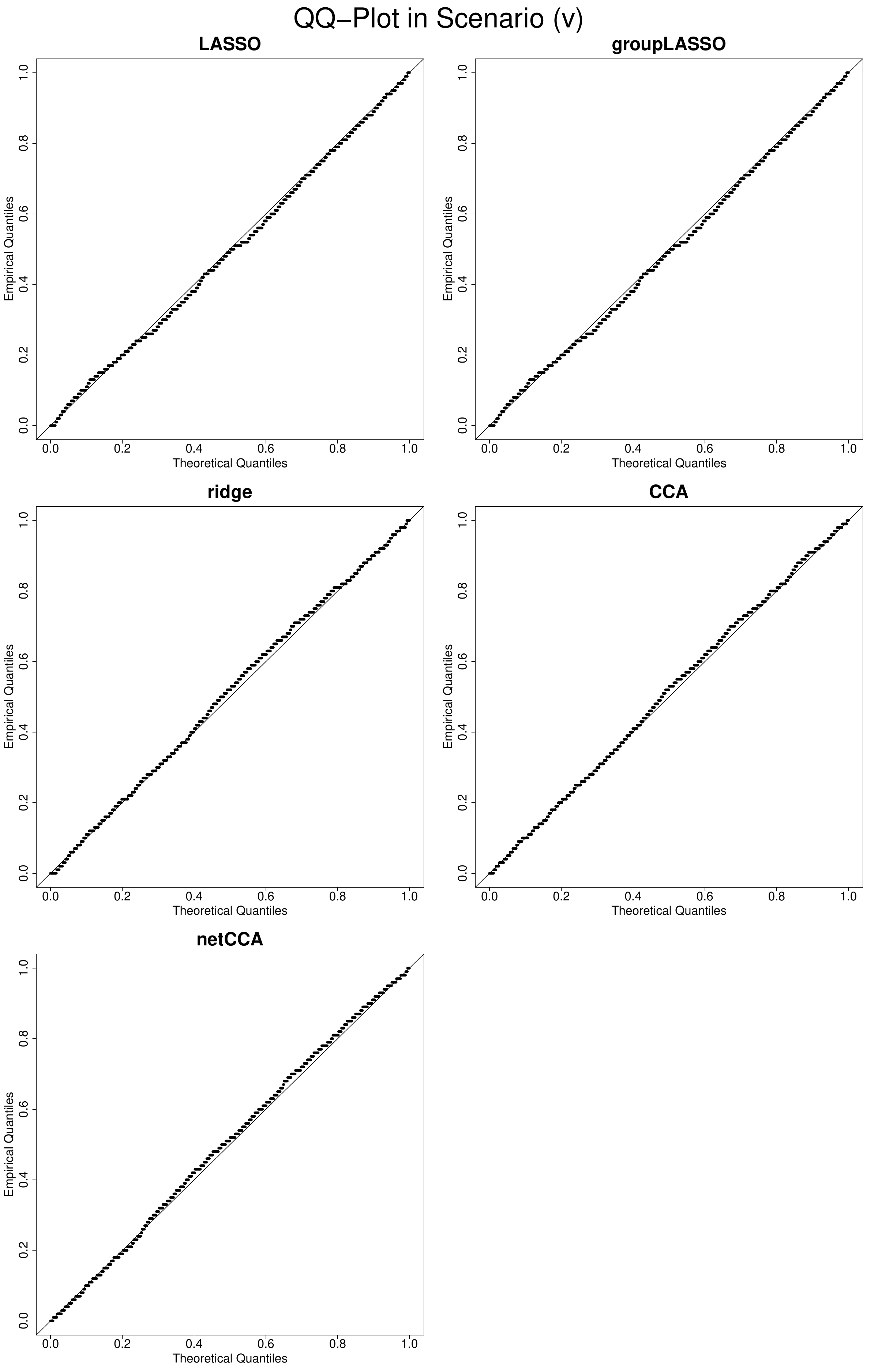}
\caption{QQ-plot of the empirical distributions of the p-values generated from the different methods in Scenario (v) vs. a uniform distribution.}
\label{fig:qq_scen_v}
\end{figure}

\subsection{Further plots illustrating the data analysis of Section \ref{subsec:data_analysis}}
\label{apx:additional_emp_res:data}
In Figure \ref{fig:spectrum} we show scree plots of the eigenvalues of the adjacency matrices of the networks analyzed in Section \ref{subsec:data_analysis}.
The selected latent dimensions are $14$ in case of the Protein dataset and $55$ in case of the Wikipedia dataset.
In both cases, Figure \ref{fig:spectrum} indicates that the selected number of eigenvalues cover a significant portion of the structure in the adjacency matrix.

\begin{figure}
    \centering
    \begin{subfigure}[t]{0.48\textwidth}
        \centering
        \includegraphics[width=\linewidth]{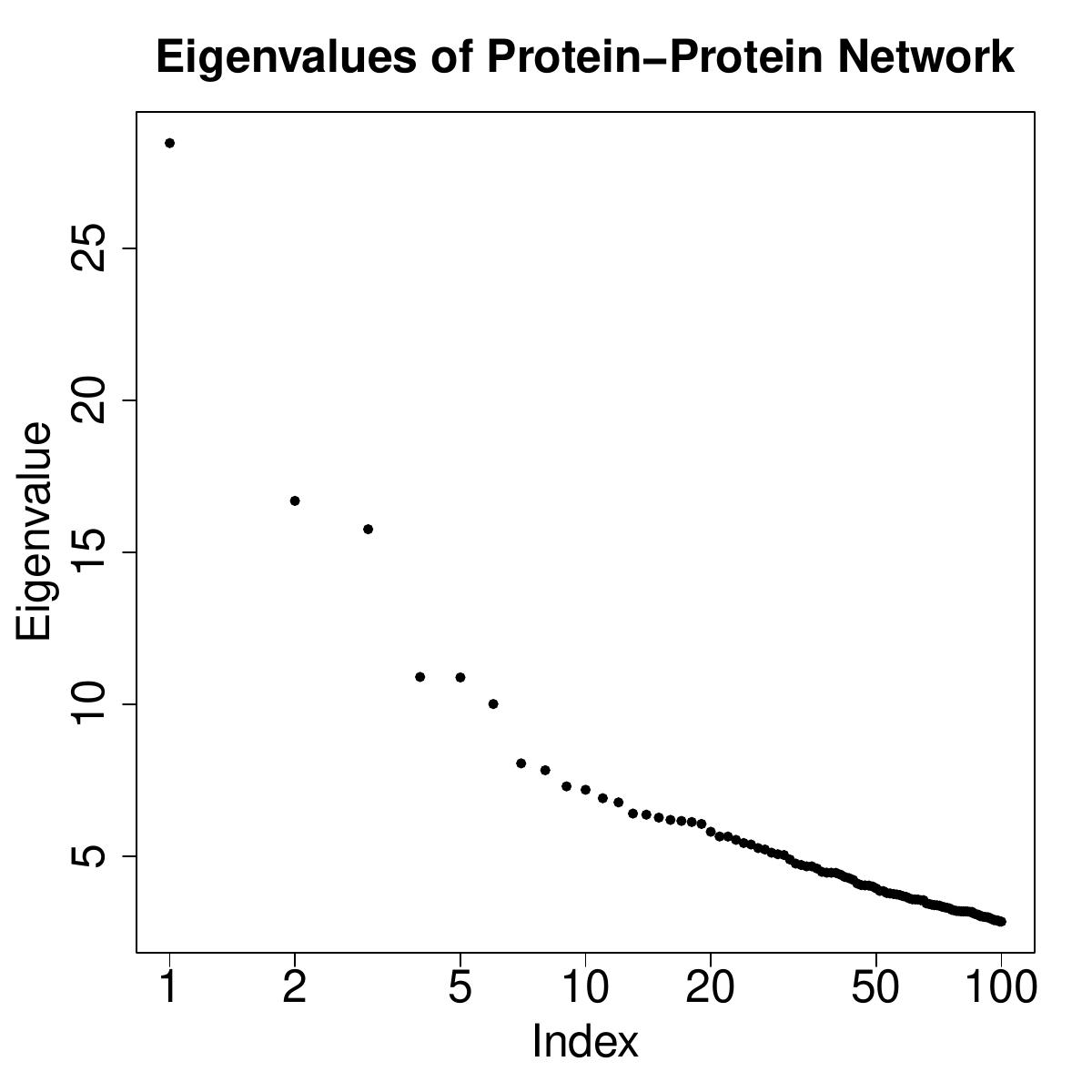}
        \caption{Protein interaction data set}
        \label{fig:spectrum:protein}
    \end{subfigure}
    \hfill
    \begin{subfigure}[t]{0.48\textwidth}
        \centering
        \includegraphics[width=\linewidth]{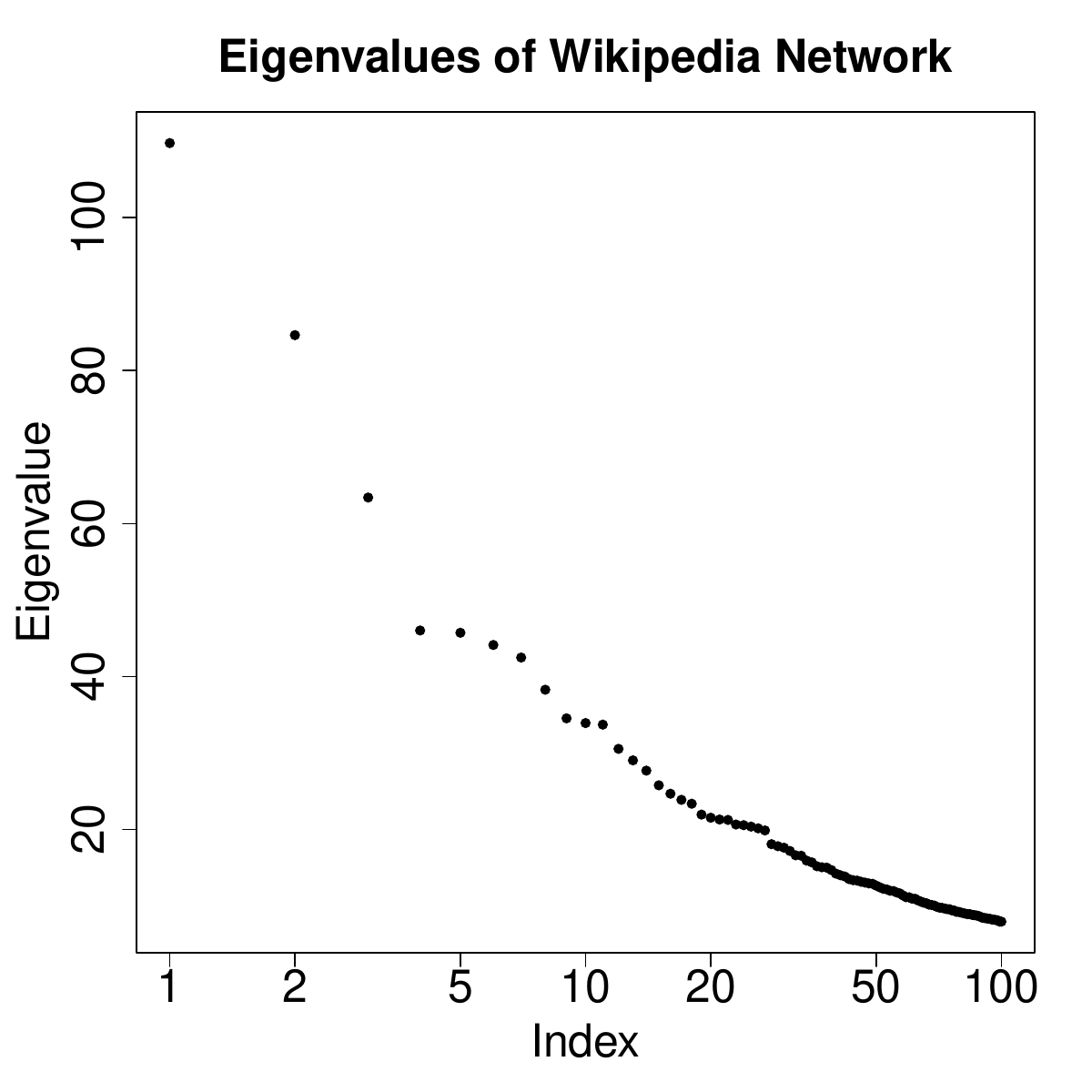}
        \caption{Wikipedia data set}
		\label{fig:spectrum:chameleon}
    \end{subfigure}
	\caption{Eigenvalues (largest to smallest) of the adjacency matrices of the respective real world data set studied in Section \ref{subsec:data_analysis}.}
    \label{fig:spectrum}
\end{figure}

To further illustrate the utility of our novel group LASSO method, we visualize the feature projections, that is, the columns of $Z\Bhat^{\gLASSO}$, or, in other words, the latent positions predicted by our model. $Z\hat{B}^{\gLASSO}$ contains in its $14$ columns a feature projection for each of the $14$ latent dimensions. Figures \ref{fig:protein_LD1} and \ref{fig:protein_LD2} show the network plot, where the vertex colors are indicative for the feature projection (there is one network plot for each of the $14$ feature projections). Different feature projections appear to be used to distinguish some clusters of proteins from the others, e.g., feature projection $1$ in Figure \ref{fig:protein_LD1} identifies higher degree vertices in the central cluster. Feature projection $4$ in the same Figure seems to differentiate between three smaller clusters below the central cluster.

\begin{figure}
\centering
\includegraphics[height=0.9\textheight]{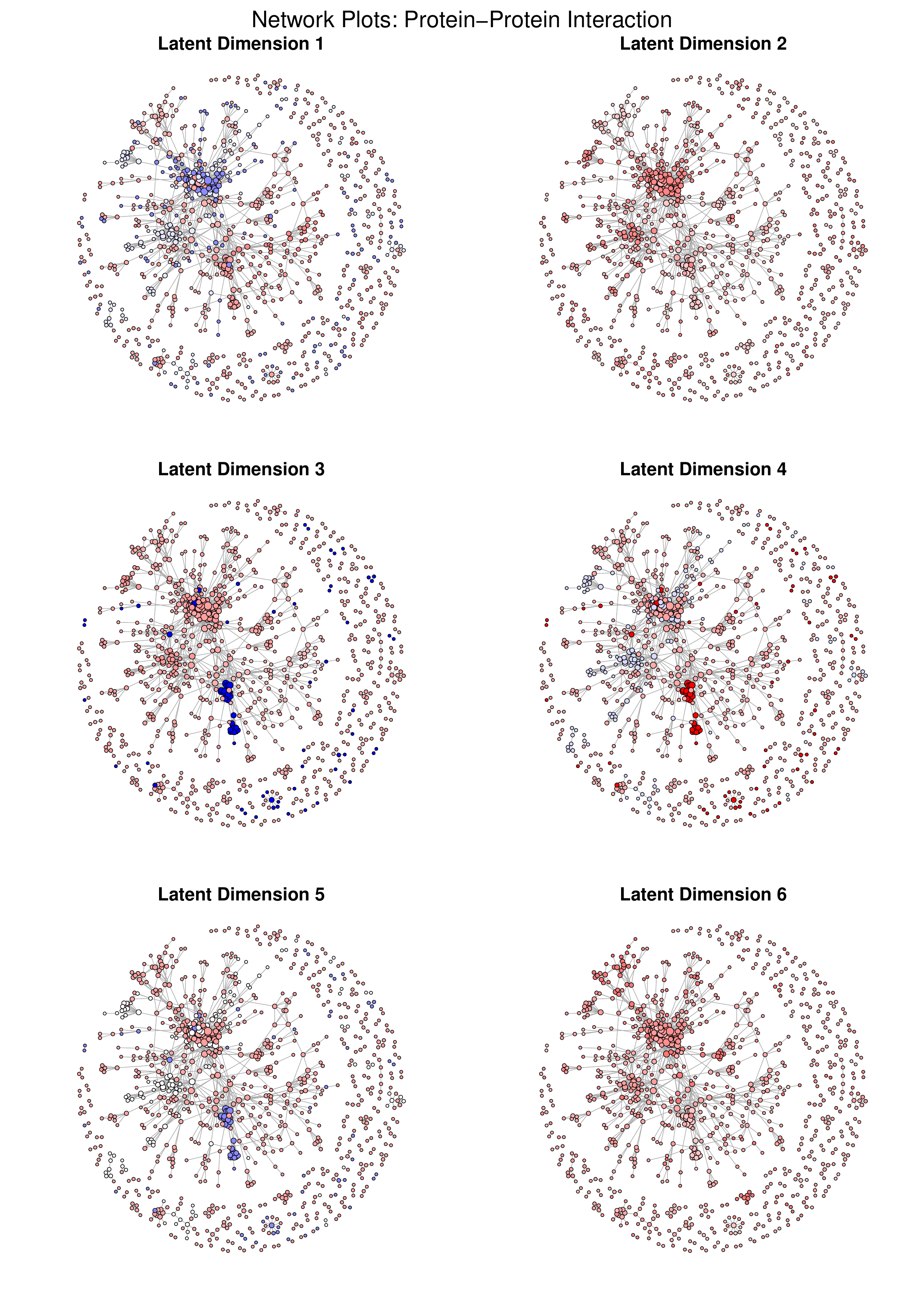}
\caption{Protein network analyzed in Section \ref{subsubsec:data_analysis:protein}. Vertex sizes are proportional to their degree. Vertex colors in plot $k$ correspond to the values of the $k$-th column of $Z\hat{B}^{\gLASSO}$.}
\label{fig:protein_LD1}
\end{figure}

\begin{figure}
\centering
\includegraphics[height=0.9\textheight]{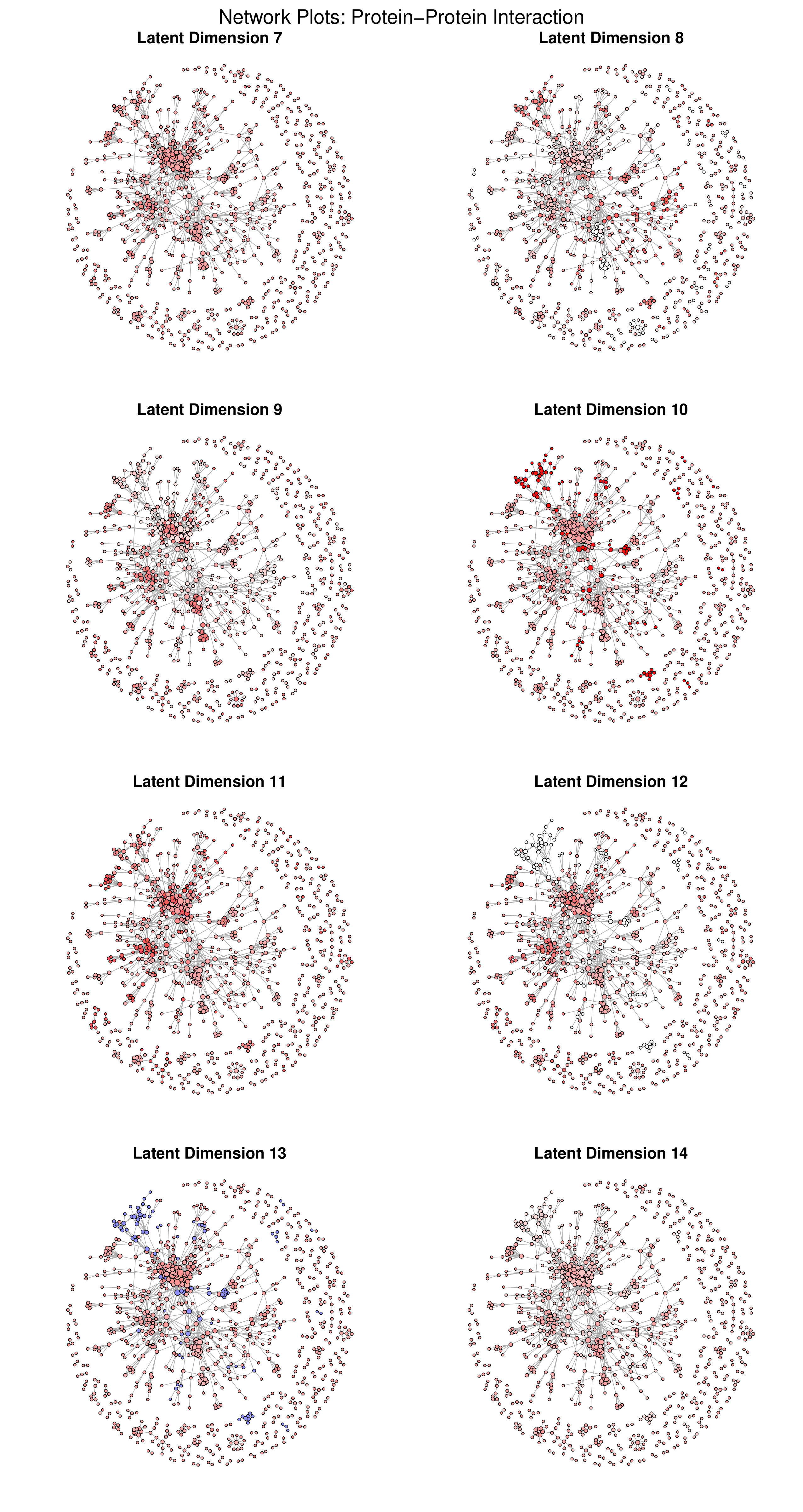}
\caption{Protein network analyzed in Section \ref{subsubsec:data_analysis:protein}. Vertex sizes are proportional to their degree. Vertex colors in plot $k$ correspond to the values of the $k$-th column of $Z\hat{B}^{\gLASSO}$.}
\label{fig:protein_LD2}
\end{figure}

\end{document}